\documentclass[11pt]{article}

\usepackage{fragkia}

\usepackage{multirow,makecell,colortbl,caption,float}

\makeatletter
\AtBeginEnvironment{algorithm}{\setcounter{AlgoLine}{0}}
\makeatother
\usepackage{subcaption}
\usepackage{booktabs}

\usepackage{wrapfig}
\usepackage[normalem]{ulem}
\usepackage{fancyvrb}
\usepackage{fvextra}
\usepackage{listings}
\usepackage{authblk}
\usepackage{tikz}

\usepackage{amsthm}
\usepackage[nameinlink]{cleveref}
\Crefname{prop}{proposition}{propositions}
\Crefname{remark}{remark}{remarks}
\AtBeginDocument{
    \hypersetup{hypertexnames=false}
}
\hypersetup{colorlinks, citecolor=magenta, linkcolor=magenta}

\newcommand{\Ldim}{\mathrm{Ldim}}
\newcommand{\SCLdim}{\mathrm{SC\text{-}Ldim}}
\newcommand{\WSCLdim}{\mathrm{WSC\text{-}Ldim}}
\newcommand{\spn}{\mathrm{span}}
\newcommand{\Hd}{H_{\mathrm{sub}}}
\newcommand{\gc}{\gamma_c}
\newcommand{\gs}{\gamma_s}
\DeclareMathOperator{\dm}{dim}

\newcommand{\yes}{\textsc{yes}}
\newcommand{\no}{\textsc{no}}

\newcommand{\cO}{\mathcal{O}}

\definecolor{revisiongreen}{rgb}{0.12,0.42,0.22}
\definecolor{zrblue}{rgb}{0.03,0.25,0.50}
\definecolor{zrnote}{rgb}{0.55,0.08,0.35}

\title{Verify to Amplify: Improving Reasoning via Learned Chain-of-Thought Verification}

\date{}

\author{Maria-Florina Balcan$^1$ \\ ninamf@cs.cmu.edu \hfill Avrim Blum$^2$ \\ avrim@ttic.edu\hfill  Zhiyuan Li$^2$\\ zhiyuanli@ttic.edu \hfill  Dravyansh Sharma$^{2,3}$\\dravy@ttic.edu}

\author{\bf Maria-Florina Balcan$^1$ \hfill Avrim Blum$^2$ \hfill Kiriaki Fragkia$^1$ \hfill  Zhiyuan Li$^2$ \hfill  Dravyansh Sharma$^{2,3}$}

\date{%
    $^1$Carnegie Mellon University\\%
    $^2$Toyota Technological Institute at Chicago\\%
    $^3$Northwestern University\\%
    \small \{ninamf, kiriakif\}@cs.cmu.edu, \{avrim, zhiyuanli, dravy\}@ttic.edu
}

\begin{document}

\maketitle

\pagenumbering{roman}

\begin{abstract}
Large Language Models (LLMs) using chain-of-thought reasoning have demonstrated great potential for solving complex reasoning and planning tasks. Despite these advances, LLM-generated outputs remain susceptible to errors, making verification an important component of reliable reasoning systems. Learned verifiers can help increase trust, enforce safety constraints, and ensure alignment with personal preferences, while also providing feedback that can be used to improve the generation process itself. This raises a central challenge: when learned verifiers are used to guide generation, the resulting feedback loop between generator and verifier may induce substantial distribution shift. This issue is particularly salient for process reward models (PRMs), a prominent class of learned verifiers that score or classify individual steps in a chain-of-thought reasoning trace.\looseness-1

Motivated by this challenge, we propose a new online learning framework for learning chain-of-thought verifiers that, given a problem statement and a reasoning trace, check the correctness of each reasoning step given the preceding steps. Highlighting the asymmetric role of soundness errors (accepting an incorrect reasoning step) and completeness errors (flagging a correct step as wrong), we introduce novel notions of dimension that characterize the optimal tradeoff between them. We then show how our learned verifiers can be used to boost the accuracy of a weak generator, or even a weak collection of generators. Under the mild assumption that one of the generators can produce a correct next step with a small success probability, we show how to learn a strong generator with small error and abstention rates. \looseness-1

Our results also allow learning from offline data in settings where queries to an expert verifier can be simulated from a small set of correct reasoning traces. However, we establish a separation between our approach and learning from offline expert demonstrations. In particular, we show that learning from offline demonstrations cannot in general achieve the soundness-completeness guarantees produced by our interactive learning approach.
\end{abstract}
\clearpage

\tableofcontents

\clearpage

\pagenumbering{arabic}

\section{Introduction}
 As large language models (LLMs) become increasingly capable of producing long and sophisticated reasoning
traces using chain-of-thought (CoT), a central bottleneck shifts from generating candidate solutions to reliably verifying them.
As formal verification methods do not directly extend to the diverse areas in which LLMs are deployed,
\emph{learned verifiers}---models that are used to check that LLM-generated outputs satisfy some desirable properties---offer a promising alternative. 
Indeed, augmenting LLMs with learned verifiers 
has already become an integral component of
  breakthroughs such as gold-level performance at the 2025 International Mathematical Olympiad~\citep{shao2025deepseekmath} and solutions to open-ended
  mathematical research problems~\citep{feng2026towards}.  Beyond checking for logical consistency, learned verifiers can also help in ensuring that LLM-generated outputs align with user preferences \cite{li2026aligning}, safety and fairness constraints \cite{hall2026guiding}, or domain-specific objectives ~\citep{ yun-etal-2025-med}. 
  
  Of particular interest are process reward models (PRMs), a prominent class of learned verifiers designed to evaluate not just the final output of a chain-of-thought reasoning trace, but all intermediate reasoning steps as well. 
PRMs designed for checking local correctness of each reasoning step have been used for complete-solution selection~\citep{uesato2022solving,lightman2023let}, inference-time search over partial solutions~\citep{ma2023letsreward,khalifa2025thinkprm}, and post-training \citep{uesato2022solving}. In general, using step-level supervision has been shown to improve reasoning correctness~\citep{uesato2022solving} and to outperform outcome-only supervision~\citep{lightman2023let}.
Despite this practical success, the study of PRMs has remained largely empirical and fundamental theoretical questions remain unanswered. What are the tight performance guarantees of learning verifiers that are capable of checking entire chains-of-thought step by step and that under uncertainty can optimally trade off being overly cautious and overly lenient?
And how can such learned verifiers quantifiably improve the correctness and reliability of chain-of-thought generation at inference time?

In this work, we aim to close this clear theoretical gap. 
We consider settings where reasoning traces are generated adaptively, possibly conditioned on prior feedback from the verifier, and may deviate significantly from any fixed distribution. Motivated by this challenge, we provide a new theoretical framework that models the verifier as an \emph{online learner}. The verifier observes a sequence of (problem statement, reasoning trace) pairs and is tasked to determine whether each step in the reasoning trace is correct.  
This formulation captures settings where the verifier must operate under minimal assumptions on its inputs, such as when the data distribution is shaped by the interaction between  the generator and the verifier. Such settings are not captured by prior learning-theoretic work \cite{balcan2025learning}, which studies chain-of-thought verifiers through a statistical learning lens.\looseness-1

We provide a mistake bound analysis that explicitly distinguishes between \emph{soundness errors} (accepting incorrect reasoning) and \emph{completeness errors} (rejecting correct reasoning), thereby allowing us to characterize the optimal behavior of the verifier under uncertainty. This distinction also captures a fundamental asymmetry in verifier-guided generation for high-stakes applications such as legal and medical domains~\citep{magesh2025hallucination,yun-etal-2025-med}. Accepting incorrect reasoning can cause the system to build on and amplify an error. By contrast, rejecting correct reasoning is often less harmful because the system can simply ask the generator to try again.

We further show that a learned chain-of-thought verifier with separate soundness and completeness guarantees can serve not only as a safeguard for LLM-generated outputs in high-stakes domains, but also as a means of \emph{improving} the inference-time accuracy and reliability of a generator (LLM). \footnote{We note that beyond improving accuracy, learned verifiers can play an important role in helping align a generator to a user’s values.  In particular, rather than logical correctness, the verifier could be learning which actions or statements the user deems acceptable vs improper, and then using this information to guide the generator to producing acceptable sequences.} Specifically, we provide a novel sufficient condition under which training a chain-of-thought verifier on the outputs of a weak, fixed generator (which might adapt to feedback from the verifier) can give a generator-verifier system, that  achieves substantially more accurate reasoning at inference time than the generator alone.\looseness-1

\subsection{Main Contributions} 

{\bf Online learnability of chain-of-thought verifiers.} Our initial focus is to study the online learnability of chain-of-thought verifiers while distinguishing between soundness errors (where the verifier accepts an incorrect reasoning step) and completeness  errors (where the verifier rejects a correct reasoning step). There is a direct tension between these two types of errors. If we want to ensure the soundness of a reasoning trace, we must also be overly conservative, rejecting any reasoning step we are even slightly uncertain of. 
We formalize this trade-off through a natural \emph{verification with a soundness budget} model that interpolates between two extremes: learning a perfectly sound verifier and treating both types of errors equally.  The learner's goal is to minimize the total number of mistakes while adhering to a given budget for soundness mistakes. 
If the budget is zero, then the learner must be perfectly sound, guaranteeing that any accepted reasoning trace is indeed correct. In this model, we ask the following fundamental learnability questions:

\vspace{7pt}

{\centering
\emph{(a) Is there an online chain-of-thought verification algorithm that minimizes the total number of mistakes, while adhering to a soundness mistake budget? \\
(b) Is there a natural notion of dimension that characterizes learnability in this setting?}
}

\vspace{7pt}

We fully resolve these questions in this work. We characterize the optimal tradeoff between soundness and completeness mistakes by introducing a new combinatorial complexity measure, the \textit{soundness-completeness Littlestone dimension}, which we prove precisely captures the complexity of online chain-of-thought verification. We provide matching upper and lower bounds, establishing that our guarantees are tight and instance-independent. We provide a clean proof of this result via an equivalence between two natural verification models: \emph{chain-of-thought verification}, where the learner must provide the location of the first faulty reasoning step, and \emph{prefix verification}, where the learner must only check the correctness of the last step in the prefix of a reasoning trace. This two-way reduction allows us to analyze the simpler prefix verification setting while ensuring that the tight bounds we provide directly apply to chain-of-thought verification as well. Beyond this abstract characterization, we compute the complete soundness-completeness Pareto frontier for several concrete verifier classes. Our examples highlight the need for carefully choosing the soundness budget, as allowing a single additional soundness mistake can eliminate an arbitrarily large number of completeness mistakes.\looseness-1

{\bf Improved performance at inference-time using verification.} Apart from establishing tight learnability results for chain-of-thought verifiers, we also study the consequences of these results for generation. In particular, we ask the following:

\vspace{7pt}

{\centering \emph{(c) How can a CoT  verifier with strong guarantees  improve the reliability of an LLM generator?\looseness-1}\par
}

\vspace{7pt}

\noindent We show that even when a generator has only a small probability of producing a correct next reasoning step, a learned verifier can be used to \emph{exponentially} amplify the probability of producing a fully correct reasoning trace by repeatedly filtering and combining candidate steps from one or more such weak generators. In this way, verifier-guided generation can obtain substantially stronger correctness guarantees than the generators achieve on their own, and can even combine different reasoning strategies across generators to solve problems beyond those that any single generator can solve.
Our result builds on and extends Littlestone's mistake-bound-to-PAC argument~\citep{littlestone:1989} to our setting, where samples are drawn from the stochastic process induced by the dynamic interaction between the generator and verifier. 
We further show how the error and abstention rates of the verifier-assisted generator can be bounded in terms of the soundness and completeness mistake bounds of the verifier. In particular, the rate of incorrect reasoning trace generation is governed by the verifier's soundness error, justifying stronger emphasis on limiting soundness mistakes. 

We further establish the crucial role of interaction in learning verifiers with favorable sample complexity by showing a separation from behavior-cloning. Our learner receives supervision on prefixes induced by its own predictions, whereas under behavior cloning the learner observes only interactions between the generator and the expert verifier. For the same sample size, we show that training the verifier using behavior cloning can increase the soundness error rate by a factor proportional to the VC-dimension of the verifier class. At a high level, we construct a history-dependent generator for which an initial completeness error shifts the generator to a new distribution containing a hidden incorrect step that is essentially absent from the behavior-cloning training data.

{\bf Training with offline data.} Our online learnability results assume that the verifier gets to observe the location of the first faulty step of each reasoning trace, which could be collected offline and presented sequentially to the learner. For improving the performance of a generator at inference time, we also prove that under stronger feedback, our system-level guarantees can be obtained without a human (expert verifier) in the loop, using only offline data (\emph{cf.} Remark \ref{rem:offline}). In particular, to train a verifier that improves a weak generator, it suffices that each problem admits a small set of correct reasoning traces, all of which are provided offline by a gold-standard reasoner. These traces can then be used to answer the queries to the expert verifier during training. Such training data may be more readily available than the online protocol, where the expert verifier needs to identify the location of the first mistake in each adaptively generated trace.

\subsection{Related Work} 
  \citet{balcan2025learning} consider chain-of-thought verification from a statistical (PAC) learning perspective. They show that if the (problem statement, reasoning trace) pairs are drawn from a fixed but unknown distribution,
  then with a sample complexity linear in the VC dimension it is possible to learn a verifier that with high probability correctly predicts the location of the first error in the reasoning trace (if one exists). On the other hand, learning a \emph{perfectly sound} verifier (i.e., one that never accepts an incorrect reasoning trace) requires sample complexity that is linear in the size of the class of verifiers. This lower bound motivates our study of online learning with a flexible tradeoff between soundness and completeness, which we show is sufficient to give performance guarantees for amplifying weak generators, bypassing the hardness result. 

  Our generator-verifier interaction is closely related to \emph{imitation learning}, where a learner’s decisions shape the distribution of states it subsequently encounters. Classical interactive (on-policy) methods such as DAgger address the resulting distribution shift by training on states induced by the learner’s own behavior \citep{ross2011reduction}. More recently, \citet{foster2024behavior} showed that, under suitable realizability and parameter-sharing assumptions, log-loss behavior cloning can match the bounds achieved by interactive methods, indicating that interaction does not help under their aggregate error objectives. We show that there is a qualitatively different separation that matters for learning verifiers. 
  Concretely, we show that on-policy feedback can suppress the  soundness-error component by a factor proportional to the VC dimension, while shifting the remaining error toward completeness mistakes. This asymmetric profile is crucial because soundness and completeness failures can have very different downstream consequences, and shows that interaction can provide a substantial advantage over behavior cloning that is invisible to a symmetric aggregate-error objective.\looseness-1

Our soundness--completeness formulation is also related to online learning with abstention~\citep{sayedi2010trading,zhang2016extended}, and more broadly to reliable learning, where one type of classification error is required to be absent or tightly controlled \citep{kalai2012reliable,kanade2014distribution}. In the online reliable learning setting, the  learner is tasked to either predict a label or abstain in each round, with the goal of minimizing the number of abstentions subject to a mistake budget. Despite the conceptual similarity, our chain-of-thought verification problem differs substantially in its prediction and feedback structure. A key technical contribution is our two-way reduction showing that, despite this structured feedback, chain-of-thought verification is equivalent to prefix verification (Theorems~\ref{thm:red1} and \ref{thm:red2}). Building on this reduction allows us to fully characterize the optimal soundness-completeness tradeoff of chain-of-thought verification using techniques inspired by online learning with abstentions, albeit requiring a problem-specific treatment. Beyond the soundness-completeness Pareto frontier, we also characterize optimal online learning under arbitrary linear costs for the two mistake types and show how these guarantees translate to downstream performance when the learned verifier interacts with a generator.

We discuss a more comprehensive overview of related work and connections to PRMs, prover-verifier interactions, theoretical foundations of post-training, and imitation learning in Section \ref{app:additional-related}.\looseness-1

\subsection{Structure of the Paper} 
In Section \ref{sec:overview} we provide an overview of our main results along with descriptions of our proposed algorithmic procedures and proof techniques. In Section \ref{sec:cot}, we present our main results regarding learning chain-of-thought verifiers in full detail. Specifically, Sections \ref{sec:reduction}-\ref{sec:lscmistakes} establish the tight characterizations for verification with a soundness budget and verification with linear costs, and Section \ref{sec:frontier-examples} computes exact Pareto frontiers for several concrete verifier classes. Section~\ref{sec:boost} shows how our online verifiers can be used to boost history-dependent generators, and Section~\ref{sec:bc-lower-bound} shows that interaction is necessary in general by proving a sample-complexity separation from behavior cloning. 
Finally, we close with a discussion and open questions in Section \ref{sec:conclusion}.

\section{Technical Overview}\label{sec:overview}

In this section, we provide an overview of our results and main technical insights. We start with a formal introduction to our model.\looseness-1

\subsection{Formal Definition of Chain-of-Thought Verification}
\noindent {\bf Problem statements, reasoning traces, and verifiers.} Let $X$ denote a set of problem statements written in natural language. For example, a problem statement could be a mathematical problem such as ``Find $y \geq 1$, such that $\sqrt{y - 1} = 4$'', a list of contributions claimed to be solved in a research paper, or a prompt such as ``Please give me a list of instructions for booking my conference trip''. Let $\Sigma$ denote a set of reasoning steps, where each step consists of a finite sequence of tokens, such as ``We first square both sides of the equation: $\left( \sqrt{y-1} \right)^2 = 4^2$'' or ``Book a United flight from Chicago to Sydney from 12/5 to 12/12''. A reasoning trace $\tau \in \Sigma^*$ is a finite sequence of reasoning steps.\looseness-1

We are interested in using a learned verifier that checks each step against some desirable properties such as mathematical correctness, safety, or preference alignment. 
We denote such a verifier as a function $h: X \times \Sigma^* \to \{\yes,  \no \}$ that, given a problem statement $x$ and a finite sequence of reasoning steps $\tau = (\tau_1, \tau_2, \dots, \tau_\ell)$, outputs $\yes$ if $\tau_\ell$ is a valid inference from $x$ and the preceding steps $\tau_{1:\ell-1}$, assuming the prefix $\tau_{1:\ell-1}$ is correct. Otherwise, $h$ outputs $\no$. In the case where $\tau_{1:\ell-1}$ contains a faulty reasoning step, we allow $h$ to output anything on $(x,\tau_\ell)$. 

Given a problem statement and a reasoning trace $(x, \tau= (\tau_1, \tau_2, \dots, \tau_L))$, we evaluate the correctness of the reasoning trace by ``running'' a verifier $h \in H$ on every prefix:
    \[
    h(x, (\tau_1)),\quad
    h(x, (\tau_1, \tau_2)),\quad
    h(x, (\tau_1, \tau_2, \tau_3)),\;\dots
    \]

\noindent We say that $h$ \emph{accepts} $(x, \tau)$ if 
$h(x,\tau_{1:\ell}) = \yes$ for all $\ell \in [L]$.
Else, we say that $h$ \emph{rejects} $(x, \tau)$ and outputs $ \min\{\, \ell \in[L] : h(x,\tau_{1:\ell})=\no \,\}$ as the location of the first incorrect step. 

Our framework relies on the following natural assumption (``inductive bias" in machine learning): the correctness of each reasoning step can be checked \emph{locally} given the context of the problem statement and preceding steps.\footnote{We note that this inductive bias was also considered in prior work on learning verifiers \cite{balcan2025learning} and learning with chain-of-thought feedback \cite{joshi2025theory}.}
Checking for local correctness here could be viewed as answering questions such as ``Is this step justified given the previous work?” or ``Is this action safe given the current plan?”. Concretely, we assume that there exists a target (unknown) verifier $h^*$ in a class of verifiers $H$ that always agrees with an oracle $\cO: X \times \Sigma^* \mapsto \{\yes, \no \}$ (e.g. a human expert) about the validity of each reasoning step.

\begin{remark}
The goal here is to emulate the behavior of the target verifier $h^* \in H$. The formalization is agnostic to certifying that each reasoning trace makes progress toward answering the given problem statement, which may be important in some applications (e.g. verifying mathematical proofs) more than others (e.g. verifying that no action in a plan is unsafe). Later we discuss how this issue can be addressed under additional assumptions on the human oracle (Remark \ref{rem:oracleassumption}).
\end{remark}

{\bf Online chain-of-thought verification.} In this setting a learner is asked to verify an arbitrary sequence of problem statements and reasoning traces over $T$ timesteps. At each timestep $t \in [T]$ the learner observes a pair $(x^{(t)},\tau^{(t)})$ that consists of a problem statement $x^{(t)} \in X$ and a reasoning trace $\tau^{(t)} \in \Sigma^L$\footnote{For simplicity, we assume that all reasoning traces observed by the learner will have length exactly $L$. Our results also extend to variable length reasoning traces (up to some finite upper bound $L$).}.  
The learner must decide whether $\tau^{(t)}$ is a correct reasoning trace for $x^{(t)}$, and if not, identify the location of the first incorrect step. Formally, at timestep $t$, the learner outputs a label $\hat{y}^{(t)}$ from a label set $\cY = \{1, \dots, L\} \cup \{ \infty \}$. Returning $i \in [L]$ indicates that the location of the first faulty step is $i$, whereas returning ``$\infty$'' indicates that all steps in the reasoning trace are correct. After the learner predicts, the oracle reveals the true label $y^{(t)} = \min\bigl(\{\, i\in[L] : \cO(x^{(t)},\tau^{(t)}_{1:i})=\no \,\}\cup\{\infty\}\bigr)$.

We can think of the learner playing according to an online verification algorithm,
$V_H : \bigl((X\times \Sigma^{L})\times \cY\bigr)^{*}\times (X\times \Sigma^{L}) \to \cY$ which, given the history so far and a new example, uses a class of verifiers, $H$, to output a prediction. We say that an online verification algorithm makes a mistake in round~$t$ if 
$\hat{y}^{(t)} \neq y^{(t)}$. Throughout, we distinguish between the following two types of mistakes:

\begin{enumerate}[itemsep=2pt, topsep=2pt, parsep=2pt, partopsep=2pt]
    \item \textbf{Soundness mistake}, which occurs when the learner predicts the location of the first faulty step ``too late''.
    Formally, the learner makes a soundness mistake in round~$t$ if $ \hat{y}^{(t)} > y^{(t)}$.\looseness-1
    \item \textbf{Completeness mistake},  which occurs when the learner predicts the location of the first faulty step ``too early''. Formally, the learner makes a 
    completeness mistake in round~$t$ if $ \hat{y}^{(t)} < y^{(t)}$.\looseness-1
\end{enumerate}

{\bf Online Prefix Verification.} 
Parts of our technical analysis (e.g., interaction between generator and verifier in \Cref{sec:interaction}) require online verification algorithms that can check a reasoning trace step-by-step. We term this online verification problem \emph{online prefix verification}. At a high-level, in prefix verification, the learner observes a sequence of \emph{prefixes} (partial reasoning traces) and is asked to verify whether the last step of each prefix is a valid inference from the previous reasoning steps. The learner then makes a \emph{prefix-}soundness mistake in the case of false-positive and a \emph{prefix-}completeness mistake in the case of false-negative.

Formally, the learner at timestep $t \in [T]$ is presented with an instance $z_{\text{pfx}}^{(t)} = (x^{(t)}, \tau_{1:\ell}^{(t)}) \in X \times \Sigma^{\leq L}$, with the promise that steps $\tau_{1:\ell-1}$ are correct, that is, $\cO(x^{(t)}, \tau_{1:i}^{(t)}) = \yes$ for all $i \in [\ell-1]$ . The learner must then decide whether step $\tau_\ell$ is correct or not and predict $\hat{y}_{\text{pfx}}^{(t)} = \yes$ or $\hat{y}_{\text{pfx}}^{(t)} = \no$ accordingly. Then the true label $y_{\text{pfx}}^{(t)}$ is revealed to the learner such that $y_{\text{pfx}}^{(t)} = \yes$ if and only if $\cO(x^{(t)}, \tau_{1:\ell}^{(t)}) = \yes$ and $y_{\text{pfx}}^{(t)} = \no$ otherwise.\footnote{One may view prefix verification as a structured binary classification problem, where the ``domain" consists of all prefixes of reasoning traces, where all but the last reasoning step are correct. However, this ``domain'' depends on the target function and would therefore be unknown to the learner. Also we note that our prefix verifier only outputs $\yes$ or $\no$, unlike the interactive verifiers of \cite{amit2024models, amit2025theory} that may pose a new question for the prover. \looseness-1 } In this setting, we will distinguish between the following two types of mistakes: 

\begin{enumerate}[itemsep=2pt, topsep=2pt, parsep=2pt, partopsep=2pt]
    \item \textbf{(Prefix) soundness mistake}, which occurs when the learner accepts an incorrect prefix. Formally, the learner makes a prefix soundness mistake at round~$t$ if $ \hat{y}_{\text{pfx}}^{(t)} = \yes$ but $ y_{\text{pfx}}^{(t)} = \no$.\looseness-1
    \item \textbf{(Prefix) completeness mistake},  which occurs when the learner rejects a correct prefix. Formally, the learner makes a prefix completeness mistake at round~$t$ if $ \hat{y}_{\text{pfx}}^{(t)} = \no$ but $ y_{\text{pfx}}^{(t)} = \yes$.\looseness-1
\end{enumerate}

\subsection{Online Chain-of-Thought Verification}\label{sec:onlineverification}

To highlight the challenge of guaranteeing soundness, we first consider finite classes of verifiers and show an exponential separation between verification algorithms that may incur both soundness and completeness errors and algorithms that are constrained to be sound (i.e. never make a soundness mistake). For the former, a simple majority-based algorithm achieves the optimal mistake bound $O(\log|H|)$, whereas any sound algorithm may incur a linear number of mistakes in the worst case.\looseness-1

\begin{restatable}{proposition}{propcompressed}\label[proposition]{prop:compressed}
    For any finite class of verifiers, $H$,  there exists an online chain-of-thought verification algorithm that makes at most $O(\log |H|)$ mistakes. Moreover, there exists a class of verifiers, $H'$, for which any sound chain-of-thought verification algorithm cannot make fewer than $|H'|-1$ mistakes.\looseness-1
\end{restatable}

\noindent Formal arguments are presented in Section \ref{app:exp-gap-finite}, where we further show that these bounds are tight. We note that depending on $H$, it may be possible to obtain sharper mistake bounds for learning sound verifiers (see Example \ref{example:puzzle} and further concrete examples in Appendix \ref{app:examples}).\looseness-1

\Cref{prop:compressed} highlights a sharp transition and motivates a more refined study of the trade-off between soundness and completeness mistakes, which we present in subsequent sections.

\subsubsection{Chain-of-Thought Verification with a Soundness Mistake Budget}
We will be interested in analyzing the trade-off between soundness and completeness mistakes in chain-of-thought verification when the learner is allowed to make a small number of soundness mistakes, $k \in \N$. Specifically, the learner's objective is to minimize the number of overall mistakes subject to adhering to the soundness mistake budget, $k$. As motivation, we observe that even with $k=1$, the verifier class used to show the $\Omega(|H|)$ lower bound in Proposition \ref{prop:compressed} can now be learned with no more than $1$ mistake overall (\emph{cf.} Proposition \ref{prop:rn} for a detailed proof).

We now introduce a new variant of the familiar notion of mistake trees from online learning (\emph{cf.} Appendix \ref{app:mistake-tree} for relevant background). Our combinatorial construction is conceptually related to mistake trees for online learning with abstention~\citep{sayedi2010trading,zhang2016extended}, but address a fundamentally different prediction problem and feedback model. Existing mistake-tree characterizations therefore do not directly apply to our problem, and it is not immediate how guarantees for one setting should translate to the other. To exploit the connection with binary online learning, we must first establish that online chain-of-thought verification is equivalent, in the relevant mistake-bound sense, to online prefix verification, where the learner only decides whether the final step of a correct prefix is valid (Theorem \ref{thm:red1}, together with the converse reduction of Theorem \ref{thm:red2}). We summarize this result below in \Cref{mainthm}. Only through this two-way reduction can we obtain a full characterization of the optimal soundness--completeness tradeoff for chain-of-thought verification.\looseness-1

{\bf Soundness-Completeness (SC)-mistake tree.} An SC-mistake tree  is a rooted binary tree, where internal nodes are examples in $X \times \Sigma^{\leq L}$. Every internal node, $z$, has exactly two children, $z_s$ and $z_c$. There are two types of edges: $(z,z_s)$ is a straight edge labeled $\no$ and $(z,z_c)$ is a curvy edge labeled $\yes$. We say that an SC mistake tree is \emph{shattered} by a verifier class $H$ if for any root-to-leaf path that traverses nodes $z_1, z_2,  \dots, z_d$ there exists a verifier $h \in H$ such that the label of edge $(z_i, z_{i+1})$ is $h(z_i)$ for all $i \leq d$.

\begin{definition}
    We say that an SC mistake tree is $(k, m)$-difficult if every root-to-leaf path that has at most $k$ straight edges is of length at least $m$.
\end{definition}
For example, \Cref{fig:sctree} shows a construction of an SC-mistake tree that is $(0, 2)$- and $(1, 1)$-difficult. For this example, the verifier class is $H=\{ h_i: \{x\} \times \{0,1\}^4\to \{0,1\} \; | \; i\in[4], \; h_i(\tau)=\tau_i \underset{\substack{j\in[4]\setminus\{i\}}}{\wedge}\neg \tau_j \}$. We have $X = \{x\}$ and $\Sigma = \{0, 1\}$, interpreting the output $0$ as $\no$ and $1$ as $\yes$.

\begin{wrapfigure}[13]{r}{0.3\textwidth} 
  \centering
  \includegraphics[width=0.28\textwidth]{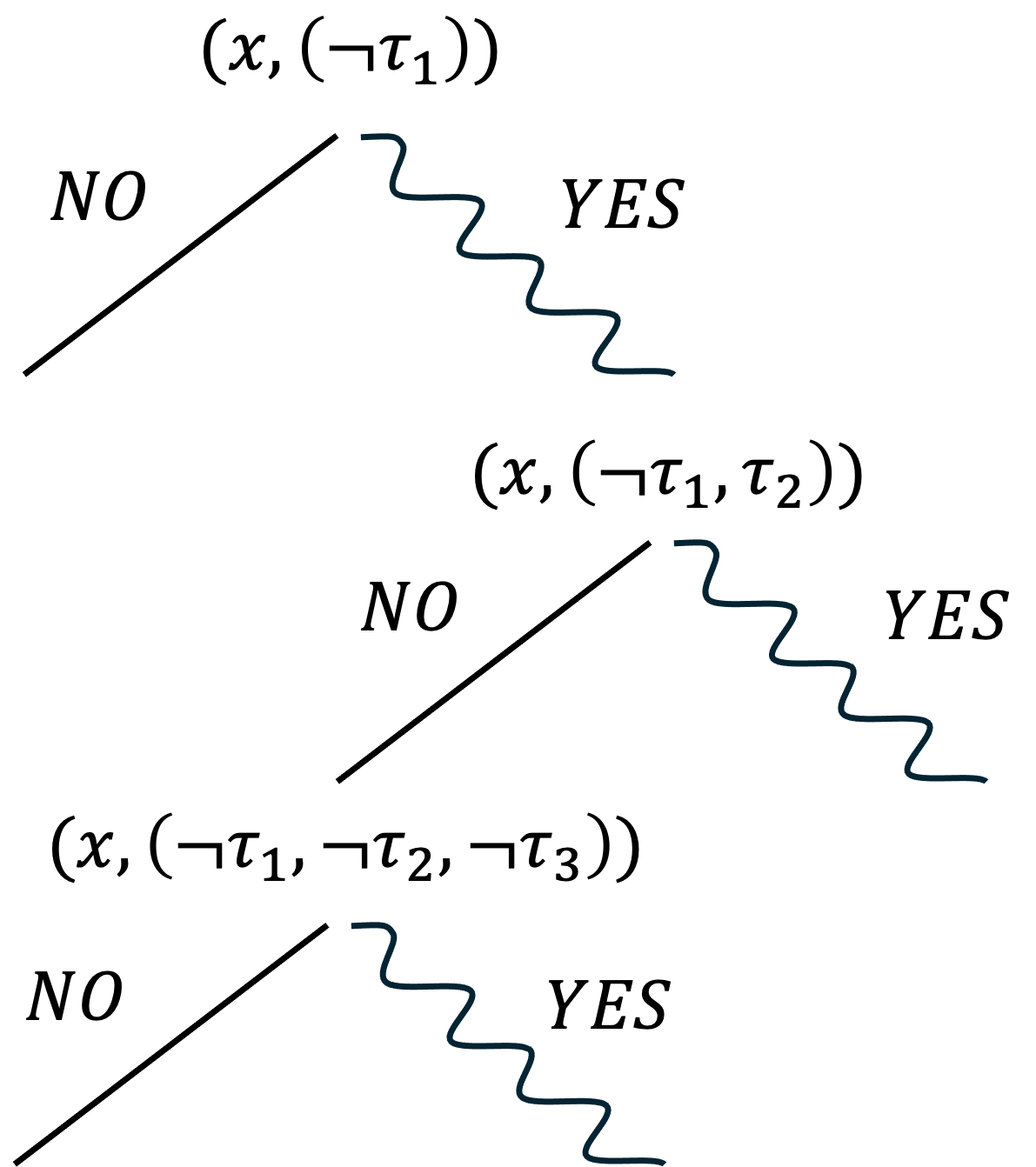}
  \caption{An SC-mistake tree.}\label{fig:sctree}
\end{wrapfigure}

\begin{definition}[SC-Littlestone dimension]
    For a class of verifiers, $H$, and a soundness mistake budget, $k \in \mathbb{N}$, the SC-Littlestone dimension is defined as $\text{SC-Ldim}(H, k) :=\sup \{m \in \N : \text{there exists a} \; (k, m)\text{-difficult mistake tree for}\;  H\}.$\looseness-1
\end{definition}
If $H = \varnothing$, then $\text{SC-Ldim}(H, k) = 0$ for all $k \in \N$. If for every $m \in \N$ there exists a $(k, m)$-difficult mistake tree for  $H$, then we let $\text{SC-Ldim}(H, k) = \infty$. We note that for a finite class of verifiers, we have the following bound on the SC-Littlestone dimension.

\begin{proposition}\label[proposition]{prop:finiteH}
    For a finite class of verifiers, $H$ with $|H|=n$, we have that $\text{SC-Ldim}(H,0)\leq n-1$ and more generally $\text{SC-Ldim}(H,k)\leq O\!\left(n^{1/(k+1)}\right)$. For $k\ge \lfloor \log_2 n\rfloor-1$, $\text{SC-Ldim}(H,k)\leq \lfloor \log_2 n\rfloor$.\looseness-1
\end{proposition}
\noindent The proof of \Cref{prop:finiteH} involves showing that $\text{SC-Ldim}(H,k) \leq \max\Bigl\{m\in\mathbb{N}:\ \sum_{i=0}^{k+1}\binom{m}{i}\le n\Bigr\}$ and follows similarly to that of \cite{zhang2016extended} so we omit it.

In the main result of this section, we show that the SC-Littlestone dimension fully characterizes learnability of chain-of-thought verifiers with a soundness mistake budget.

\begin{theorem}\label{mainthm}
    Consider a class of verifiers, $H$, a soundness mistake budget $k \in \N$ and $m \in \N$ such that $\text{SC-Ldim}(H, k) = m$. Then the following hold:
    \begin{enumerate}[label=(\alph*), itemsep=2pt, topsep=2pt, parsep=2pt, partopsep=2pt]
        \item There exists an online chain-of-thought verification algorithm that guarantees at most $k$ soundness mistakes and at most $m$ mistakes overall on any realizable sequence of examples.
        \item For any deterministic chain-of-thought verification algorithm that guarantees at most $k$ soundness mistakes, there exists a realizable sequence of examples that forces the algorithm to make at least $m$ mistakes overall.
    \end{enumerate}
\end{theorem}

{\bf Algorithmic approach.} We give an overview of our main algorithm for online chain-of-thought verification with a soundness mistake budget (\Cref{alg:cot-sc-soa}) that we use to prove \Cref{mainthm}(a). 
Our algorithm's prediction rule is driven by the SC-Littlestone dimension. Specifically, given a problem-trace pair $(x, (\tau_1, \dots, \tau_L))$, the learner examines each prefix, $(\tau_1, \dots, \tau_\ell)$ for $\ell \in [L]$. For each step $\ell$, the learner uses the SC-Littlestone dimension as a potential function in order to compare the worst-case residual difficulty of the two possible actions: (1) certifying that the prefix up to step $\ell$ is correct, or (2) declaring that the trace's $\ell$-th step is faulty. If wrong, prediction (2) would preserve the current soundness budget $k$ (completeness mistake), while prediction (1) would consume one unit of budget (soundness mistake). Accordingly, at each step $\ell$ of the reasoning trace, the algorithm evaluates the SC-Littlestone dimension with budget $k$ of the subspace of verifiers that predict $\yes$ at step $\ell$ and the SC-Littlestone dimension with budget $k-1$ of the subspace of verifiers that predict $\no$ at step $\ell$. The algorithm then predicts failure at the first index where the former option is at least as favorable as the latter; if no such index exists, it predicts $\infty$. When 
the soundness mistake budget $k=0$, the learner can no longer afford predicting a faulty step ``too late'', and therefore defaults to predicting the first position at which some verifier predicts failure.

After the index of the first faulty step is revealed, the learner updates the version space in the natural way. If it predicted a failure too late, then it made a soundness mistake: the trace had a faulty step earlier, so the learner restricts to verifiers that reject the true failing prefix and decrements the remaining soundness budget. If it predicted failure too early, then it made a completeness mistake: the predicted prefix was in fact still correct, so the learner restricts to verifiers that accept that prefix.\looseness-1

{\bf Proof Overview.} Our main technical insight for the upper bound is an \emph{online-to-online reduction} (\Cref{thm:red1}) from online chain-of-thought verification to online \emph{prefix} verification (summarized in \Cref{alg:red1}). We then construct an online prefix verification algorithm (\Cref{alg:sc-soa}) and prove a mistake bound based on the SC-Littlestone dimension (\Cref{thm:scub}). \Cref{thm:red1,thm:scub} establish the upper bound of \Cref{mainthm}(a). In the case of a completeness mistake, the online prefix verification algorithm does not observe whether the reasoning trace indeed contained a later error. 
This would suggest that online prefix verification may be a strictly harder problem than chain-of-thought verification. Nonetheless, the optimality of our approach comes from showing an online-to-online reduction in the opposite direction. Namely, we show that (under a mild assumption) an online chain-of-thought verification algorithm can be transformed to an online prefix verification algorithm (\Cref{thm:red2}). This together with a lower SC-Littlestone dimension based lower bound on prefix verification (\Cref{thm:sclb}) establishes the tight lower bound of \Cref{mainthm}(b). 

{\bf What can the Pareto frontier look like?} \Cref{mainthm} gives an abstract characterization, but the resulting frontier can behave very differently across verifier classes. In Section~\ref{sec:frontier-examples} we compute several exact examples. 

\subsubsection{Chain-of-Thought Verification with a Linear Cost Objective}

Motivated by downstream applications that may have a well-defined relative cost for each type of mistake, in this section we study the trade-off between soundness and completeness mistakes via a linear cost objective. Specifically, the learner incurs cost $\gamma_s \in \R_{\geq 0}$ for every soundness mistake and cost $\gamma_c \in \R_{\geq 0}$ for every completeness mistake. The learner's objective is to minimize the linear cumulative cost $\gamma_s M_s + \gamma_c M_c$, where $M_s, M_c$ denotes the number of soundness and completeness mistakes (respectively) that the learner makes over $T$ timesteps. As in \Cref{sec:tradingsoundnesscompleteness}, we hope to analyze this setting by defining the right notion of a mistake tree that will capture the optimal loss achievable by an online chain-of-thought verification algorithm. However, here the loss incurred by the learner depends on the \emph{real-valued} cost for each mistake, which must be properly embedded into the tree structure. We now introduce the notion of a WSC-mistake tree that augments the SC-mistake tree definition by including weighted edges designed to handle the added complexity of the linear cost objective.\looseness-1

{\bf Weighted Soundness-Completeness (WSC)-mistake tree.} A WSC-mistake tree is a rooted binary tree, where internal nodes are examples in $X \times \Sigma^{\leq L}$. Every internal node, $z$, has exactly two children: $z_s$ and $z_c$. There are two types of edges: $(z,z_s)$ is a straight edge labeled $\no$ and has weight $w(z,z_s) = \gamma_s$, whereas $(z,z_c)$ is a curvy edge labeled $\yes$ and has weight $w(z, z_c) = \gamma_c$. We say that a WSC mistake tree is \emph{shattered} by a verifier class $H$ if for any root-to-leaf path that traverses nodes $z_1, z_2, z_3, \dots, z_d$ there exists a verifier $h \in H$ such that the label of edge $(z_i, z_{i+1})$ is $h(z_i)$ for all $i \leq d$.\looseness-1

\begin{definition}[WSC-Littlestone dimension]
    For a class of verifiers, $H$, soundness mistake cost $\gamma_s \in \R_{\geq 0}$ and completeness mistake cost $\gamma_c \in \R_{\geq 0}$, the WSC-Littlestone dimension, $\text{WSC-Ldim}(H, \gamma_s, \gamma_c)$\footnote{For brevity, in the remainder of the paper we drop the dependence on $\gamma_s, \gamma_c$ and write $\text{WSC-Ldim}(H)$ instead.}, is defined as the supremum $W \in \R$ such that there exists a WSC mistake tree, where all root-to-leaf paths have cumulative weight at least $W$.
\end{definition}
If $H = \varnothing$, then $\text{WSC-Ldim}(H) = 0$. If for all $W \in \R_{\geq 0}$ there exists a WSC-mistake tree shattered by $H$, where all root-to-leaf paths have weight at least $W$, then we say that $\text{WSC-Ldim}(H) = \infty$. \looseness-1

We show that the WSC-Littlestone dimension characterizes the optimal cumulative loss among all deterministic chain-of-thought verification algorithms, which is the main result of this section.

\begin{theorem}\label{mainthm2}
    Consider a class of verifiers, $H$, a soundness mistake cost $\gamma_s \in \R_{\geq 0}$ and a completeness mistake cost $\gamma_c \in \R_{\geq 0}$. Then the following hold:
    \begin{enumerate}[label=(\alph*), itemsep=2pt, topsep=2pt, parsep=2pt, partopsep=2pt]
        \item There exists an online chain-of-thought verification algorithm that achieves at most $\text{WSC-Ldim}(H)$ cumulative loss on any realizable sequence of examples.
        \item For any deterministic chain-of-thought verification algorithm there exists a realizable sequence of examples that forces a cumulative cost of at least $W$ for every $W< \text{WSC-Ldim}(H)$.\looseness-1
    \end{enumerate}
\end{theorem}

In Sections \ref{sec:reduction} and \ref{sec:lscmistakes}, we provide a formal proof of this result, which is in the spirit of prior results in online realizable regression \cite{10.5555/3666122.3668058}. We summarize our algorithm in Algorithm \ref{alg:cot-wsc-soa}.\looseness-1

\subsection{Improving Weak Generators with Online Verification}\label{sec:interaction}

We now present an application of our online verification guarantees to the problem of improving a weak generator, or even a weak collection of generators, into a highly-accurate generator.  For this section, we assume our verifier-learning algorithm operates in the prefix verification model, where given a prefix $\tau_{1:\ell}$ such that $\tau_{1:\ell-1}$ is correct, it must output whether or not $\tau_\ell$ is correct.
Underlying these results is a generalization of classic mistake-bound-to-PAC conversions that involves repeated applications of the online algorithm to variations of a given instance and crucially uses the fact that our online learning guarantees apply to arbitrary sequences and not just to i.i.d. data. Specifically, the goal is to obtain a PAC-style guarantee for the stochastic process induced by the generator-verifier interaction. In this process, a problem statement is first drawn from a fixed underlying distribution over problems. The subsequent reasoning traces seen by the verifier are then generated adaptively through the interaction between the generator and the verifier, depending on the verifier’s previous decisions and the generator’s subsequent refinements. Thus, the relevant PAC guarantee is with respect to the full, induced process: over fresh sampled problem statements and the adaptive traces produced by the generator-verifier system on those problems.

We will first present a simplified version of a generator, whose output reasoning steps depend only on prefixes. Section~\ref{sec:boost} extends our results to \emph{history-dependent} generators, whose distribution over the next candidate step may depend on the full {\it interaction transcript} for the given problem instance.\looseness-1 

Formally, we define a generator as follows. For any problem instance $x \in X$, $G$ produces a first step $\tau_1$ of a reasoning trace from some distribution $G(x)$. Given $x$ and a first step $\tau_1$, $G$ produces a second step $\tau_2$ from some distribution $G(x,\tau_1)$. In general, given an instance $x$ and a partial reasoning trace $\tau_1, ..., \tau_\ell$ for $\ell<L$, $G$ selects a next step $\tau_{\ell+1}$ from a distribution $G(x,\tau_1, ..., \tau_\ell)$.  

We say that generator $G$ is ``$\alpha$-good'' for an instance $x$ if it has at least an $\alpha$ probability of producing a correct next reasoning step conditioned on the previous steps being correct, where correctness is defined with respect to {\em the same} labeling oracle $\mathcal{O}$ used for training the verifier.

\begin{definition}[$\alpha$-good generator]
    We say that a generator $G$ is {\it $\alpha$-good} for a problem $x \in X$ with respect to oracle $\cO$ if for all $0\leq \ell<L$ and all correct partial reasoning traces $\tau_1, ..., \tau_\ell$ (i.e., $\cO(x,\tau_{1:\ell'})=\yes$ for all $\ell'\leq \ell$) produced by $G$, we have $\Pr_{\tau_{\ell+1} \sim G(x,\tau_1,...,\tau_\ell)}[\cO(x,\tau_1,...,\tau_{\ell+1})=\yes] \geq \alpha.$\looseness-1
\end{definition}
This definition naturally extends to a distribution $D$ over problem instances as follows.
\begin{definition}[$(\alpha, \gamma)$-good generator]
    We say that a generator is \emph{$(\alpha,\gamma)$-good} for a distribution $D$ over problems in $X$ and with respect to oracle $\cO$, if it is $\alpha$-good for at least a $\gamma$ probability mass of the problems under $D$, that is, $\Pr_{x\sim D}[G\text{ is }\alpha\text{-good on }x]\ge\gamma$.
\end{definition}
\begin{figure}[t]
    \centering
    \includegraphics[height=1.5in]{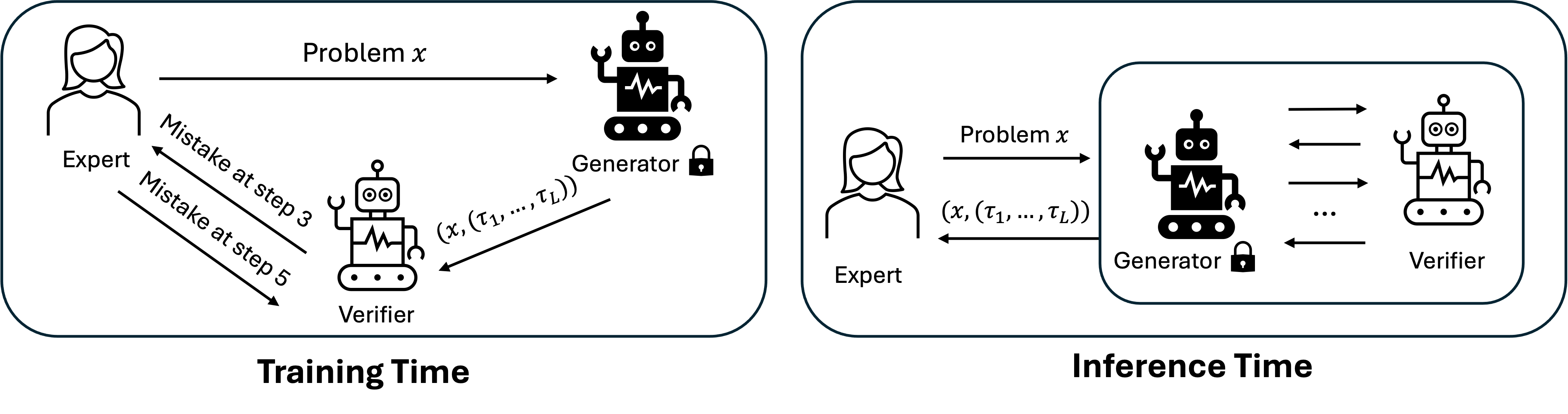}
    \caption{An illustration of the interaction between the fixed generator and the learned verifier at training and inference time.}\label{fig:pvdiagram}
\end{figure}

\begin{remark}\label{rem:oracleassumption}
    In the context of our results, the notion of an $\alpha$-good generator implies that given a problem $x \in X$, any correct prefix $\tau_1, ..., \tau_\ell$ generated by the generator can be extended to a fully correct reasoning trace $\tau_1, ..., \tau_L$. Thus, this notion is most relevant for settings where the oracle ${\cal O}$ is checking both for logical consistency and progress. For example, in grading a student's homework, each step of the argument should not be fallacious {\em and} should make progress towards the problem's solution.  In a planning domain, each step of the plan should both be feasible and helpful.
    We note that an $(\alpha,\gamma)$-good generator would require only that a $\gamma$ fraction of $x$'s to satisfy this property.\looseness-1
\end{remark}

We observe that if a generator is $\alpha$-good for instance $x$, then it might still have only an $\alpha^L$ probability of producing a fully correct reasoning trace for $x$, which is exponentially small in the length of the trace. Nevertheless, we show that a learned verifier can amplify this \emph{exponentially}, boosting any such weak generator into one that produces correct reasoning traces with high probability.\looseness-1

In fact, we can weaken our condition further by considering a {\em set} of generators ${\cal G} = \{G_1, ..., G_N \}$ to be $\alpha$-good for a problem instance $x$ if, for each step $\ell$, {\em at least one} of the generators $G_i \in {\cal G}$ has probability  $\geq \alpha$ of producing a correct next step, with perhaps different generators for different steps. For instance, this could model a setting where solving an instance $x$ requires first doing some calculus, then a proof by induction, and then a geometric argument, where different generators specialize in different kinds of reasoning.
\begin{definition}[good {\em set} of generators]
    We say that a {\em set} of generators ${\cal G}=\{G_1,...,G_N\}$ is {\it $\alpha$-good} for a problem instance $x \in X$ with respect to correctness oracle $\cO$ if for all $0\leq \ell<L$, for all correct partial reasoning traces $\tau_1, ..., \tau_\ell$ (i.e., $\cO(x,\tau_{1:\ell'})=\yes$ for all $\ell'\leq \ell$) generated by ${\cal G}$, there exists $i\in [N]$ such that $ \Pr_{\tau_{\ell+1} \sim G_i(x,\tau_1,...,\tau_\ell)}[\cO(x,\tau_1,...,\tau_{\ell+1})=\yes] \geq \alpha.$
    ${\cal G}$ is $(\alpha,\gamma)$-good for distribution $D$ if it is $\alpha$-good for at least a $\gamma$ probability mass of problems under $D$.
\end{definition}

We now present the main result of this section and outline our algorithmic procedure and proof techniques. Formal arguments are presented in \Cref{sec:boost}.

\begin{theorem}(Informal; \Cref{{thm:wrapper}})
 Under the assumption that $\cO=h^*\in H$ (there exists a perfect verifier in $H$), we can use an online verification algorithm $V_H$ together with an $(\alpha,\gamma)$-good set of generators ${\cal G}$ and, using a limited number of draws of problem instances $x \sim D$ and a limited number of calls to the labeling oracle $\cO$, learn a highly-accurate generator $\vhp$.\footnote{In the language of \citet{rohatgi2026taming}, we are learning a verifier that can successfully boost ${\cal G}$ using action-level rejection sampling. However, we are only aiming to generate correct reasoning traces, not to sample from the conditional.\looseness-1}  Specifically, with probability $1-\delta$, the generator $\vhp$ has the property that:\looseness-1
\begin{enumerate}[itemsep=2pt, topsep=2pt, parsep=2pt, partopsep=2pt, leftmargin=*]
    \item Given a problem instance $x$, $\vhp$ outputs either ``I don't know'' or a reasoning trace $\vtau$.
    \item The probability over $x\sim D$ that $\vhp$ outputs  ``I don't know'' is at most $(1-\gamma)+\epsilon_c + \epsilon_s + \delta$.\looseness-1
    \item The probability over $x\sim D$ that $\vhp$ outputs an incorrect reasoning trace is at most $\epsilon_s$,
\end{enumerate}
where $\epsilon_c$ and $\epsilon_s$ are functions of the completeness and soundness mistake bounds (respectively) of the verification algorithm $V_H$.  
In particular, if $V_H$ is a fully-sound verification algorithm then $\epsilon_s = 0$.
\end{theorem}

{\bf Algorithmic Approach and Proof Overview.} Given an online verification algorithm $V_H$, a set of $N$ generators ${\cal G}$, a value $\alpha$ such that ${\cal G}$ is $(\alpha,\gamma)$-good for some $\gamma>0$, and access to samples from $D$ and to the oracle $\cO$ for training, the algorithm to create $\vhp$ proceeds as follows.\looseness-1

First, as in the classic mistake-bound-to-PAC conversion of Littlestone \citep{littlestone:1989}, we draw two sets of examples $S_1,S_2$ from $D$, using $S_1$ to generate a series of candidate verifiers and $S_2$ to test and select from them.  However, unlike the classic reduction, processing an example $x$ will involve multiple calls to the verification algorithm and the labeling oracle.  Specifically, given $x \in S_1$, we run all generators $G_j\in {\cal G}$ on $x$ to generate candidate first-steps of a reasoning trace, and ask $V_H$ if any of them were correct. If all candidate first steps $\tau_{1j}$ were deemed incorrect, then we re-run all generators in ${\cal G}$ to generate new candidate first steps, repeating this process until either we have found a first-step deemed correct by $V_H$ or we time out after $O(\frac{1}{\alpha}\log \frac{NL}{\delta})$ tries, concluding that $x$ must not have been in the good set for ${\cal G}$. 
If we found at least one first step $\tau_1$ deemed correct by $V_H$, we fix it and now run all generators $G_j\in {\cal G}$ to generate candidate second-steps, and so on, repeating this process until either we time out at some step $\ell$ or we reach step $L$ with a completed reasoning trace. If we produce a completed trace, we ask the labeling oracle to verify it, and charge the verifier with a soundness mistake (making progress) if any step in that trace is incorrect.  On the other hand, if we time out, we ask the labeling oracle to check the verifier's responses, and update the verifier on its first mistake (completeness or soundness) if it made one.  Once we have completed processing an example, which might involve updating the verifier, we then move on to the next $x\in S_1$.\looseness-1

Once we have completed processing $S_1$, we then test all verifiers produced on $S_2$.  As usual, we may assume without loss of generality that $V_H$ is a ``conservative'' algorithm, only changing its current verifier when a mistake is made.  So, if $V_H$ has total mistake-bound $M$, we have at most $M$ candidate verifiers to test. One additional difference with the standard online-to-PAC conversion is that we care separately about soundness and completeness errors, so will need to account for this in the analysis.\looseness-1

Finally, we use the final candidate verifier to create the generator $\vhp$ by processing each new example $x$ as above, but without the labeling oracle (just assuming our verifier is accurate).  If our verifier times out, we output ``I don't know''; else we output the produced reasoning trace.\looseness-1

{\bf The need for interaction.} The above procedure trains the verifier on prefixes generated under its own evolving predictions. We now show that this on-policy supervision cannot in general be replaced by behavior cloning on interactions generated under the perfect verifier if we want to maintain good soundness/completeness tradeoffs. At a high level, this is because at inference time early completeness mistakes made by the verifier may lead to soundness mistakes later on. However, training using demonstrations from the perfect verifier will never expose such events occurring.\footnote{\cite{viano2026doesonpolicyinteractionhelp} show a different kind of imitation learning problem where interaction also provably helps.}\looseness-1

\begin{theorem}[Informal; \Cref{thm:bc-lower-bound}]
There is a verifier class $H$ with VC dimension $d$ and generator such that any possibly randomized and improper behavior-cloning learner trained on $m\gtrsim d$ expert verifier transcripts must incur either constant completeness error or soundness error $\Omega(d/m)$. In contrast, interactive training on learner-induced transcripts achieves, up to logarithmic and confidence factors, $\epsilon_s=\widetilde O(1/m)$ and $\epsilon_c=\widetilde O(d/m)$.
\end{theorem}

\begin{proof}[Proof Sketch]
Let $\Sigma$ and $\Sigma'$ be disjoint, with $|\Sigma|=n$ and $|\Sigma'|\gg m$. A target verifier rejects a random constant fraction of $\Sigma$ and one hidden point $\sigma^\star\in\Sigma'$. This class has VC and Littlestone dimension $\Theta(n)$ and admits an online learner with soundness and completeness mistake bounds $M_s=1$ and $M_c=O(n)$. The generator samples from $\Sigma$ with probability $\rho=\Theta(n/m)$ and otherwise from $\Sigma'$. However, after the verifier rejects a correct candidate, it enters a ``failure mode" and proposes $\sigma^\star$ with constant probability. Because the expert verifier never makes such a mistake, behavior-cloning data contains only ``normal-mode" samples, leaving most labels in $\Sigma$ and the identity of $\sigma^\star$ unseen. A learned verifier must therefore either reject many unseen points of $\Sigma'$, causing constant completeness error, or accept $\sigma^\star$ with constant probability. In the latter case, accepting many unseen points of $\Sigma$ directly causes soundness error $\Omega(\rho)$, while rejecting many instead causes a completeness mistake with probability $\Omega(\rho)$, triggers the ``failure mode", and again produces soundness error $\Omega(\rho)=\Omega(n/m)$. Yao's minimax principle covers randomized and improper learners. In contrast, interactive training avoids this failure by obtaining oracle labels on the trajectories induced by its own predictions. Substituting $M_s=1$ and $M_c=O(n)$ into \Cref{{thm:wrapper}} yields the desired rates.\looseness-1
\end{proof}
\begin{remark}\label{rem:offline}
\textbf{Training the verifier without online access to oracle queries.} Our primary goal in studying \emph{online} verification is to improve a weak generator that may \emph{adapt} to the verifier's feedback on problem statements drawn from a fixed distribution. We note that in some cases, this can be achieved using only offline access to correct reasoning traces. Consider the setting where for a given problem statement $x \in X$, there is only a small set of correct reasoning traces $\vtau^*(x) = \left\{
    \tau\in\Sigma^{L} :
    \cO(x,\tau_{1:\ell})=\yes \text{ for all } \ell\in[L]
    \right\}$. Then under Remark \ref{rem:oracleassumption} (i.e., for every prefix $\rho_{1:\ell} \in \Sigma^{\leq L}$, $\cO(x,\rho_{1:\ell})=\yes \; \Longleftrightarrow \; \exists \tau\in\vtau^*(x) \text{ such that } \rho_{1:\ell} = \tau_{1:\ell}$), a dataset $S=\left\{(x^{(i)},\vtau^*(x^{(i)}))\right\}_{i=1}^N$ can be used to answer the oracle queries made while training (in
\Cref{alg:process example,alg:test hypothesis}) on the sampled problem statements in $S$: a queried prefix is labeled $\yes$ if and only if it is a prefix of some trace in $\vtau^*(x^{(i)})$. Consequently, whenever the number of sampled instances satisfies the sample requirement of the online-to-PAC conversion (\emph{cf.} \Cref{thm:wrapper}),
the verifier can be learned online on arbitrary reasoning traces given by the generator using only an offline dataset, which might be easier to obtain in practice. We note that this condition is similar to that considered by \citet{balcan2025learning}, where a gold-standard reasoner can output a small number of correct reasoning traces. This setting does not contradict the behavior-cloning lower bound of Section~\ref{sec:bc-lower-bound} because the offline traces are used to answer oracle queries on prefixes adaptively induced by the learner and the generator. Thus, the labels may come from an offline dataset, but training remains interactive in this sense.\looseness-1
\end{remark}

\section{Online Chain-of-Thought Verification}\label{sec:cot}

We now begin presenting our results and proofs in full detail. This section provides a resolution to the overarching questions (a) and (b) highlighted in the introduction.

\subsection{Warm-up: Sound vs. Non-sound Chain-of-Thought Verification}\label{app:exp-gap-finite}

We start our analysis with a warm-up that highlights a clear distinction in terms of optimal mistake bounds between sound and non-sound chain-of-thought verification.

\begin{restatable}{proposition}{finiteclassbounds}\label{thm:finiteclassbounds}
    For any finite class of verifiers, $H$,  there exists an online verification algorithm that makes at most $O(\log |H|)$ mistakes. Moreover, there exists a class of verifiers such that any online learning algorithm must make at least $\Omega(\log |H|)$ mistakes in expectation.
\end{restatable}

\begin{proof}
    We consider the majority vote verification algorithm, which given example $(x,\tau_{1:L})$ sequentially verifies each subtrace $(x, \tau_{1:\ell})$. It accepts the entire reasoning trace only if a majority of verifiers consistent with the history so far (initially the full set $H$) accept each sub-trace $(x,\tau_{1:\ell})$, and else rejects at the first step where the majority fails to accept. Each time there is a mistake, we discard all inconsistent verifiers. We claim that this algorithm suffers at most $\log |H|$ mistakes. Indeed, if we make a soundness mistake and accept an incorrect $(x,\tau_{1:\ell})$ (either accept $(x,\tau)$ or say the first mistake is after step $\ell$), it must be that at least half of the surviving verifiers accept $(x,\tau_{1:\ell})$ and we can discard them. Similarly, if we wrongly mark $(x,\tau_{1:\ell})$ as incorrect, we again discard at least half the remaining consistent verifiers. The perfect verifier $h^*$ never makes a mistake and is therefore never discarded. Thus, the number of consistent verifiers after $M$ mistakes is at most $|H|/2^M$ and at least one. This implies $M\le \log |H|$ always.

    For the lower bound, consider the setting, where $X = \{x\}, \Sigma = \{0, 1\}$, and a correct proof is an unknown length-$L$ binary sequence, $b^*$. Then we can define the class of verifiers, $H = \{h_b : b \in \{0, 1\}^L\}$, where
    \begin{align*}
        h_b(x, \tau_{1:\ell}) = \begin{cases} 
            \yes \;\; \text{if} \; \tau_\ell = b_\ell \\ 
            \no \;\;\;\; \text{if} \; \tau_\ell \neq b_\ell
        \end{cases}
    \end{align*}
    The adversary can then force $\Omega(L
    ) = \Omega(\log|H|)$ mistakes in expectation as follows. Set $y^{(0)}$ = 1. At timestep $t$ present $\tau^{(t)} = (b^*[:y^{(t-1)}], \mathbf{1}_{L-t})$. Then set the location of the first incorrect step $y^{(t)} = y^{(t-1)}+1$ or $y^{(t)} = y^{(t-1)}+2$ with probability $1/2$ each. Then any verification algorithm will make a mistake with probability at least $1/2$ for at least $L/2$ timesteps.
\end{proof}

\begin{restatable}{proposition}{finiteclassboundssound}\label{thm:finiteHsound}
    For any finite class of verifiers, $H$, there is a sound online verification algorithm that makes at most  $|H|-1$ mistakes. Moreover, there exists a class of verifiers, $H$, for which any sound verification algorithm cannot make fewer than $|H|-1$ mistakes.
\end{restatable}

\begin{proof}
    A sound verification algorithm must accept $(x, \tau_i)$ only if {\it all} verifiers consistent with the history so far accept. If we incorrectly reject a correct trace $(x,\tau_i)$, then at least one of the consistent verifiers (other than $h^*$) rejects it and we discard all such verifiers. Thus, the size of consistent verifiers decreases by at least one on each mistake, so we can make no more than $|H|-1$ mistakes.

    For the lower bound, consider a setting with a single problem statement $X = \{x\}$ and $\Sigma$ such that for any $t \in \N$ the learner is given an example $(x, \tau_1^{(t)}, \dots, \tau_L^{(t)}) \in \{x\} \times \Sigma^L$. We let $\tau^* = (\tau_1^*, \dots, \tau_L^*) \in \Sigma^L$ be the only incorrect reasoning trace for problem instance $x$ and define the class of verifiers, $H = \{h_\sigma: \{x\} \times \Sigma^* \mapsto \{\yes, \no \} , \sigma \in \Sigma^L \; | \; h_\sigma(x, \tau) = \; \no \; \text{iff} \; \tau = \sigma\}$, 
    such that $h_{\tau^*} \in H$ is the only correct verifier. We now consider any sound verification algorithm and an adversary that chooses the realizable sequence of examples. For any reasoning trace $\tau_{1:L}^{(t)}$ the adversary chooses, there will always be one verifier $h^{(t)} \in H$ consistent with the examples seen so far that predicts $\no$ on $\tau_{1:L}^{(t)}$. Therefore, the sound verifier will always predict $L$ as the first faulty step for any such instance. An adversary will then reveal the true label, $\infty$, and $h^{(t)}$ will be the only verifier removed from the set of verifiers consistent with the history so far at this timestep. Therefore, the adversary can guarantee at least $|H|-1$ mistakes against the sound verifier. 
\end{proof}

Depending on  $H$, it may be possible to obtain sharper mistake bounds for learning sound verifiers compared to \Cref{thm:finiteHsound}. 
For example, in the following  (adapted from \cite{balcan2025learning}), we show that there is a sound online verification algorithm that achieves a $O(\log |H|)$ mistake bound. We provide additional examples in Appendix \ref{app:examples}.

\begin{example}[Sound verification of discrete puzzles] We consider the problem of verifying a discrete puzzle given an online sequence of (possibly incorrect) solutions to the puzzle. These puzzles can be encoded as a directed graph where the nodes correspond to puzzle states and edges define reasoning steps that correspond to moves or transitions between the states. The precise rules of the puzzle are encoded by some $h^*$  (to be learned)  from a class $H$ of verifiers defined over the graph. 
As a concrete example, consider a generalized {\it river-crossing} puzzle. The rules that will govern $h^*$ are as follows. A farmer must transport $n$  items across a river using a boat that only holds a limited number of them, subject to constraints on which items can be left alone with each other.  We model the puzzle as a directed graph $G=(V,E)$.  A state $s\in V$ is a vector $s\in\{0,1\}^{n+1}$ giving the river bank of each item and the farmer, where $0$ is the left bank and $1$ is the right bank. A directed edge $(s,s')\in E$ exists iff the move from $s$ to $s'$ satisfies all the constraints. The start and goal states are $s_0=(0, \ldots, 0)$ and $s_g=(1, \ldots, 1)$.
A solution is a valid path from $s_0$ to $s_g$. Now consider prompts $x=(V,E_0)$ where $E_0\subseteq E$ is a set of edges revealed in the prompt, and let a reasoning trace be a sequence of vertices, $\tau \in V^*$. We define a class of verifiers $H_L = \{h_{\tilde{E}}: \tilde{E} \in E\}$ for a fixed maximum length of traces $L$.
For any hidden edge set $\tilde{E}\subseteq E$, define a verifier $h_{\tilde{E}}$ that outputs $\yes$ on $(x=(V, E_0), \tau_{1:t})$ if 
(1) $t=1 \text{ and } \tau_1=s_0$, (2) $2\le t\le L-1 \text{ and } (\tau_{t-1},\tau_t)\in E_0\cup \tilde{E}$, or (3) $t=L,\; (\tau_{L-1},\tau_L)\in E_0\cup \tilde{E},\text{ and } \tau_L=s_g$. Otherwise it outputs $\no$. We will show that there is a sound online verification algorithm with a mistake bound of $|E^*|\le|E|=\log |H|= \log(2^{|E|})$. The online verifier maintains a list $\overline{E}$ of edges, initially empty, and accepts proofs for which the paths consist of edges from the set $\overline{E}\cup E_0$, where $E_0$ is the set of edges in the problem statement. Each time the verifier makes a mistake, it must have rejected an edge $e'$ in $E^*\setminus \overline{E}$, and it can simply update $\overline{E}$ to $\overline{E}\cup \{e'\}$.\label[example]{example:puzzle}
\end{example}

\subsection{Two-way Online-to-Online Reduction to Verifying Prefixes}\label{sec:reduction}
 In this section, we show that online chain-of-thought verification reduces to the problem of verifying \emph{prefixes} of reasoning traces. At a high level, in this setting the learner observes a sequence of prefixes of varying length and must predict whether the last step of those prefixes is correct.

We explicitly construct a chain-of-thought verification algorithm from a prefix verification algorithm $\cA_{\mathrm{pfx}}$, as shown in \Cref{alg:red1}.  Given a labeled example $\bigl((x,\tau_{1:\ell}), y\bigr)$ where $(x,\tau_{1:\ell}) \in X \times \Sigma^{\le L}$ and $y \in \{ \yes,\no \}$, the call $\cA_{\mathrm{pfx}}.\textsc{Update}((x,\tau_{1:\ell}), y)$ updates the internal parameters of $\cA_{\mathrm{pfx}}$. 

\begin{theorem}\label{thm:red1}
Let $H$ be a class of verifiers. Suppose there exists an online algorithm $\cA_{\mathrm{pfx}}$ for the prefix online verification problem that for class $H$ achieves soundness and completeness mistake bounds $M_s$ and $M_c$, respectively. Then there exists an online algorithm $\cA$ for the chain-of-thought verification problem with the same soundness and completeness mistake bounds.
\end{theorem}

\begin{proof}
    \sloppy We explicitly construct a chain-of-thought verification algorithm from a prefix verification algorithm $\cA_{\mathrm{pfx}}$, as shown in \Cref{alg:red1}.  
    We observe that algorithm $\cA$ makes a completeness mistake at timestep $t$, only if $\cA_{\mathrm{pfx}}(x^{(t)},\tau^{(t)}_{1:\hat{y}^{(t)}}) =\no$, when $\cO(x^{(t)},\tau^{(t)}_{1:\hat{y}^{(t)}}) = \yes$. In that case, Step \ref{step:comp} guarantees that  $\cA_{\mathrm{pfx}}$ also makes a completeness mistake at that timestep. Similarly, algorithm $\cA$ makes a soundness mistake at timestep $t$, only if $\cA_{\mathrm{pfx}}(x^{(t)},\tau^{(t)}_{1:y^{(t)}}) =\yes$, when $\cO(x^{(t)},\tau^{(t)}_{1:y^{(t)}}) = \no$. In that case, Step \ref{step:sound} guarantees that $\cA_{\mathrm{pfx}}$ also makes a soundness mistake at that timestep.

    Since the feedback passed into $\cA_{\mathrm{pfx}}$ is always consistent with the oracle $\cO$, the sequence of examples used to update $\cA_{\mathrm{pfx}}$ remains realizable by $H$, which concludes the proof.

    \begin{algorithm}[H]
\caption{Constructing $\cA$ using $\cA_{\mathrm{pfx}}$}\label{alg:red1}
\DontPrintSemicolon
\LinesNumbered
\SetNoFillComment
\SetKwComment{tcp}{\(\triangleright\) }{}

\KwIn{prefix chain-of-thought verification algorithm $\cA_{\mathrm{pfx}}$}

\For{$t \in [T]$}{
    Observe example $z^{(t)}=\bigl(x^{(t)}, \tau^{(t)}=(\tau^{(t)}_1,\ldots,\tau^{(t)}_L)\bigr)$\;
    
    Predict
    $\hat{y}^{(t)}=\min\bigl(\{\ell\in[L]:\cA_{\mathrm{pfx}}(x^{(t)},\tau^{(t)}_{1:\ell})=\no\}\cup\{\infty\}\bigr)$\;
    
    Receive correct label $y^{(t)}\in[L]\cup\{\infty\}$\;
    
    \uIf(\tcp*[f]{Completeness mistake}){$\hat{y}^{(t)} < y^{(t)}$}{
        $\cA_{\mathrm{pfx}}.\text{{\sc Update}}((x^{(t)},\tau^{(t)}_{1:\hat{y}^{(t)}}), \yes)$\nllabel{step:comp}\;
    }
    \uElseIf(\tcp*[f]{Soundness mistake}){$\hat{y}^{(t)} > y^{(t)}$}{
        $\cA_{\mathrm{pfx}}.\text{{\sc Update}}((x^{(t)},\tau^{(t)}_{1:y^{(t)}}), \no)$\nllabel{step:sound}\;
    }
}
\end{algorithm}
\end{proof}

We highlight that in the reduction used to prove \Cref{thm:red1}, when algorithm $\cA$ makes a completeness mistake $\hat y<y$, $\cA_{\mathrm{pfx}}$ is updated but only on the prematurely rejected prefix $(x, \tau_{1:\hat y})$, which is labeled $\yes$, and the actual location of the first faulty step $y$ is not passed to  $\cA_{\mathrm{pfx}}$ (Step \ref{step:comp}). 
This would suggest that prefix chain-of-thought verification is more difficult than chain-of-thought verification. However, somewhat surprisingly, via the following result we show that that is not the case. Specifically, under a mild assumption on the verifier class (there exists a reasoning step that every verifier in the class rejects) we give a reduction in the opposite direction.

\begin{theorem}\label{thm:red2}
Let $H \subseteq \{\yes, \no\}^{X \times \Sigma^*}$ be a class of verifiers, where there exists $F \in \Sigma$ such that for every $h \in H$, every $x \in X$, and every correct prefix $\tau_{1:\ell-1}$, $h(x, (\tau_{1:\ell-1}, F)) = \no$. Suppose there exists an online algorithm $\cA$ for the online verification model that for class $H$ achieves soundness and completeness mistake bounds $M_s$ and $M_c$, respectively. Then there exists an online algorithm $\cA_{\mathrm{pfx}}$ for the prefix chain-of-thought verification model with the same soundness and completeness mistake bounds.
\end{theorem}

\begin{proof}
    Given $\cA$ we will construct algorithm, $\cA_{\mathrm{pfx}}$ as shown in \Cref{alg:red2}.

\begin{algorithm}[t]
\caption{Constructing $\cA_{\mathrm{pfx}}$ using $\cA$}\label{alg:red2}
\DontPrintSemicolon
\LinesNumbered
\SetNoFillComment
\SetKwComment{tcp}{\(\triangleright\) }{}

\KwIn{chain-of-thought verification algorithm $\cA$}

\For{$t \in [T]$}{
    Observe example
    $z_{\mathrm{pfx}}^{(t)}=\bigl(x^{(t)}, \tau^{(t)}=(\tau^{(t)}_1,\ldots,\tau^{(t)}_\ell)\bigr)\in X\times \Sigma^{\le L}$\;

    Construct
    $z^{(t)}=\bigl(x^{(t)},(\tau^{(t)}_1,\ldots,\tau^{(t)}_\ell, F,\ldots,F)\bigr)\in X\times \Sigma^{L}$
    
    Compute $\hat{y}^{(t)} \leftarrow \cA(z^{(t)})$\;
    
    \eIf{$\hat{y}^{(t)} \le \ell$}{
        Predict $\hat{y}_{\mathrm{pfx}}^{(t)}=\no$\;
    }{
        Predict $\hat{y}_{\mathrm{pfx}}^{(t)}=\yes$\;
    }
    
    Receive correct label $y_{\mathrm{pfx}}^{(t)}\in\{\yes,\no\}$\;
    
    \eIf{$y_{\mathrm{pfx}}^{(t)}=\no$}{
        $\cA.\text{{\sc Update}}(z^{(t)}, \ell)$\nllabel{step:sound2}\;
    }{
        \eIf{$\ell< L$}{
        $\cA.\text{{\sc Update}}(z^{(t)}, \ell+1)$\nllabel{step:comp2-prefix}\;
    }{
        $\cA.\text{{\sc Update}}(z^{(t)}, \infty)$\nllabel{step:comp2-full}\;
    }
    }
}
\end{algorithm}

Now suppose that at some timestep $t$, $\cA_{\mathrm{pfx}}$ made a prefix soundness mistake, that is $\hat{y}_{\mathrm{pfx}}^{(t)} = \yes$ but $y_{\mathrm{pfx}}^{(t)} = \no$. Then it must have been that $\cA(z^{(t)}) > \ell$ and Step \ref{step:sound2} guarantees that $\cA$ will make a soundness mistake at this timestep too.

Suppose now that at some timestep $t$, $\cA_{\mathrm{pfx}}$ made a prefix completeness mistake, that is $\hat{y}_{\mathrm{pfx}}^{(t)} = \no$ but $y_{\mathrm{pfx}}^{(t)} = \yes$. Then it must have been that $\cA(z^{(t)}) \leq \ell$. Since we know that steps $\tau^{(t)}_1, \dots, \tau^{(t)}_{\ell-1}$ are correct, and that token $F$ is rejected by every verifier in the $H$, it must be that $\cO(z^{(t)}) = \ell+1$ and Steps \ref{step:comp2-prefix} and \ref{step:comp2-full} guarantee that $\cA$ will also make a completeness mistake at this timestep. 

Notice that in any case, the feedback provided to $\cA$ is always consistent with the oracle $\cO$, so the sequence of examples that $\cA$ observes remains realizable by the class, $H$.
\end{proof}

\remark{\Cref{thm:red1} and \Cref{thm:red2} establish the equivalence between chain-of-thought verification and prefix verification. Therefore, in what is to follow we will mostly focus on the latter, dropping the term ``prefix'' when it is clear from context.}

\subsection{Prefix Verification with a Soundness Mistake Budget}\label{sec:tradingsoundnesscompleteness}

In this section, we study the trade-off between soundness and completeness mistakes in prefix verification when the learner is allowed to make a small number of soundness mistakes, $k \in \N$. Specifically, the learner aims to minimize the number of overall mistakes while respecting the soundness mistake budget, $k$.

We recall the definition of an SC-mistake tree and argue that it fully  describes the strategy of an adversary against a deterministic prefix verification algorithm. Specifically, the adversary can ask the learner to verify internal nodes of the tree starting from the root. Say that for an internal node, $z$, the learner outputs $\no$. Then the adversary reveals the correct label, $(z, z_c) = \yes$ (forcing the learner to make a completeness mistake) and follows the curvy downward edge. If the learner outputs $\yes$, then the adversary reveals the true label $(z,z_s) = \no$ (forcing the learner to make a soundness mistake). Since the learning algorithm is deterministic, the interaction between learner and adversary forms a sequence of labeled examples in a path down the SC-mistake tree, where every curvy edge in the path corresponds to a completeness mistake and every straight edge corresponds to a soundness mistake that the learner made.

With the above interpretation of a shattered SC mistake tree, it follows that the SC-Ldim$(H, k)$ lower bounds the number of overall mistakes of any deterministic online verification algorithm that guarantees at most $k$ soundness mistakes.

\begin{restatable}{theorem}{sclb}\label{thm:sclb}
 Consider a class of verifiers, $H$ and any $k \in \N$. There exists a strategy of an adversary that defines a realizable sequence of reasoning traces, such that any deterministic algorithm that guarantees at most $k$ soundness mistakes can be forced to make at least SC-Ldim$(H, k)$ mistakes overall.
\end{restatable}

\begin{proof}
    By definition of the SC-Littlestone dimension, there must exist a $(k, m)$-difficult mistake tree for $H$. We consider a deterministic learner and a strategy of the adversary that presents the reasoning trace examples in the internal nodes of the tree, starting from the root. The adversary then reveals the true label of the instance: if the learner predicted the first faulty trace to be at some step $i = \no$ then the true label is set to $v_c = \yes$, otherwise it is set to $v_s = \no$. The adversary then follows the downward edge that corresponds to the true label. Since the  tree is full, the adversary can guarantee a mistake at every node. Since the learner is deterministic, her interaction with the adversary follows a root-to-leaf path, every curvy edge of which corresponds to a completeness mistake and every straight edge corresponds to a soundness mistake. The path must have at most $k$ straight edges, since the learner guarantees at most $k$ soundness mistakes. Since the tree is $(k, m)$-difficult, we also have that every such path has length at least $m$, which concludes the proof.
\end{proof}
We now present \Cref{alg:sc-soa}, which we show that for any soundness mistake budget $k \in \N$ guarantees at most SC-Ldim$(H, k)$ overall mistakes. For a class of verifiers, $H$, example $(x, \tau) \in X \times \Sigma^{\leq L}$, and label $y \in \{ \yes, \no \}$, we write $H[((x, \tau), y)] := \{ h \in H \; | \; h(x, \tau) = y\}$.

\SetAlgoNoEnd
\begin{algorithm}[t]
\caption{Prefix verification with at most $k$ soundness mistakes}
\label{alg:sc-soa}

\DontPrintSemicolon
\SetAlgoLined
\LinesNumbered
\SetNoFillComment
\SetKwComment{tcp}{\(\triangleright\) }{}
\SetKwInOut{KwIn}{Input}

\KwIn{ class of verifiers $H$; soundness mistake budget $k$.}

$H^{(1)} \gets H$\tcp*[r]{set initial version space}

\For{$t \in [T]$}{
  Observe example
  $z^{(t)}=\bigl(x^{(t)},\tau^{(t)}=(\tau^{(t)}_1,\ldots,\tau^{(t)}_{\ell^{(t)}}) \bigr) \in X \times \Sigma^{\leq L}$\;

Compute $\cY^{(t)} = \{ h(z^{(t)}) \; | \; h \in H^{(t)}\}$

  \uIf(\tcp*[f]{all verifiers accept}){$\cY^{(t)}= \{ \yes \} $}{Predict $\hat{y}^{(t)} \gets \yes $\nllabel{step:predictL+1}
  }
\uElseIf(\tcp*[f]{no budget or all verifiers reject}){$k=0 \textbf{ \emph{or} } \yes \not\in \cY^{(t)}$}{
    Predict $\hat{y}^{(t)} \leftarrow \no $\tcp*[r]{must reject}\nllabel{step:predictL}
  }
  \Else{
    Compute $m_c \gets \text{SC-Ldim}\!\left(H^{(t)}[(z^{(t)},\yes)],\,k\right)$

    Compute $m_s \gets \text{SC-Ldim}\!\left(H^{(t)}[(z^{(t)}, \no)],\,k-1\right)$
    
    \If(\tcp*[f]{minimize SC-Ldim of future version space}){$m_c \le m_s$}{
    Predict $\hat{y}^{(t)} \gets \no $\;
    }
    \Else{
      Predict $\hat{y}^{(t)} \gets \yes $\;
    }
    \nllabel{step:cases-sc}
  }

  Receive correct label $y^{(t)} \in \{\yes, \no \}$\;

  $H^{(t+1)} \gets H^{(t)}[(z^{(t)},y^{(t)})]$\tcp*[r]{update version space}

  \If{$\hat{y}^{(t)}=\yes \ \textbf{\emph{and}} \ y^{(t)} = \no $}{
    $k \gets k-1$\tcp*[r]{update soundness mistake budget}
  }
}
\end{algorithm}

\begin{restatable}{theorem}{scub}\label{thm:scub}
    Consider a class of verifiers, $H$, and a soundness mistake budget $k \in \N$ such that SC-Ldim$(H, k) = m$. Then \Cref{alg:sc-soa} guarantees at most $k$ soundness mistakes and at most $m$ mistakes overall on any realizable sequence of examples.
\end{restatable}

\begin{proof}

\begin{figure}[t]
    \centering
    \includegraphics[height=2in]{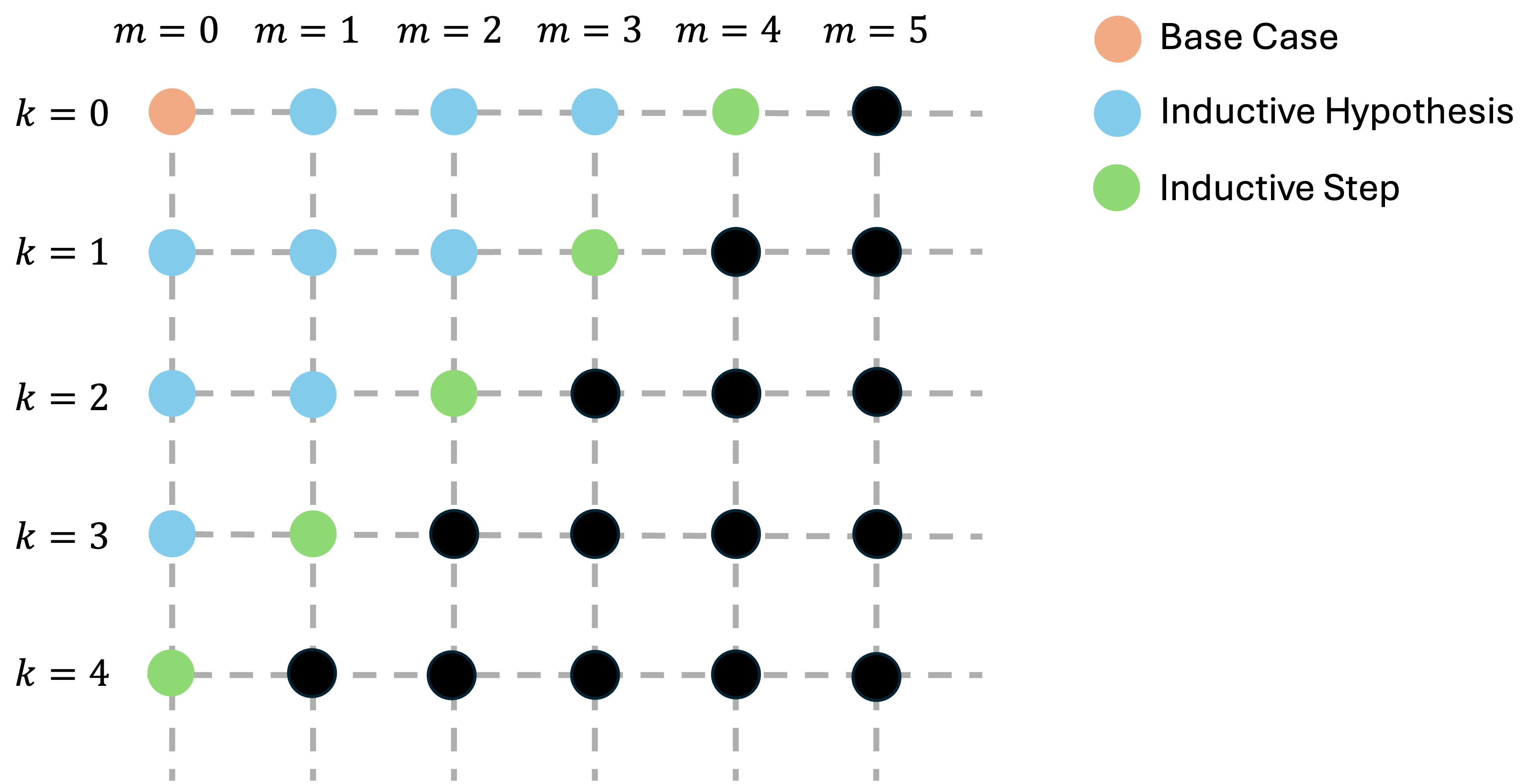}
    \caption{Visual representation of induction on $m+k$ used to prove \Cref{thm:scub}. }\label{fig:induction}
\end{figure}

    We proceed by  induction on $m + k$.
    
    \emph{Base case 1.} For $m = 0$ and $k \geq 0$, there does not exist a $(k, m)$-difficult SC mistake tree, so every $h \in H$ must predict unanimously on every instance. Therefore, Lines \ref{step:predictL+1} and \ref{step:predictL} of \Cref{alg:sc-soa} guarantee that the algorithm predicts correctly on every instance.
    
    \emph{Base case 2.} For $k=0$ and $m \geq 0$, we proceed with induction on $m$. We have already proved the base case for $k=m=0$. We now assume that the statement holds for SC-Ldim$(H, 0) = m-1$ and we will prove it for SC-Ldim$(H, 0)=m$. Since $k=0$, \Cref{alg:sc-soa} will always reject every instance. Consider the first timestep $t$, where the algorithm makes a mistake. Then it must be that $\text{SC-Ldim}(H^{(t)}[(z^{(t)}, \yes)], 0) \leq m-1$, as otherwise, we would be able to construct a $(0, m+1)$ difficult tree. From timestep $t+1$ and onwards the adversary can only present sequences that are realizable by $H^{(t)}[(z^{(t)}, \infty)]$ and by the inductive hypothesis the desired statement holds.
    
    \emph{Inductive Hypothesis.} We consider $k\geq 1$ and $m\geq 1$ and assume that the desired statement holds for all $k' \leq k, m' \leq m$, and $k'+m' \leq k+m-1$. We consider a class of verifiers, where $\text{SC-Ldim}(H)=m$. Consider the first timestep, $t$ where the algorithm makes a mistake. Then there must exist a verifier in $H^{(t)}$ that predicts $\yes$ and one that predicts $\no$.
    
    \emph{Inductive Step.} We first consider the case, where the predicted label is $\hat{y}^{(t)} = \yes$. It suffices to argue that $\text{SC-Ldim}(H^{(t)}[(z^{(t)}, \no)], k-1) \leq m-1$. For the sake of contradiction, assume the opposite. Then by Line \ref{step:cases-sc} of \Cref{alg:sc-soa}, we know that $m \leq \text{SC-Ldim}(H^{(t)}[(z^{(t)}, \no)], k-1) \leq \text{SC-Ldim}(H^{(t)}[(z^{(t)}, \yes)], k)$. This implies that there exists a $(k, m+1)$-difficult tree, which is a contradiction. When the predicted label is $\hat{y}^{(t)} = \no$, it suffices to prove that $\text{SC-Ldim}(H^{(t)}[(z^{(t)}, \yes)], k) \leq m-1$, which follows similarly to the argument above.
    From timestep $t+1$ and onwards the adversary is constrained to only presenting sequences of examples that are realizable by $H^{(t)}[(z^{(t)}, y^{(t)})]$ and since $\text{SC-Ldim}(H^{(t)}[(z^{(t)}, y^{(t)})]) \leq m-1$ for any $y^{(t)}$, the desired statement follows from the inductive hypothesis.
\end{proof}

\subsection{Examples: The Pareto Frontier of Concrete Verifier Classes}\label{sec:frontier-examples}

Theorem~\ref{mainthm} identifies the optimal total mistake bound under a soundness
budget $k$ with the combinatorial quantity $\SCLdim(H,k)$. In this section we compute the \emph{entire} Pareto
frontier $k\mapsto\SCLdim(H,k)$ for several natural verifier
classes. For an intersection-closed class of verifiers $H$ such as the discrete-puzzle
verifiers of \Cref{example:puzzle}, the frontier is flat and equal to the optimum $\log|H|$. So, the optimal algorithm already makes zero soundness mistakes. The classes below display 
other forms the frontier can take. Throughout, $\Ldim(H)$ denotes the Littlestone dimension
 of $H$ viewed as binary classifiers over
$X\times\Sigma^{\le L}$, and $H[(z,y)]:=\{h\in H\mid h(z)=y\}$ as in
Section~\ref{sec:tradingsoundnesscompleteness}. We first record three elementary facts that apply to all
frontiers.

\begin{lemma}\label{lem:general-shape}
For every class of verifiers $H$ and every $k\ge 0$:
\begin{enumerate}[label=(\alph*),itemsep=1pt,topsep=2pt]
\item $\SCLdim(H,k+1)\le\SCLdim(H,k)$;
\item $\SCLdim(H,k)\ge\Ldim(H)$;
\item if $\Ldim(H)=\ell<\infty$ and $k\ge\ell$, then $\SCLdim(H,k)=\ell$.
\end{enumerate}
\end{lemma}

\begin{proof}
(a) A root-to-leaf path with at most $k$ straight edges has at most $k+1$
straight edges, so every $(k+1,m)$-difficult tree is $(k,m)$-difficult.

(b) A complete mistake tree of depth $\Ldim(H)$ shattered by $H$ is an
SC-mistake tree all of whose root-to-leaf paths have length exactly
$\Ldim(H)$; it is therefore $(k,\Ldim(H))$-difficult for every $k$.

(c) By (b) it suffices to show $\SCLdim(H,k)\le\ell$. The Standard
Optimal Algorithm \cite{littlestone1988learning} makes at most $\ell$ mistakes in
total on any realizable sequence, hence in particular at most
$\ell\le k$ soundness mistakes, so it is a deterministic algorithm that
guarantees at most $k$ soundness mistakes and at most $\ell$ mistakes
overall. By Theorem~\ref{thm:sclb}, $\SCLdim(H,k)\le\ell$.
\end{proof}

Lemma~\ref{lem:general-shape} ahows that every frontier is nonincreasing,
bounded below by $\Ldim(H)$, and equal to $\Ldim(H)$ from budget
$\Ldim(H)$ onward. The examples below are therefore entirely about the
regime $0\le k<\Ldim(H)$. We note that part (c) slightly refines the last
statement of \Cref{prop:finiteH}. For any finite $H$ and at soundness budget $\Ldim(H)$ the frontier becomes equal to $\Ldim(H)$, which in turn
equals $\lfloor\log_2 |H|\rfloor$ for the hardest classes but
could be much smaller. 

\subsubsection{Singleton rejectors}\label{sec:sc-singleton}

We begin with singleton rejectors, which is a  class with which we prove the $|H|-1$ lower bound for sound
verification in Proposition~\ref{thm:finiteHsound}. Here we show that once a single soundness mistake is
allowed, this class becomes
learnable with a single mistake overall. Specifically, consider a set $D=\{\sigma_1,\dots,\sigma_n\}$ of $n\ge 2$ distinct single-step
examples $\sigma_i=(x_i,\tau_{1,i})\in X\times\Sigma$, and let
$R_n=\{h_\sigma:\sigma\in D\}$, where $h_\sigma(z)=\no\iff z=\sigma$.

\begin{proposition}\label{prop:rn}
$\Ldim(R_n)=1$, and
$\SCLdim(R_n,k)=
\begin{cases}
n-1, & k=0,\\
1, & k\ge 1
\end{cases}.$
\end{proposition}

\begin{proof}
\emph{$\Ldim(R_n)=1$.} 
We can use any $z\in D$ to label a complete shattered tree of
depth one. Conversely, for any example $z$ the set
$R_n[(z,\no)]$ is $\{h_z\}$ if $z\in D$ and empty otherwise. Therefore, a mistake tree of depth $2$ cannot be shattered by $H$.

\emph{$\SCLdim(R_n,0)\ge n-1$.} Consider a tree with nodes
$\sigma_1,\dots,\sigma_{n-1}$ such that the curvy edge at $\sigma_j$ leads to
$\sigma_{j+1}$ (to a leaf for $j=n-1$), and each straight edge child is a leaf.
The only root-to-leaf path with no straight edge is that of
length $n-1$, so the tree is $(0,n-1)$-difficult. The path
$\yes^{\,j-1}\no$ is realized by $h_{\sigma_j}$, which accepts
$\sigma_1,\dots,\sigma_{j-1}$ and rejects $\sigma_j$, and the all-$\yes$
path is realized by $h_{\sigma_n}$.\looseness-1

\emph{$\SCLdim(R_n,0)\le n-1$.} Consider any
$(0,m)$-difficult SC-mistake tree shattered by $R_n$. We show $m\le n-1$.
Let the all-$\yes$ path of the tree traverse internal nodes $z_1,\dots,z_m$. For each $t\le m$, the verifier consistent with the path of curvy edges to $z_t$ followed by a
straight edge is $h^{(t)}\in R_n$ with
$h^{(t)}(z_t)=\no$ and $h^{(t)}(z_s)=\yes$ for all $s<t$. The first
condition forces $h^{(t)}=h_{z_t}$, so $z_t\in D$. The second forces
$z_t\neq z_s$ for $s<t$. Hence $z_1,\dots,z_m$ are distinct elements of
$D$. Moreover, the all-$\yes$ path itself has a consistent verifier 
$h_{\sigma^*}$ with $\sigma^*\in D\setminus\{z_1,\dots,z_m\}$. Thus
$m\le|D|-1=n-1$.

$\SCLdim(R_n,k) = 1$ for $k\ge 1$. Follows from Lemma~\ref{lem:general-shape}(c) with $\Ldim(R_n)=1$.
\end{proof}

\subsubsection{Layered product verifiers}\label{sec:sc-layered}

Consider a setting where generating a solution to a problem naturally
decomposes into intermediate steps, where each step needs to be
verified against different criteria. For example, automated hiring
involves a multi-stage application process, where each component of an
application requires different criteria of evaluation (a CV requires a
keyword similarity checking model vs.\ a coding interview must be
checked for correctness). Similarly, in automated itinerary planning,
different components of the plan require different types of checks
(lodging requires checking availability of nearby hotels vs.\ activity
selection requires aligning with user preferences). Concretely, we
consider \emph{structured verifiers} that can be decomposed into a
product of step-specific \emph{mini-verifiers},
$H = H_1 \times H_2 \times \cdots \times H_L$, where each mini-verifier
$h_i \in H_i$ verifies the $i$th step of a reasoning trace, i.e.,
$h_i \colon X \times \Sigma^i \to \{\textsc{yes}, \textsc{no}\}$. Then
every verifier $h \in H$, where $h \colon X \times \Sigma^L \to
\{\yes, \no\}$, can be described as a list of $L$ mini-verifiers
$(h_1, \ldots, h_L)$:
\[
  h\bigl(x, \tau = (\tau_1, \ldots, \tau_L)\bigr)
  = \min \bigl\{\, i \in [L] : h_i\bigl(x, (\tau_1, \ldots, \tau_i)\bigr)
    = \textsc{no} \,\bigr\}.
\]
We argue that when $|H_i| \le n$, there exists a sound verifier that
makes at most $(n-1)L$ completeness (false-negative) mistakes. Consider
the algorithm that outputs the first step $i \in [L]$ for which at
least one function in $H_i$ that is consistent with the history so far
outputs \textsc{no}. Then for every completeness mistake the learner
can discard at least one mini-verifier. Therefore the learner can make
at most $(n-1)L$ completeness mistakes. We formalize this argument
below, extend it to any soundness budget $k$, and complement it with a
lower bound argument.

Another way to view the structured-verifier setup is to notice that it has two properties, (a)
the example space splits into disjoint parts on which the verifier
coordinates act independently, and, (b) part $i$
consists of the length-$i$ prefixes. We
 isolate the first part as a separate definition as it implies useful structure for $\Ldim$ and $\SCLdim$.

\begin{definition}[Layered product]
\label{def:layered-product}
Let $X\times \Sigma^{L} = Z_1 \sqcup \cdots \sqcup Z_L$ be
a partition of the domain  into disjoint \emph{layers}, and for each
$i \in [L]$ let
$H_i \subseteq \{\yes, \no\}^{Z_i}$ be a
class of functions on layer $i$. The \emph{layered product}
$H = H_1 \times \cdots \times H_L$ is the class of functions
$h = (h_1, \ldots, h_L)$ acting layerwise: $h(z) = h_i(z)$ for
$z \in Z_i$.
\end{definition}

More generally, the layers may be any partition of the example
space, with no sequential interpretation.
We first make a useful observation for this general setting. The Littlestone dimension is additive and the SC-Littlestone dimension is superadditive in terms of corresponding dimensions of the classes $H_i$.

\begin{restatable}{proposition}{additive}\label{prop:additive}
Let $H=H_1\times\cdots\times H_L$ be a layered product as defined above. Then,
\begin{enumerate}[label=(\roman*),itemsep=1pt,topsep=2pt]
\item $\Ldim(H)=\sum_{i=1}^L\Ldim(H_i)$,
\item $\SCLdim(H,k)\ge\sum_{i=1}^L\SCLdim(H_i,k)$ for every $k\ge 0$,
\item $\SCLdim(H,0)=\sum_{i=1}^L\SCLdim(H_i,0)$,
\end{enumerate}
where $\Ldim(H_i)$ and $\SCLdim(H_i,\cdot)$ are computed over trees whose nodes lie in layer ${Z}_i$.
\end{restatable}

Part (iii) shows that the sound learner from above, which makes at
most $(n-1)L$ completeness mistakes when $|H_i|\le n$, is exactly optimal
precisely when every mini-class is
a singleton-rejector class.  The full
frontier can be computed exactly, and it interpolates \emph{linearly}
between the two endpoints (see Example \ref{ex:structured-singleton-rejectors}), where each unit of soundness budget buys exactly
$n-2$ mistakes. We present the proof of Proposition \ref{prop:additive} in Section \ref{sec:omitted-proofs} in the appendix.

We next give a two-layer product where any algorithm minimizing the total number of mistakes will make all soundness mistakes, but where one can reduce the number of soundness mistakes down to 1 with only a slight penalty in total mistakes; however, reducing soundness mistakes to 0 requires a huge penalty.  This class is used later in Section \ref{sec:bc-lower-bound} to show a lower bound for behavior cloning.

Let $\Sigma$ and $\Sigma'$ be disjoint sets of single-step reasoning traces, with $|\Sigma|=n$ and $|\Sigma'|=n'$, where $n'$ may be arbitrarily larger than $n$. Fix $\beta\in(0,1/2)$, say $\beta\approx 1/3$, such that $b:=\beta n$ is an integer. This example is a layered product in the sense of Definition~\ref{def:layered-product}, with layers $Z_1=X\times\Sigma$ and $Z_2=X\times\Sigma'$. On the first layer, consider the class
\[ R_{n,b} := \{h_S:S\subseteq\Sigma,\ |S|=b\}, \qquad h_S(\tau)=\no \Longleftrightarrow \tau\in S. \] On the second layer, consider the singleton-rejector class $R_{n'}=\{r_{\sigma^\star}:\sigma^\star\in\Sigma'\}$, where $r_{\sigma^\star}(\tau)=\no$ if and only if $\tau=\sigma^\star$. Their layered product is \[ H_\beta := R_{n,b} \times R_{n'} = \left\{ h_{S,\sigma^\star} \;\middle|\; S\subseteq\Sigma,\ |S|=b,\ \sigma^\star\in\Sigma' \right\}, \] where $h_{S,\sigma^\star}(\tau)=\no$ if and only if $\tau\in S\cup\{\sigma^\star\}$.

Thus, every verifier rejects exactly $b$ traces from the first layer and exactly one trace from the second, potentially much larger, layer. The two layers create qualitatively different sources of uncertainty. On $\Sigma$, a perfectly sound learner may reject all $n-b$ correct traces before identifying the unknown set $S$. On $\Sigma'$, perfect soundness may require rejecting all $n'-1$ correct traces before identifying $\sigma^\star$, but allowing one soundness mistake eliminates those mistakes.

\begin{restatable}{proposition} {twoscalefrontier}\label{prop:two-scale-frontier} Suppose $n\ge 2b+1$, $n'\ge n-2b+2$. For the layered product $ H_\beta$ defined above, $\Ldim(H_\beta)=b+1$. Moreover, \[ \SCLdim( H_\beta,k) = \begin{cases} (n-b)+(n'-1), & k=0,\\ (n-b)+1, & 1\leq k\leq b,\\ b+1, & k\geq b+1. \end{cases} \] \end{restatable}

We present the proof of Proposition \ref{prop:two-scale-frontier} in Section \ref{app:examples} in the appendix.

\subsubsection{Subspace verifiers}\label{sec:sc-subspace}

We will now look at the subspace verifier class from Example~B.5 of
\citealt{balcan2025learning}, an infinite class for which the frontier turns out to be
maximally sharp. The class is \emph{not learnable} for every soundness budget below
its Littlestone dimension. Let $d\ge 2$ and
$d'\ge d-1$ be integers, let $X=\R^{d\times d'}$, $\Sigma=\R^{d}$, and
\[
\Hd=\bigl\{h_v:v\in\R^{d}\setminus\{0\}\bigr\},\qquad
h_v(x,\tau_{1:\ell})=\yes\iff\tau_\ell\in\spn(x,v),
\]
where $\spn(x,v)$ is the span of the columns of $x$ together with
$v$. Note $h_v=h_{cv}$ for any scalar $c\neq 0$.

\begin{restatable}{theorem}{subspace}\label{thm:subspace}
For $d\ge 2$ and $d'\ge d-1$,
\[
\SCLdim(\Hd,k)=
\begin{cases}
\infty, & 0\le k\le d-2,\\[2pt]
d-1, & k\ge d-1.
\end{cases}
\]
In particular, $\Ldim(\Hd)=d-1$, and the transition occurs exactly at
$k=\Ldim(\Hd)$. 
\end{restatable}

At a high-level, we first show (Lemma \ref{lem:normal}) that after rescaling, every query is equivalent to asking whether $v\in T\setminus S$
where $S\subset T$ and $\dim(T)=\dim(S)+1$. If $T\neq\mathbb{R}^d$, a $\yes$ answer places $v$ in the lower-dimensional space $T$. If $T=\mathbb{R}^d$, a $\no$ answer places $v$ in the hyperplane $S$. Thus, when $k\ge d-1$, the learner can ensure that every mistake reduces by at least one the dimension of the space in which $v$ may lie. Our construction (Lemma \ref{lem:difficult}) shows that the adversary can keep this reduction to one dimension per mistake. Therefore, $\Ldim(\Hd)=d-1$.

For $0\le k\le d-2$, consider a query asking whether $v$ lies outside a hyperplane $W$. As long as some possible $v$ lies in $W$, a sound learner must answer $\no$. If this answer is incorrect, the learner eliminates only the candidates in $W$, while the remaining candidates still span the full space. The adversary can therefore repeat such queries indefinitely. To reduce the dimension, the learner must answer $\yes$ and risk a soundness mistake. If such a mistake occurs, then $v\in W$, and Lemma \ref{lem:ldim-ub} reduces the problem to the same problem in one fewer dimension. Hence, each allowed soundness mistake reduces the dimension by exactly one.

This also explains the update $S \leftarrow S \cap\operatorname{span}(x,\tau_1)$
in Example B.5 of \cite{balcan2025learning}: each false acceptance reduces the dimension by one. Our lower bound shows that this rate is optimal. The complete proof appears in Section \ref{sec:subspace}.

\subsection{Prefix Verification with Linear Cost Objective}\label{sec:lscmistakes}

We show that any deterministic verification algorithm using $H$ must suffer cumulative loss at least $\text{WSC-Ldim}(H)$. 

\begin{restatable}{theorem}{scllb}\label{thm:scllb}
For any verifier class, $H$, and any deterministic learning algorithm, there exists a realizable sequence such that the learner must suffer cumulative loss at least $W$ for every $W < \text{WSC-Ldim}(H)$.
\end{restatable}

\begin{proof}
Fix $W < \mathrm{WSC\text{-}Ldim}(H)$. By definition, there exists a WSC mistake tree $\cT$ shattered by $H$ such that every root-to-leaf path in $\cT$ has cumulative weight at least $W$.
We consider a deterministic learning algorithm, $V_H$ and we define an adaptive adversary that interacts with $V_H$ as follows. Starting at the root of $\cT$, at each internal node the adversary presents the label of the node $(x,\tau)$ to the learner and observes the learner’s prediction $\hat{y} \in \cY$. At every node, the adversary can choose an outgoing edge inconsistent with the learner’s prediction and induce loss equal to the edge's weight. To see this, we consider the following cases. 
    
    First, we consider the case where the learner's prediction is $\hat{y} = \yes$. From the definition of a WSC mistake tree, we know that there exists a straight edge with weight $\gamma_s$ and label $\no$ that the adversary can choose to force the learner to make a soundness mistake and incur cost $\gamma_s$. We now consider the case, where the learner's prediction is $\hat{y} = \no$. Again, from the definition of a WSC-mistake tree, there must exist a curvy edge with weight $\gamma_c$ and label $\yes$ that the adversary can choose to force the learner to make a completeness mistake and incur cost $\gamma_c$. The adversary proceeds by presenting the example to which the chosen outgoing edge leads in the tree $\cT$ and repeats the process of finding the corresponding true label.
     
    The learner is deterministic, so this sequence of examples will correspond to a root-to-leaf path in $\cT$ and each edge will correspond to a mistake, where the loss of the learner is at least the weight of the edge. 
    By definition of $\cT$, however, every root-to-leaf path has cumulative weight at least $W$. Lastly, since $\cT$ is shattered by $H$, there must exist an $h \in H$ that is consistent with the sequence of examples presented to the learner, which ensures realizability.
\end{proof}

 We next present and analyze the performance of \Cref{alg:SCL-soa}. At a high-level, given a problem statement and a reasoning trace, the algorithm predicts the first faulty step that minimizes the worst-case sum of the immediate loss and the complexity of the future version space, as captured by the WSC-Littlestone dimension. We denote the immediate loss for a learner's prediction $\hat{y}$ and a true prediction $y$ as $\ell(\hat{y}, y) \in \{0, \gamma_s, \gamma_c\}$.

\begin{algorithm}[t]
\caption{Prefix verification with linear cost objective}\label{alg:SCL-soa}
\DontPrintSemicolon
\LinesNumbered
\SetNoFillComment
\SetKwComment{tcp}{\(\triangleright\) }{}
\KwIn{class of verifiers $H$, mistake costs $\gamma_s, \gamma_c \in \mathbb{R}_{\ge 0}$.}

$H^{(1)} \leftarrow H$\tcp*[r]{set initial version space}

\For{$t \in [T]$}{
    Observe example $z^{(t)} = \bigl(x^{(t)}, \tau^{(t)} = (\tau_1^{(t)}, \ldots, \tau_{\ell^{(t)}}^{(t)})\bigr)$\;

    Compute $\cY^{(t)} = \{ h(z^{(t)}) \; | \; h \in H^{(t)} \}$\tcp*[r]{set of possible outputs}
    
    \eIf(\tcp*[f]{all consistent verifiers agree}){$|\cY^{(t)}|=1$}{
        Predict $\hat{y}^{(t)} \in \cY^{(t)}$\;
    }{
        Compute $m_c = \gamma_c + \mathrm{WSC\text{-}Ldim}\bigl(H^{(t)}[(z^{(t)}, \yes)]\bigr)$\;

        Compute $m_s = \gamma_s + \mathrm{WSC\text{-}Ldim}\bigl(H^{(t)}[(z^{(t)}, \no)]\bigr)$\;

        \If(\tcp*[f]{minimize WSC-Ldim of future version space}){$m_c \le m_s$}{
            Predict $\hat{y}^{(t)} \gets \no $\;
            }
        \Else{
            Predict $\hat{y}^{(t)} \gets \yes $\;
            }
    \nllabel{step:cases-wsc}
  }
    }

    Receive correct label $y^{(t)} \in \cY$\;
    
    Set $H^{(t+1)} \leftarrow H^{(t)}[(z^{(t)}, y^{(t)})]$\tcp*[r]{update version space}
\end{algorithm}

\begin{restatable}{theorem}{sclsoaub}
    For any class of verifiers, $H$, \Cref{alg:SCL-soa} achieves cumulative loss at most $\text{WSC-Ldim}(H)$ on any realizable sequence of examples.\label{thm:wscl-ub}
\end{restatable}

\begin{proof}
We will show that for every timestep $t$,
\begin{equation}
    \ell(\hat{y}^{(t)}, y^{(t)}) \leq \text{WSC-Ldim}(H^{(t)}) - \text{WSC-Ldim}(H^{(t+1)}).\label{eq:wsc}
\end{equation}
Summing over $t$ gives $\ell(\hat{y}^{(t)}, y^{(t)}) \leq \text{WSC-Ldim}(H) - \text{WSC-Ldim}(H^{(T)})\le \text{WSC-Ldim}(H)$ as desired.

Fix $t$ and write $z = z^{(t)}$, $A = \text{WSC-Ldim}(H^{(t)}[(z,\textsc{yes})])$, and $B = \text{WSC-Ldim}(H^{(t)}[(z,\textsc{no})])$. We use two elementary facts about the WSC-Littlestone dimension.

We first show monotonicity, i.e., if $H' \subseteq H''$ then $\text{WSC-Ldim}(H') \le \text{WSC-Ldim}(H'')$. Any WSC-mistake tree shattered by $H'$ is also shattered by $H''$, since a realizer in $H'$ for a root-to-leaf path lies in $H''$. Hence the supremum defining $\text{WSC-Ldim}(H')$ is taken over a subfamily of the trees available for $H''$.

Next, we show that if both $H^{(t)}[(z,\textsc{yes})]$ and $H^{(t)}[(z,\textsc{no})]$ are nonempty, then
$$\text{WSC-Ldim}(H^{(t)}) \;\ge\; \min\{\, \gamma_s + B,\ \gamma_c + A \,\}.$$
To see this, fix any $W < \min\{\gamma_s + B,\ \gamma_c + A\}$. Choose WSC-mistake trees $T_{\textsc{no}}$ shattered by $H^{(t)}[(z,\textsc{no})]$ and $T_{\textsc{yes}}$ shattered by $H^{(t)}[(z,\textsc{yes})]$ whose root-to-leaf paths all have weight at least $W - \gamma_s$ and $W - \gamma_c$ respectively (possible by definition of $A, B$). Form the tree rooted at $z$ whose straight edge (label $\textsc{no}$, weight $\gamma_s$) leads to $T_{\textsc{no}}$ and whose curvy edge (label $\textsc{yes}$, weight $\gamma_c$) leads to $T_{\textsc{yes}}$. This tree is shattered by $H^{(t)}$. Indeed, a path descending the straight edge is realized by some $h \in H^{(t)}[(z,\textsc{no})] \subseteq H^{(t)}$ with $h(z) = \textsc{no}$ matching the edge label, and symmetrically for the curvy edge. Every root-to-leaf path has weight at least $\min\{\gamma_s + (W-\gamma_s),\ \gamma_c + (W-\gamma_c)\} = W$. Letting $W$ increase up to $\min\{\gamma_s + B,\ \gamma_c + A\}$ gives the desired bound.

We will now establish (\ref{eq:wsc}). If all verifiers in $H^{(t)}$ agree on $z$, \Cref{alg:SCL-soa} predicts their common value. So $\ell(\hat y^{(t)}, y^{(t)}) = 0$ and $H^{(t+1)} = H^{(t)}$ and (\ref{eq:wsc}) holds. Otherwise we apply the above lower bound on $\text{WSC-Ldim}(H^{(t)})$. Consider the two possible cases.

\begin{itemize}
    \item Correct prediction ($\hat y^{(t)} = y^{(t)}$): then $\ell(\hat y^{(t)}, y^{(t)}) = 0$ and $H^{(t+1)} = H^{(t)}[(z, y^{(t)})] \subseteq H^{(t)}$, so by monotonicity $\text{WSC-Ldim}(H^{(t+1)}) \le \text{WSC-Ldim}(H^{(t)})$, giving (\ref{eq:wsc}).
    \item Mistake ($\hat y^{(t)} \ne y^{(t)}$): recall Line 10 predicts $\textsc{no}$ when $m_c \le m_s$ and $\textsc{yes}$ otherwise, where $m_c = \gamma_c + A$ and $m_s = \gamma_s + B$. If $\hat y^{(t)} = \textsc{no}$ then $m_c \le m_s$, and a mistake forces $y^{(t)} = \textsc{yes}$, so $\ell(\hat y^{(t)}, y^{(t)}) = \gamma_c$ and $H^{(t+1)} = H^{(t)}[(z,\textsc{yes})]$. Hence $\ell(\hat y^{(t)}, y^{(t)}) + \text{WSC-Ldim}(H^{(t+1)}) = \gamma_c + A = m_c = \min\{m_c, m_s\}$.
If instead $y^{(t)} = \textsc{yes}$ then $m_s < m_c$, a mistake forces $y^{(t)} = \textsc{no}$, and analogously the left side equals $\gamma_s + B = m_s = \min\{m_c, m_s\}$. In either case the left side of (\ref{eq:wsc}) equals $\min\{m_c, m_s\} = \min\{\gamma_c + A,\ \gamma_s + B\}$, which is at most $\text{WSC-Ldim}(H^{(t)})$.
\end{itemize}
This establishes (\ref{eq:wsc}) in all cases and completes the proof.
\end{proof}
We note that when $\gamma_s=\gamma_c = 1$, the WSC-Littlestone dimension recovers the definition of the Littlestone dimension. Moreover, when $\gamma_c = 1$ and $\gamma_s = d$, any finite class $H$ has $\text{WSC-Ldim}(H) \leq O\left( \frac{d}{\log d} \log(|H|) \right)$. This recovers a previous result by \citet{helmbold:2000} on online binary classification with linear costs and finite verifier classes. We provide a formal argument in Proposition \ref{prop:zerolocation}.\looseness-1

\subsection{Examples: The Weighted Frontier of Concrete Verifier Classes}
\label{sec:wsc-examples}

We will now compute the $\WSCLdim$ for a few concrete examples of verifier classes.
A useful property is the following recursive characterization, which is
immediate from the optimality of Algorithm~\ref{alg:SCL-soa}
(Theorem~\ref{thm:wscl-ub}) and the matching lower bound
(Theorem~\ref{thm:scllb}). Writing $H[(z,y)]:=\{h\in H:h(z)=y\}$, and
setting $\WSCLdim(H)=0$ whenever all $h\in H$ agree on every relevant
example,
\begin{equation*}
\label{eq:wsc-recursion}
\WSCLdim(H)\;=\;\sup_{z\in\text{DIS}(H)}\;\min\!\Big(\gs+\WSCLdim\big(H[(z,\no)]\big),\;\;
\gc+\WSCLdim\big(H[(z,\yes)]\big)\Big),
\end{equation*}
where the supremum ranges over examples $z$ for which both $H[(z,\yes)]$ and
$H[(z,\no)]$ are nonempty. At each such $z$ the learner
incurs cost $\gs$ and restricts the verifier class to $H[(z,\no)]$, or incurs cost $\gc$ and restricts the class to
$H[(z,\yes)]$, taking whichever is reduces WSC-Ldim the most (the $\min$).

We next prove the following two elementary facts.

\begin{lemma}
\label{lem:wsc-boundary}
For any verifier class $H$, if $\min\{\gs,\gc\}=0$ then $\WSCLdim(H)=0$.
\end{lemma}

\begin{proof}
If $\gs=0$, the learner that predicts $\yes$ on every example never makes a
completeness mistake. Every mistake it makes is a soundness mistake, of
weight $\gs=0$. Hence its cumulative cost is $0$ on every sequence, so
$\WSCLdim(H)=0$ by Theorem~\ref{thm:scllb}. The case $\gc=0$ follows similarly.
\end{proof}

When soundness errors are at least as costly as
completeness errors and the class is intersection-closed, there is never any
reason to risk a soundness mistake. If the class $H$ admits a sound learner with at most $\Ldim(H)$ completeness mistakes, the weighted SC-Littlestone dimension is
$\gc\,\Ldim(H)$. The interest lies entirely in classes where $\gs<\gc$, or where the
sound learner is far from the Littlestone optimum.

\begin{lemma}[Sandwich bound]
\label{lem:wsc-sandwich}
If $\Ldim(H)=\ell<\infty$ then
$\ell\,\min(\gs,\gc)\le\WSCLdim(H)\le\ell\,\max(\gs,\gc)$.
Moreover, if $H$ admits a sound learner making at most $\ell$ completeness
mistakes and no soundness mistakes (as for the intersection-closed classes
of Example~\ref{example:puzzle}) then for all $\gs\ge\gc$,
\[
\WSCLdim(H)=\gc\,\Ldim(H).
\]
\end{lemma}

\begin{proof}
A complete mistake tree of depth $\ell$ shattered by $H$ is also a
WSC-mistake tree. Each of its root-to-leaf paths has $\ell$ edges, each of
weight at least $\min(\gs,\gc)$, giving the lower bound. For the upper bound,
the Standard Optimal Algorithm \cite{littlestone1988learning} makes at most $\ell$ mistakes total on any
realizable sequence, each of cost at most $\max(\gs,\gc)$. For the last claim,
the sound learner makes at most $\ell$ completeness mistakes and no soundness
mistakes, so its cost is at most $\gc\ell$. Combined with the lower bound
$\min(\gs,\gc)\ell=\gc\ell$ (valid since $\gs\ge\gc$), this gives equality.
\end{proof}

\paragraph{Singleton rejectors.}
Recall $R_n=\{h_\sigma:\sigma\in D\}$ with $|D|=n\ge2$ and
$h_\sigma(z)=\no\iff z=\sigma$ (Section~\ref{sec:sc-singleton}).  We show that the sharp transition of
Proposition~\ref{prop:rn} (from $n-1$ mistakes down to $1$ for a single unit of soundness
budget) reappears here.

\begin{proposition}
\label{prop:wsc-singleton}
$\displaystyle \WSCLdim(R_n)=\min\big(\gs,\;(n-1)\gc\big).$
\end{proposition}

\begin{proof}
\emph{Upper bound.} The  learner that predicts $\yes$ on every unqueried
point makes no completeness mistakes and at most one
soundness mistake for total cost at most $\gs$. The learner that predicts $\no$ until $\sigma$ is found makes at most $n-1$ completeness mistakes and no soundness mistakes for
total cost of at most $(n-1)\gc$. Hence $\WSCLdim(R_n)\le\min(\gs,(n-1)\gc)$.

\emph{Lower bound.} When the
learner predicts $\yes$ on a point, the adversary can reveal $\no$, which forces a soundness mistake of cost $\gs$, after
which the class is a singleton and the interaction ends. When the learner
predicts $\no$, the adversary can reveal $\yes$ and force a
completeness mistake of cost $\gc$. This sequence of examples is realizable
by any $h_{\sigma'}$ with $\sigma'$ among the unqueried points. Every
deterministic learner is thus driven along a root-to-leaf path that either
ends after one soundness mistake (weight $\gs$) or accumulates completeness
mistakes until all $n-1$ points are queried (weight
$(n-1)\gc$). Thus, such a path has weight at least $\min(\gs,(n-1)\gc)$.
\end{proof}

\paragraph{Additivity for layered products.}
We now show that the weighted dimension is exactly additive across the layers
of a layered product (Definition~\ref{def:layered-product}), a cleaner statement than
the superadditivity of Proposition~\ref{prop:additive}(ii) for the budgeted
frontier. A layer-$i$ query restricts only the $i$-th factor of the
(product) version space, so the layers never interact. With a shared integer
budget the adversary can spend it in whichever layer is most damaging, which
is why only superadditivity holds for $\SCLdim$. With independent per-mistake
costs each layer is charged separately and the totals simply add.

\begin{restatable}{proposition}{wscadditive}
\label{prop:wsc-additive}
Let $H=H_1\times\cdots\times H_L$ be a layered product on disjoint layers
$X\times\Sigma^L=Z_1\sqcup\cdots\sqcup Z_L$. Then
\[
\WSCLdim(H)=\sum_{i=1}^{L}\WSCLdim(H_i),
\]
where $\WSCLdim(H_i)$ is computed over trees whose nodes lie in layer $Z_i$.
\end{restatable}

\paragraph{The two-layer class.}
Recall $H_\beta=R_{n,b}\times R_{n'}$ with $n\ge2b+1$ and $n'\ge n-2b+2$
(Section~\ref{sec:sc-layered}), the class used for the behavior-cloning
separation of Section~\ref{sec:bc-lower-bound}. Recall also that
$R_{n,b}=\{h_S:S\subseteq D,\;|S|=b\}$ with $h_S(z)=\no\iff z\in S$, where
$|D|=n$. For this class we show the following.

\begin{restatable}{proposition}{wscfixedcard}
\label{prop:wsc-fixed-card}
For $1\le b\le n-1$,
$\displaystyle \WSCLdim(R_{n,b})=\min\big(b\gs,\;(n-b)\gc\big).$
\end{restatable}

At a high level, accepting everything costs one soundness mistake per element of the hidden set $S$, hence at most $b\gs$. Rejecting everything
costs one completeness mistake per correct point, hence at most $(n-b)\gc$. The key structural point is that no other strategy performs better. Note that for $\gs=\gc=1$ this is $\min(b,n-b)=b$ so under the assumption
$n\ge2b+1$ this matched $\Ldim(R_{n,b})=b$ from Proposition~\ref{prop:two-scale-frontier}. We present the full proof in Appendix~\ref{sec:wsc-examples}.

For $H_\beta$, we then have that the two layers contribute additively to the WSC-Ldim. 

\begin{restatable}{proposition}{wschbeta}
\label{prop:wsc-hbeta}
Writing $\lambda=\gs/\gc$, for all $\gs,\gc>0$,
\[
\WSCLdim(H_\beta)=\min\big(b\gs,(n-b)\gc\big)+\min\big(\gs,(n'-1)\gc\big),
\]
which, under the standing assumptions $n\ge2b+1$ and $n'\ge n-2b+2$, has
exactly three regimes:
\[
\WSCLdim(H_\beta)=
\begin{cases}
(b+1)\,\gs, & \lambda\le \dfrac{n-b}{b},\\[2mm]
(n-b)\,\gc+\gs, & \dfrac{n-b}{b}\le\lambda\le n'-1,\\[2mm]
\big((n-b)+(n'-1)\big)\gc, & \lambda\ge n'-1.
\end{cases}
\]
\end{restatable}

Intuitively, we show that the optimal learner switches its strategy on each layer (from accepting unresolved points to rejecting them) at a different cost-ratio threshold $\lambda=\frac{\gamma_s}{\gamma_c}$. On 
$R_{n,b}$ this switch occurs at $\lambda=(n-b)/b$ and on $R_{n'}$ at
$\lambda=n'-1$. The intermediate regime is the interesting one, the learner
handles $R_{n,b}$ soundly (paying $(n-b)\gc$ in completeness mistakes) yet
still gambles on the single hidden point of the huge $\Sigma'$-layer (paying
one soundness mistake $\gs$ rather than up to $(n'-1)\gc$). This reproduces
the three-regime behavior of Proposition~\ref{prop:two-scale-frontier}, but now the
regimes are selected by the continuous cost ratio rather than by the integer
soundness budget $k\in\{0,1,\dots,b{+}1\}$.

\paragraph{Subspace verifiers.}
Recall the subspace verifier class (Section~\ref{sec:sc-subspace}), where for
$d\ge2$, $d'\ge d-1$,
\[
H_{\mathrm{sub}}=\{h_v:v\in\mathbb R^d\setminus\{0\}\},\qquad
h_v(x,\tau_{1:\ell})=\yes \text{ iff }\tau_\ell\in\mathrm{span}(x,v).
\]
In Theorem~\ref{thm:subspace}, we showed that
$\SCLdim(H_{\mathrm{sub}},k)=\infty$ for $0\le k\le d-2$ and
$=d-1$ for $k\ge d-1$. We now show that the WSC-Ldim, by contrast, is finite for all strictly positive costs.

\begin{restatable}{theorem}{wscsubspace}
\label{thm:wsc-subspace}
For $d\ge2$ and $d'\ge d-1$,
\[
\WSCLdim(H_{\mathrm{sub}})=
\begin{cases}
0, & \min(\gs,\gc)=0,\\[1mm]
(d-1)\max(\gs,\gc), & \gs,\gc>0.
\end{cases}
\]
In particular, at $\gs=\gc$ this equals $d-1=\Ldim(H_{\mathrm{sub}})$.
\end{restatable}

At a high level, suppose first that $\gs\ge\gc$. The adversary asks whether $v$ lies outside a hyperplane. If the learner answers $\yes$, the adversary can force a soundness mistake of cost $gs$, after which $v$ is restricted to that hyperplane. If the learner repeatedly answers $\no$, the adversary can instead force enough completeness mistakes to incur the same total cost. Thus, reducing the dimension by one costs at least $\gs$. When $\gc\ge\gs$, a symmetric construction makes each dimension reduction cost at least $\gc$. Since $d-1$ reductions are necessary, this gives the lower bound
$(d-1)\max(\gs,\gc)$.

For the upper bound, every query has an answer that restricts $v$ to a subspace of smaller dimension. This answer costs at most $\max(\gs,\gc)$, and the dimension can decrease at most $d-1$ times. Therefore, when $\gs,\gc>0$, $\operatorname{WSC\text{-}Ldim}(\Hd)
=(d-1)\max(\gs,\gc)$.

\section{Improving Weak Generators with Online Verification}\label{sec:boost}

In this section, we show how we can exponentially amplify the performance of a set of weak generators by training a verifier online using the methods from the previous section.  In particular, we show that we can obtain strong guarantees for the correctness and reliability of the overall system even if the generators adapt their outputs to the history of interaction with the verifier. We then complement this result with a lower bound showing that the interaction used in this reduction is necessary in general: training only on transcripts generated under the perfect verifier can require significantly more samples.

\subsection{Interactive Learning with History-Dependent Generators} \label{sec:interactive-upper-bound}

In this section, we show that an online prefix verification algorithm can be used to exponentially amplify the performance of a set of ``weak'' generators.  

We begin by making explicit the interaction history observed by ${\cal G}$, the set of $N$ generators. Fix a problem instance $x\in X$. An \emph{interaction transcript} (or history) is a finite sequence \[ \theta = \bigl((\mathbf c_1,\mathbf v_1),\ldots,(\mathbf c_\ell, \mathbf v_\ell)\bigr) \in \left(\left(\Sigma\times\{\yes,\no\}\right)^N\right)^*, \]  where $(\mathbf c_j=(c_{1,j},\dots,c_{N,j}),\mathbf v_j=(v_{1,j},\dots,v_{N,j}))$ record the proposed candidate step $c_{i,\cdot}$ by every generator $i\in[N]$ and corresponding verifier verdict $v_{i,\cdot}$. Let $\mathfrak T$ denote the set of such transcripts. We write $\tau(\theta)\in\Sigma^{\leq L}$ for the \emph{accepted prefix} induced by $\theta$.  Whenever $v_{i,j}=\yes$ for some $c_{i,j}$ for $i\in[k]$, then the accepted $c_{i,j}$ with the smallest index $i$ is appended to the accepted prefix. The verdicts in $\theta$ are those of the verifier actually deployed in the interaction. The generator does not observe the oracle's label directly. Thus, verifier mistakes can change the generator's future output distribution. We call a transcript $\theta = \bigl((\mathbf c_1, \mathbf v_1),\ldots,(\mathbf c_\ell, \mathbf v_\ell)\bigr)$ {\it oracle-consistent} for instance $x$ w.r.t.\ oracle $\cal O$, if all recorded verdicts match the oracle, i.e., $v_{i,j}=\mathcal{O}(x, (\tau(\theta_{<j}), c_{i,j}))$ for every round $j\in[\ell]$ and every generator $i\in[N]$, where $\theta_{<j}$ is the accepted prefix just before round $j$.

We assume that transcripts are reset at the beginning of each problem instance. In particular, the interaction on an instance $x$ starts from $\theta=\emptyset$, and the generators carry no state between different problem instances. 
A history-dependent generator is a map $G: X\times\mathfrak T\rightarrow\Delta(\Sigma)$. Thus, $G(x,\theta)$ denotes the distribution of its next candidate step on instance $x$ after transcript $\theta$. A memoryless generator of the form considered in Section~2.3 is the special case in which $G(x,\theta)$ depends on $\theta$ only through the accepted prefix $\tau(\theta)$.

\begin{definition}[$(\alpha,\gamma)$-good history-dependent generators] \label{def:good-history-dependent-generators} A set of history-dependent generators $\mathcal G=\{G_1,\ldots,G_N\}$ is \emph{$\alpha$-good} for an instance $x\in X$ with respect to correctness oracle $\mathcal O$ if, for every \emph{oracle-consistent} transcript $\theta$ whose accepted prefix $\tau(\theta)=\tau_{1:\ell}$ is correct and has length $\ell<L$, there exists an $i\in[N]$ such that \[ \Pr_{c\sim G_i(x,\theta)} \bigl[ \mathcal O(x,\tau(\theta),c)=\yes \bigr] \geq \alpha. \] We say that $\mathcal G$ is $(\alpha,\gamma)$-good for a distribution $D$ if $\Pr_{x\sim D}[\mathcal G\text{ is $\alpha$-good for }x]\geq\gamma$. \end{definition}

Definition~\ref{def:good-history-dependent-generators} quantifies only over oracle-consistent transcripts, i.e., transcripts where every verifier verdict so far has agreed with the oracle. In particular, a generator is free to lose its $\alpha$-goodness once a verifier mistake has driven the interaction off the oracle-consistent path, without affecting any of our guarantees. Our main result gives a bound on the sample complexity of instances and oracle calls needed to learn a verifier that can then be used to amplify the correctness of generators in $\mathcal G$.

\begin{algorithm}[t]
\caption{Constructing $\vhp$}
\label{alg:vhp}

\DontPrintSemicolon
\SetAlgoLined
\LinesNumbered
\SetNoFillComment
\SetKwComment{tcp}{\(\triangleright\) }{}
\SetKwInOut{KwIn}{Input}

\KwIn{ set of $N$ generators ${\cal G}$; values $\alpha,\epsilon,\delta$; step-by-step verification algorithm $V_H$ with soundness mistake-bound $M_s$ and completeness mistake-bound $M_c$.}

\BlankLine

Draw set $S_1$ of $8\!\left(\frac{M_c+M_s}{\epsilon} + \log \frac{2}{\delta}\right)$ random examples from $D$\;

\For{\emph{each} $x \in S_1$}{
  Run \textsc{Process-Example}$(x,{\cal G},\alpha,V_H)$\tcp*[r]{Train $V_H$ on $x$ and ${\cal G}$}\label{step:train}
}

Let $h_1,\ldots,h_M$ be the candidate verifiers produced in Step \ref{step:train}\;

\BlankLine

Draw set $S_2$ of $
O\!\left(
\frac{1}{\epsilon}\cdot
\frac{M_s+M_c}{\min(M_s,M_c)+1}\cdot
\log\!\left(\frac{M_s+M_c}{\delta}\right)
\right)$
examples from $D$\;

\For{\emph{each} $x \in S_2$}{
  \For{\emph{each} $h_i \in \{h_1,\ldots,h_M\}$}{
    Run \textsc{Test-Verifier}$(x, {\cal G} , \alpha, h_i)$}
  }

Choose any $h_i$ with
\Indp
empirical soundness error $\le \frac{3\epsilon}{4}\cdot\frac{M_s}{M_s+M_c}$ and
empirical completeness error $\le \frac{3\epsilon}{4}\cdot\frac{M_c}{M_s+M_c}$. If no such $h_i$ exists, halt with failure.\;
\Indm

Use $h = h_i$ in Algorithm~\ref{alg:weak-to-strong}, creating $\vhp$\;

\end{algorithm}

\begin{theorem}\label{thm:wrapper}
Given an $(\alpha,\gamma)$-good set of $N$ history-dependent generators ${\cal G}$ satisfying Definition~\ref{def:good-history-dependent-generators}, a verification algorithm $V_H$ with stepwise completeness and soundness mistake-bounds $M_c$ and $M_s$ respectively, and given the value $\alpha$, we can learn a generator $\vhp$ using a training sample of $O\left(\frac{M_s+M_c}{\epsilon} + \frac{1}{\epsilon}\left(\frac{M_s+M_c}{\min(M_s,M_c)+1}\right)\log\left(\frac{M_s+M_c}{\delta}\right)\right)$ problem instances $x\sim D$ and making $O\left(\left(\frac{M_s+M_c}{\epsilon}\right)\left(\frac{M_s+M_c}{\min(M_s,M_c)+1}\right)\left(\log\frac{M_s+M_c}{\delta}\right)\left(\frac{LN}{\alpha}\log \frac{LN}{\delta}\right)\right)$ calls to a labeling oracle $\cO=h^*$ such that with probability at least $1-\delta$:
\begin{enumerate}[topsep=2pt,partopsep=1ex,parsep=2ex]\itemsep=-4pt
    \item Given a problem instance $x$, $\vhp$ outputs either ``I don't know'' or a reasoning trace $\vtau$.
    \item The probability over $x\sim D$ that $\vhp$ outputs  ``I don't know'' is at most $(1-\gamma)+\epsilon_c + \epsilon_s + \delta$.
    \item The probability over $x\sim D$ that $\vhp$ outputs an incorrect reasoning trace is at most $\epsilon_s$,
\end{enumerate}
where $\epsilon_c = \frac{\epsilon M_c}{M_c+M_s}$ and $\epsilon_s = \frac{\epsilon M_s}{M_c+M_s}$.  
\end{theorem}

\begin{proof}
We begin by drawing a set $S_1$ of $8\left(\frac{M_c+M_s}{\epsilon} + \log \frac{2}{\delta}\right)$ random examples from $D$.  We then run Process-Example (Algorithm \ref{alg:process example}) on each example $x\in S_1$.  Each run of Process-Example makes $O(\frac{LN}{\alpha}\log \frac{LN}{\delta})$ calls to the labeling oracle $\cO$ and results in either $V_H$ making a mistake and updating its current verifier $h$ or not.  Let $err^s(h)$ denote the probability that $h$ makes a soundness mistake when running Process-Example$(x)$ on a random $x\sim D$ (probability taken over the draw from $D$ and the randomness in the generators in ${\cal G}$).  Similarly, let $err^c(h)$ denote the probability that $h$ makes a completeness mistake when running Process-Example$(x)$ on a random $x \sim D$.  Notice that if each $h$ produced had $err^s(h)\geq \frac{\epsilon}{2}\frac{M_s}{M_s+M_c}$ or $err^c(h) \geq \frac{\epsilon}{2}\frac{M_c}{M_s+M_c}$, then the expected number of soundness mistakes would be greater than $2M_s$ or the expected number of completeness mistakes would be greater than $2M_c$. By the Martingale Chernoff argument of \citet{littlestone:1989}, we can therefore conclude that with probability $\geq 1-\delta/2$, at least one of the $h$ produced has $err^s(h)\leq \frac{\epsilon}{2}\frac{M_s}{M_s+M_c}$ and $err^c(h) \leq \frac{\epsilon}{2}\frac{M_c}{M_s+M_c}$.

We now draw a fresh set $S_2$ of $O\left(\frac{1}{\epsilon}\left(\frac{M_s+M_c}{\min(M_s,M_c)+1}\right)\log\left(\frac{M_s+M_c}{\delta}\right)\right)$ examples to test the at most $M_s+M_c$ verifiers $h_1,...,h_M$ produced above, by running each of them through Test-Verifier (Algorithm \ref{alg:test hypothesis}), which simulates the use of the verifier at test-time (Algorithm \ref{alg:weak-to-strong}), checking to see if any soundness or completeness mistake is made. Note that if $M_s=0$ then {\em all} $h_i$ will have $err^s(h_i)=0$ and similarly if $M_c=0$ then {\em all} $h_i$ will have $err^c(h_i)=0$, so we may wlog focus on the case that $M_s\geq 1$ and $M_c \geq 1$. Therefore, our sample $S_2$ is large enough that by Chernoff bounds, with probability $\geq 1-\delta/2$, (a) at least one of the verifiers $h_i$ tested will have empirical soundness error at most $\frac{3\epsilon}{4}\frac{M_s}{M_s+M_c}$ and empirical completeness error at most $\frac{3\epsilon}{4}\frac{M_c}{M_s+M_c}$, and (b) all such verifiers $h_i$ will have $err^s(h_i)\leq \epsilon\frac{M_s}{M_s+M_c}$ and $err^c(h_i) \leq \epsilon\frac{M_c}{M_s+M_c}$.

Finally, we choose one such $h_i$ and use it in Algorithm \ref{alg:weak-to-strong}, creating the generator $\vhp$.  The probability that the generator returns ``I don't know'' in Step 12 is at most the probability $1-\gamma$ that ${\cal G}$ is not $\alpha$-good for $x$, plus the probability $\delta$ that the algorithm reaches its timeout even if no mistake is made (\emph{i.e.}, the generator fails to generate a reasoning trace where all $L$ steps are correct), plus the probability $\epsilon_c$ of a completeness mistake plus the probability $\epsilon_s$ of a soundness mistake (that could lead to a state where there is no correct continuation of the trace).  The probability the generator returns an incorrect reasoning trace in Step 13 is the probability $\epsilon_s$ of a soundness mistake, as desired.  The entire process is described formally in Algorithm \ref{alg:vhp}.
\end{proof}

\begin{algorithm}[t]
\caption{\textsc{Process-Example}}
\label{alg:process example}

\DontPrintSemicolon
\SetAlgoLined
\SetCommentSty{smallcommentstyle}
\LinesNumbered
\SetNoFillComment
\SetKwComment{tcp}{\(\triangleright\) }{}
\SetKwInOut{KwIn}{Input}

\KwIn{ example $x$; set of $N$ generators ${\cal G}$; value $\alpha$; step-by-step verification algorithm $V_H$.}

\BlankLine

$\theta\gets \emptyset$ \tcp*[r]{initialize the per-instance transcript}

\For{$\ell = 1$ \KwTo $L$}{
    Set $i \gets 1$
    
  \While{$i \le \left\lceil \frac{1}{\alpha}\ln\frac{NL}{\delta} \right\rceil$\label{step:while}}
  {
    Sample $\tau^{(i,j)}_{\ell} \sim G_j(x,\theta)$ for all $G_j \in {\cal G}$\tcp*[r]{ask each generator for next step}

    Set $v^{(i,j)}_{\ell} = V_H(x,\tau_1,\ldots,\tau_{\ell-1},\tau^{(i,j)}_{\ell}) \in \{\yes,\no\}$\tcp*[r]{run $V_H$}

    $\theta.\text{append}((\tau_\ell^{(i,1)}, \dots, \tau_\ell^{(i,N)}), (v_\ell^{(i,1)}, \dots, v_\ell^{(i,N)}))$\tcp*[r]{append all verdicts to transcript history}}

    Set $y^{(i,j)}_{\ell} = \mathcal{O}(x,\tau_1,\ldots,\tau_{\ell-1},\tau^{(i,j)}_{\ell}) \in \{\yes,\no\}$\tcp*[r]{query labeling oracle}

    \eIf(\tcp*[f]{verifier believes all $\tau^{(i,j)}_{\ell}$ are incorrect}){$v^{(i,j)}_{\ell}=\no$ \emph{for all} $j\in[N]$}{
      Set $i = i+1$\tcp*[r]{discard and try again} 
    }(\tcp*[f]{verifier believes it found a correct next step}){
      Set $\tau_\ell \gets \tau^{(i,j)}_{\ell}$ for the smallest $j\in[N]$ with $v^{(i,j)}_{\ell}=\yes$
      \tcp*[r]{append step}
      Break and continue the for-loop\tcp*[r]{exit while; continue with for-loop}
    }
  }
  
  \uIf(\tcp*[f]{no correct next step found; check all predictions}){$v_{\ell'}^{(i,j)} \neq y_{\ell'}^{(i,j)}$ \emph{for any} $\ell' \le \ell$, \emph{any} $i$, \emph{any} $j \in [N]$}{
    $V_H.\text{{\sc Update}}((x, \tau_{1:\ell'-1}, \tau_{\ell'}^{(i,j)}),\, y_{\ell'}^{(i,j)})$ \tcp*[r]{make progress on learning verifier}

  \Return\tcp*[r]{timeout (no trace found)} 
}

\uIf(\tcp*[f]{trace found; check for any errors}){$\exists\,\ell \le L \emph{ such that } \mathcal{O}(x,\tau_{1:\ell})=\no$ \label{line:all-prefixes-rejected2}}{
  $V_H.\text{{\sc Update}}((x, \tau_{1:\ell}),\, \no)$ \tcp*[r]{learning progress}
}

\end{algorithm}

\begin{remark}
    A standard VC lower-bound construction can be specialized to our reasoning setting. Suppose there are $n$ ``hard'' problems $x_1, \dots, x_n$ and one ``easy'' problem $x_0$. We can choose a uniformly random set $S^\star$ of size $\beta n$. Each $x_i$ has an ``ambiguous'' step $a_i$, which is incorrect exactly when $i\in S^\star$, a ``good'' step $g_i$, which is always correct, and a ``bad'' step $z_i$, which is always incorrect. We assign probability $1/(8m)$ on each $x_i$ and the remaining probability on $x_0$. On $x_i$ the generator first proposes $a_i$ or $g_i$, each with probability $1/2$. If the verifier correctly rejects an incorrect step $a_i$, the generator outputs $g_i$. If the verifier rejects a correct step, the generator enters into ``failure mode'' and outputs $z_i$ forever, so the wrapper must either accept an incorrect step or eventually output $\texttt{"I don't know"}$. Notice that the generator is $(1/2,1)$-good. Then over $m$ training samples, only $n/8$ ``hard'' problems are seen in expectation and with constant probability, at least half of both the correct and incorrect hard problems remain unseen. If $q$ is the fraction of unseen $a_i$'s accepted by the learned verifier, then, conditional on the training sample, the expected numbers of accepted incorrect and rejected correct unseen steps are at least $\beta nq/2$ and $(1-\beta)n(1-q)/2$, respectively. Each corresponding test problem has probability $1/(8m)$ and the generator outputs $a_i$ with probability $1/2$. Hence the expected probability of an incorrect answer or $\texttt{"I don't know"}$ is at least $\Omega(n/m)$.
\end{remark}

Our generator-verifier formulation can also be viewed through the lens of imitation learning. In standard imitation learning, a central limitation of learning strictly from offline data (behavior cloning) is that the learner is trained only on states induced by the expert policy; at test time, if the learner makes an early mistake, this may steer the trajectory away from the expert’s path and on unexplored states, leading to compounding errors \cite{ross2010efficient,ross2011reduction}. 
Similarly here, this issue arises when the verifier makes a completeness mistake, that is, it rejects a correct prefix. Specifically, after a completeness mistake,  we cannot simply stop training on the reasoning trace if we want to understand the downstream effect of this error. Instead, we must continue verifying along the trace until either the verifier makes a soundness mistake (at which point we should {\sc Update} the verifier on that soundness mistake), or the true first faulty step is reached (Lines 17-18 in \Cref{alg:process example}). This accounting is crucial for separating the two ways verifier errors affect the performance of the generator(s)-verifier system: completeness errors may cause unnecessary ``I don't know" outputs, whereas soundness errors may lead to incorrect outputs. 

We will next show that for behavior cloning, where training is performed only on interactions between the generator and the oracle $\cal O$, the sample complexity is necessarily larger.

\begin{algorithm}[t]
\caption{\textsc{Test-Verifier}}
\label{alg:test hypothesis}

\DontPrintSemicolon
\SetAlgoLined
\LinesNumbered
\SetNoFillComment
\SetKwComment{tcp}{\(\triangleright\) }{}
\SetKwInOut{KwIn}{Input}

\KwIn{ example $x$; set of $N$ generators ${\cal G}$; value $\alpha$; verifier $h: X \times \Sigma^* \mapsto \{\yes,\no\}$.}

Run Algorithm~\ref{alg:weak-to-strong}$(x,{\cal G},\alpha,h)$, recording all predictions made by $h$\;

\If{\emph{Algorithm~\ref{alg:weak-to-strong} returned a reasoning trace} $\vtau$}{
  \If(\tcp*[f]{$L$ oracle calls}){$\cO(x,\tau_{1:\ell})=\no$ \emph{for any} $\ell \le L$ \label{line:all-prefixes-rejected}}{
    \Return \texttt{"Soundness Mistake"}\;
  }
}
\Else(\tcp*[f]{returned \texttt{"I don't know"}}){
  \If(\tcp*[f]{$O(\frac{NL}{\alpha}\ln \frac{NL}{\delta})$ oracle calls}){$h$ \emph{made any mistakes}}{
    \Return the type of the first mistake (\texttt{"Completeness Mistake"} or \texttt{"Soundness Mistake"})
  }
}

\Return \texttt{"Correct"}\;
\end{algorithm}

\begin{algorithm}[t]
\caption{Combining set of generators ${\cal G}$ and learned verifier $h$ to create $\vhp$}
\label{alg:weak-to-strong}

\DontPrintSemicolon
\SetAlgoLined
\LinesNumbered
\SetNoFillComment
\SetKwComment{tcp}{\(\triangleright\) }{}
\SetKwInOut{KwIn}{Input}

\KwIn{ prompt $x \in X$; set of $N$ generators ${\cal G}$; value $\alpha$; verifier $h: X \times \Sigma^* \mapsto \{\yes,\no\}$.}

\BlankLine
$\theta\gets \emptyset$ \tcp*[r]{initialize the per-instance transcript}

\For{$\ell = 1$ \KwTo $L$}{ 
    Set $i \gets 1$
    
  \While{$i \le \left\lceil \frac{1}{\alpha}\ln\frac{NL}{\delta} \right\rceil$}{
    Sample $\tau^{(i,j)}_{\ell} \sim G_j(x,\theta)$ for all $G_j \in {\cal G}$\tcp*[r]{ask each generator for next step}

    Set $v^{(i,j)}_{\ell} = h(x,\tau_1,\ldots,\tau_{\ell-1},\tau^{(i,j)}_{\ell}) \in \{\yes,\no\}$\tcp*[r]{run the verifier}
    
    $\theta.\text{append}((\tau_\ell^{(i,1)}, \dots, \tau_\ell^{(i,N)}), (v_\ell^{(i,1)}, \dots, v_\ell^{(i,N)}))$\tcp*[r]{update transcript history}

    \eIf(\tcp*[f]{all sampled steps were rejected}){$v^{(i,j)}_{\ell}=\no$ \emph{for all} $j\in[N]$}{
      $i = i+1$\;
      
      Continue\tcp*[r]{discard and try again}
    }(\tcp*[f]{found an accepted next step}){
      
      Set $\tau_\ell \gets \tau^{(i,j)}_{\ell}$ for the smallest $j\in[N]$ with $v^{(i,j)}_{\ell}=\yes$
      \tcp*[r]{append step}
      
      Break\tcp*[r]{exit while; continue for-loop}
    }
  }

  \Return \texttt{"I don't know"}\tcp*[r]{no correct next step found; timeout}
}

\Return $\tau_1,\ldots,\tau_L$\tcp*[r]{return trace found}

\end{algorithm}

\subsection{Why Interaction Helps: A Lower Bound for Behavior-Cloning}
\label{sec:bc-lower-bound}

Theorem~\ref{thm:wrapper} trains the verifier on transcripts
induced by its own predictions. We now show that this interaction cannot in
general be replaced by behavior cloning from transcripts generated under
the perfect verifier while maintaining a good soundness-completeness tradeoff. 
The high-level reason is that a completeness mistake may change the behavior of a
history-dependent generator and expose the learned verifier to candidate steps
that never appear in the perfect-verifier training distribution, which in
turn could lead to a downstream soundness mistake.

We first introduce the definition of behavior-cloning \citep{pomerleau1991efficient,bain1995framework,10.1145/3054912} for our setting. An
$m$-sample \emph{behavior-cloning learner} observes $m$ interactions
generated by deploying the perfect verifier $\cal O$ alongside $G$. Thus, during
training, whenever $G$ proposes a candidate $c$ after a transcript $\theta$,
the recorded verdict is the oracle label ${\cal O}(c)$ (in our construction the oracle labels each candidate independently of the
accepted prefix, so we write ${\cal O}(c)$ for its verdict on $c$), and consequently
every training transcript is oracle-consistent. The learner outputs a fixed verifier
$\hat h : \Sigma^* \to \{\mathrm{yes},\mathrm{no}\}.$
The learner may be \emph{randomized} and \emph{improper}, meaning that
$\hat h$ need not belong to the target verifier class.
 
At deployment, the generator is instead paired with $\hat h$. Starting from
the empty transcript, each candidate produced by $G$ is labeled by
$\hat h$. The learner
receives no oracle feedback during deployment and $\hat h$ is not updated.
For a fixed verifier $\hat h$, we define its \emph{on-policy} soundness and
completeness errors by
\begin{align*}
\epsilon_s(\hat h) &:=
\Pr\!\left[\,\hat h \text{ accepts an incorrect candidate during deployment}\,\right],
\\
\epsilon_c(\hat h) &:=
\Pr\!\left[\,\hat h \text{ rejects a correct candidate during deployment}\,\right],
\end{align*}
where the probability is over the deployed interaction between $G$ and
$\hat h$. When we reason about the learner's output
(Theorem~\ref{thm:bc-lower-bound}), $\hat h = \hat h(S)$ is itself
random through the training sample $S$ and the learner's internal randomness,
over which the stated guarantees are taken. We now construct the verifier class and generator used in our lower bound.

\paragraph{The verifier class.}
We use the layered-product class
$H_\beta$ from Proposition~\ref{prop:two-scale-frontier}. Recall that
$\Sigma$ and $\Sigma'$ are disjoint, $|\Sigma|=n$, $|\Sigma'|=n'$, and
$b=\beta n$ for a constant $\beta\in(0,1/2)$. Each target
$h_{S^\star,\sigma^\star}$ rejects exactly the set
$S^\star\subseteq\Sigma$, where $|S^\star|=b$, and one point
$\sigma^\star\in\Sigma'$. Note that $\text{VCdim}(H_\beta) = b+1$. Proposition~\ref{prop:two-scale-frontier} gives
$\Ldim(H_\beta)=b+1$ and
$\SCLdim(H_\beta,1)=n-b+1$. In particular, there is an
online verifier learner making at most one soundness mistake and at most $n-b$
completeness mistakes. Applying
Theorem~\ref{thm:wrapper} with $M_s=1$ and $M_c=n-b$ gives,
up to logarithmic and confidence factors, $\epsilon_s=O(1/m)$ and $\epsilon_c=O(n/m)$
after $m$ interactive training instances.

Draw the target according to
\[
    S^\star
    \sim
    \operatorname{Unif}
    \{S\subseteq\Sigma:|S|=b\},
    \qquad
    \sigma^\star\sim\operatorname{Unif}(\Sigma'),
\]
independently, and let $\mathcal O=h_{S^\star,\sigma^\star}$. 

\paragraph{A feedback-dependent generator.} We construct a single history-dependent generator $G_{\mathrm{fb}}$. Fix a training
sample size $m\geq n$ and set $\rho:=n/(8m)$. Define the baseline
distribution
\[
    Q
    :=
    \rho\,\operatorname{Unif}(\Sigma)
    +(1-\rho)\,\operatorname{Unif}(\Sigma').
\] 
\noindent Given a
transcript $\theta$, if $\theta$ contains no completeness mistake, then
$G_{\mathrm{fb}}$ samples its next candidate from $Q$. If $\theta$
contains a completeness mistake---that is, the deployed verifier has rejected
some candidate outside $S^\star\cup\{\sigma^\star\}$---then
$G_{\mathrm{fb}}$ enters a ``failure mode’’ (the generator does not actually receive the oracle labels, only its distribution depends on the transcript and the target). With probability $1/2$ it outputs
$\sigma^\star$ (the unique wrong trace in $\Sigma'$), and with probability $1/2$ it draws a fresh candidate from
$Q$.

\begin{lemma}
\label{lem:feedback-generator-good}
For all $n'\ge 4$, the generator $G_{\mathrm{fb}}$ is
$1/4$-good for any instance $x\in X$.
\end{lemma}

\begin{proof}
For every transcript, including one in failure mode,
$G_{\mathrm{fb}}$ draws from $Q$ with probability at least $1/2$. A draw
from $Q$ lies in $\Sigma'$ with probability $1-\rho\geq7/8$, and a uniform
draw from $\Sigma'$ is correct with probability $1-1/n'$. Hence, for $n'\ge 4$,
\[
    \Pr_{c\sim G_{\mathrm{fb}}(\theta)}
    [\mathcal O(c)=\yes]
    \geq
    \frac12\left(1-\rho\right)
    \left(1-\frac1{n'}\right)
    \geq  \frac12 \cdot \frac78 \cdot \frac34 >\frac14.
\]
Thus, one may take $\alpha=1/4$.
\end{proof}

Note that we actually establish $\alpha$-goodness for every transcript (not just oracle-consistent ones), which is stricter than what is needed by Definition \ref{def:good-history-dependent-generators}. This makes our lower bound more robust. Under the perfect verifier, a completeness mistake never occurs, so the
failure mode is never entered. Consequently, behavior-cloning training
reveals only independent samples from the baseline process $Q$ together with
their oracle labels.

\paragraph{Behavior cloning.}
A behavior-cloning learner receives $m$ oracle-generated interactions
between $G_{\mathrm{fb}}$ and the perfect verifier $\mathcal O$, and outputs
a fixed verifier
$\hat h:\Sigma\cup\Sigma'\rightarrow\{\yes,\no\}$. The learner may be
improper and randomized. At deployment, $\hat h$ interacts with
$G_{\mathrm{fb}}$ but receives no further oracle feedback and is not
updated. Let $\epsilon_s(\hat h)$ and $\epsilon_c(\hat h)$ denote its
on-policy soundness and completeness errors in this deployed process.

\begin{theorem}[Behavior-cloning lower bound]
\label{thm:bc-lower-bound} Fix $\beta\in(0,1/2)$ and let $b=\beta n\in\mathbb N$. Suppose
$m\geq n$, $n'\geq 8m$, and the behavior-cloning learner observes
$m$ oracle-generated interactions. Then, for every possibly randomized
behavior-cloning learner producing a fixed verifier
$\hat h$, either
\[
    \epsilon_c(\hat h)\geq \frac{3}{64}
\qquad\text{ or }\qquad
    \epsilon_s(\hat h)
    \geq
    c_\beta\frac{n}{m}
\]
holds, where
$c_\beta
    :=
    \min\left\{
        \frac{3\beta}{512},
        \frac{3(1-\beta)}{2048}
    \right\}$ is a constant for fixed $\beta$.
\end{theorem}

\begin{proof}
By Yao's minimax principle, it suffices to consider a deterministic learner
under a random target. Draw $S^\star$ uniformly from the $b$-subsets of
$\Sigma$, draw $\sigma^\star$ uniformly from $\Sigma'$, independently, and
set $\mathcal O=h_{S^\star,\sigma^\star}$.

Recall that under the perfect verifier the generator never enters its failure
mode. Thus, the behavior-cloning sample consists of $m$ independent draws
from
\[
    Q
    =
    \rho\,\text{Unif}(\Sigma)
    +(1-\rho)\,\text{Unif}(\Sigma'),
    \qquad
    \rho:=\frac{n}{8m},
\]
together with their oracle labels. Let $\mathsf S$ denote the resulting labeled
sample.

Let $R$ and $C$ be the numbers of distinct observed incorrect and correct
points of $\Sigma$, respectively. Since a training draw is an incorrect
$\Sigma$-point with probability $\rho\beta$ and a correct $\Sigma$-point
with probability $\rho(1-\beta)$,
\[
    \mathbb E[R]\leq \frac{b}{8},
    \qquad
    \mathbb E[C]\leq \frac{n-b}{8}.
\]
Therefore, by Markov's inequality,
\[
    \Pr\left[R>\frac b2\right]\leq\frac14,
    \qquad
    \Pr\left[C>\frac{n-b}{2}\right]\leq\frac14.
\]
Moreover, $\Pr[\sigma^\star\in \mathsf S]\leq m/n'\leq1/8$. Hence, by union bound, the ``bad'' event
\[
    \mathcal E
    :=
    \left\{
        R\leq\frac b2,\quad
        C\leq\frac{n-b}{2},\quad
        \sigma^\star\notin \mathsf S
    \right\}
\]
has probability at least $3/8$.

Condition on $\mathsf S$ and $\mathcal E$. Let $U:=\Sigma\setminus \mathsf S$. Since
$R+C\leq n/2$, a fresh uniform point of $\Sigma$ lies in $U$ with
probability at least $1/2$. Furthermore, by exchangeability of the unseen
labels, a uniformly random point of $U$ is incorrect with probability at
least $\beta/2$ and correct with probability at least $(1-\beta)/2$.

Define $U':=\Sigma'\setminus \mathsf S$,
\[
    q
    :=
    \Pr_{\tau\sim\text{Unif}(U)}
        [\hat h(\tau)=\yes].
\qquad
\text{ and, }
\qquad
    a
    :=
    \Pr_{\tau\sim\text{Unif}(U')}
        [\hat h(\tau)=\yes].
\]
Since $\sigma^\star$ is uniform among the unseen points of $\Sigma'$,
$a$ is also the conditional probability that $\hat h$ accepts
$\sigma^\star$. We consider three cases.

\paragraph{Case 1: $a<1/2$.}
The verifier rejects more than half of the unseen points of $\Sigma'$. All
but the single point $\sigma^\star$ are correct. Since $n'\geq8m$, at least
a $7/8$ fraction of $\Sigma'$ is unseen, and therefore the completeness
error ($\hat h$ rejects a correct $\Sigma'$ step) conditioned on $\mathsf S$ and $\mathcal E$ is at least
\[
    (1-\rho)\cdot\frac{(1-a)|U'|-1}{n'} \ \ge\ (1-\rho)\cdot\frac78\Big(\frac12-\frac{1}{|U'|}\Big)\ \ge\ \frac14 .
\]

\paragraph{Case 2: $a\geq1/2$ and $q\geq1/2$.} Let $A:=\{\tau\in U:\hat h(\tau)=\yes\}$ denote the set of unseen (in $\mathsf S$) steps from $\Sigma$ that are accepted by $\hat h$. Since $S^\star\cap U$ is a uniformly random $(b-R)$-subset of $U$, $\mathbb E[|A\cap S^\star|]=q(b-R)$.
A deployment draw lies in $\Sigma$ with probability $\rho$ and then is uniform on $\Sigma$, so it is an accepted incorrect point (soundness-error) with probability at least
\[
    \rho\cdot\frac{\mathbb E[|A\cap S^\star|]}{n}\ge\frac{\rho q(b-R)}{n}\ge \frac{q\beta\rho}{2}
    \ge
    \frac{\beta\rho}{4}.
\]

\paragraph{Case 3: $a\geq1/2$ and $q<1/2$.} Let $B:=\{\tau\in U:\hat h(\tau)=\no\}$ denote the set of unseen (in $\mathsf S$) steps from $\Sigma$ that are rejected by $\hat h$. Similarly to Case 2 above, $\mathbb E[|B\setminus S^\star|]=(1-q)(n-b-C)\ge (n-b)/4$.
A deployment draw is a rejected correct $\Sigma$-point (a completeness mistake) with
probability at least
\[
    \rho\cdot\frac{E[|B\setminus S^\star|]}{n}\ge\rho\cdot\frac{(n-b)/4}{n}
    =
    \frac{(1-\beta)\rho}{4}.
\]
This completeness mistake sends the generator into failure mode. The generator
then outputs $\sigma^\star$ with probability $1/2$, and the verifier accepts
it with probability $a\geq1/2$. Hence the resulting soundness-error
probability is at least
\[
    \frac{(1-\beta)\rho}{4}\cdot\frac12\cdot\frac12
    =
    \frac{(1-\beta)\rho}{16}.
\]

Thus, on every realization in
$\mathcal E\cap\{a\geq1/2\}$, the conditional soundness error is at least
\[
    \rho
    \min\left\{
        \frac{\beta}{4},
        \frac{1-\beta}{16}
    \right\}.
\]

Let $p:=\Pr[\mathcal E\cap\{a<1/2\}]$. If $p\geq3/16$, then Case~1 gives
\[
    \epsilon_c(\hat h)
    \geq
    \frac14 p
    \geq
    \frac{3}{64}.
\]
Otherwise,
$\Pr[\mathcal E\cap\{a\geq1/2\}]\geq3/16$, and Cases~2--3 give
\[
    \epsilon_s(\hat h)
    \geq
    \frac{3}{16}\rho
    \min\left\{
        \frac{\beta}{4},
        \frac{1-\beta}{16}
    \right\}.
\]
Substituting $\rho=n/(8m)$ yields
\[
    \epsilon_s(\hat h)
    \geq
    \frac{n}{m}
    \min\left\{
        \frac{3\beta}{512},
        \frac{3(1-\beta)}{2048}
    \right\}.
\]

\end{proof}

To obtain soundness error at most $\epsilon_s$ without incurring constant
completeness error, Theorem~\ref{thm:bc-lower-bound} requires
$m=\Omega_\beta(n/\epsilon_s)$ behavior-cloning interactions. In contrast,
Theorem~\ref{thm:wrapper}, instantiated with the
soundness-budget-one learner for $H_\beta$, uses
$\widetilde O(1/\epsilon_s+n/\epsilon_c)$ interactive instances. In the
regime $\epsilon_c=\Theta(n\epsilon_s)$, this is an $\Omega(n)$ separation.

\begin{remark}
    Our upper bound in Theorem~\ref{thm:wrapper}, requires $\alpha$-goodness only on oracle-consistent transcripts, which is weaker than requiring it on all transcripts. But our generator in the lower bound construction is actually $\alpha$-good w.r.t.\ all transcripts, including those where a completeness mistake by the verifier sends the generator into a failure mode. This strengthens our separation, in the sense that behavior cloning has the extra sample requirement even for generators that are robust to bad transcripts.
\end{remark}

\section{A Sequential Harnessing Perspective}
\label{sec:harness}

The reductions in this paper are instances of the following more general framework. A complex sequential prediction (system-level) task can often be solved by repeatedly invoking a simpler (base-level) online learner, provided that every system-level mistake can be decoded into a base-level mistake. In Appendix~\ref{app:online-template}, we formalize this principle through a \emph{sequential harnessing} framework. A \emph{harness} specifies how to interact with a base predictor in order to produce a system-level output, while a \emph{decoder} maps a labeled system-level example back to a sequence of base-level examples. When these two objects satisfy a natural consistency condition, any mistake bound guarantee for the base online learning problem transfers directly to the induced system-level problem (Theorem \ref{thm:online-to-online-reduction}).\looseness-1

This gives a clean way to view the equivalence between prefix verification and full chain-of-thought verification. Running a prefix verifier (base-level) sequentially over a reasoning trace induces a full chain-of-thought verifier (system-level), while a full chain-of-thought verifier can be used to answer prefix queries by padding the prefix so that the full-trace label reveals the prefix label. These reductions recover the prefix and full-chain-of-thought verifier equivalence (\Cref{sec:reduction}), albeit not distinguishing between soundness and completeness mistakes. This framework can also recover a version of our improving-weak-generators result (\Cref{sec:boost}) that does not distinguish between soundness and completeness mistakes. The generator-verifier wrapper can be viewed as a system whose output is produced by a sequence of verifier queries. A mistake made by the wrapper corresponds to a mistake by the underlying verifier on one of the queried prefixes. This gives an alternative, imitation-learning-style interpretation of verifier-guided reasoning. 

The analysis recovers the standard imitation-learning intuition that one should train on states induced by the learner's own behavior (Section \ref{subsec:imitation-harnessing}). It can also be applied to online chain-of-thought generation to obtain mistake bounds independent of the chain length (Section \ref{subsec:online-generations}). The sequential harnessing framework is useful for highlighting the common structure behind these reductions, but it is coarser than the analysis presented in the main body of the paper since it treats all mistakes symmetrically. The soundness-completeness Littlestone dimension and the weighted soundness-completeness dimension are therefore still needed to obtain the fine-grained tradeoffs that are central to our results.\looseness-1

\section{Conclusions and Future Research}\label{sec:conclusion}
We develop an online learning theory for chain-of-thought verifiers that explicitly captures the asymmetric roles of soundness and completeness. Our Littlestone-style dimensions tightly characterize the optimal mistake bound under both a soundness budget and a linear asymmetric-cost objective. Exact calculations for concrete verifier classes show that the resulting Pareto frontier can be flat, linear, sharply discontinuous, or even infinite below a critical soundness budget. We further show that these verifier guarantees translate into reliable generation. Online verifiers can amplify collections of weak, history-dependent generators whose candidate distributions adapt to previous verifier feedback. Finally, our behavior-cloning lower bound demonstrates that this interaction is essential in general—attaining comparable soundness and completeness using only perfect-verifier transcripts may require $\Omega(\text{VCdim}(H))$ times (where $H$ is the verifier class) more training samples than our interactive approach.

\paragraph{Open questions.}
A strength of our algorithms is that they tightly characterize online learnability in terms of mistake bounds; however, they are primarily information-theoretic and may not be computationally efficient. An important direction is to identify natural verifier classes or oracle models for which efficient, or oracle-efficient, algorithms are possible, perhaps while sacrificing optimality in the Pareto frontier or linear-cost objectives. Here ideas from online learning with predictable sequences might be useful \cite{pmlr-v30-Rakhlin13}. Another interesting open question is to extend our results beyond the realizability assumption. In practice, human feedback may be noisy or the classes of verifiers may be misspecified, which raises the possibility of agnostic guarantees, such as regret-style or competitive bounds, that separately track unavoidable soundness and completeness mistakes from the additional mistakes incurred by the learner. Such guarantees may connect naturally with ideas from data-driven algorithm design, where the verifier is viewed as an algorithm selected from data \cite{gupta2016pac,Sharma2019LearningPL, balcan2020semi, 8555142, Balcan_2021}. A further direction is to extend the verification model itself. In our work, verifiers output a binary value that captures whether a partial reasoning trace is correct locally and they do not explicitly check for progress made toward the final solution, unless further assumptions are made on the oracle (see Section \ref{sec:boost}). This interpretation of verifier outputs is consistent with the works of \cite{uesato2022solving, lightman2023let}, although in  recent literature PRM outputs provide a real-valued score in $[0,1]$ for each intermediate step meant to represent the probability that a partial reasoning trace could be extended to a fully correct trace, which would be interesting to analyze from a learning-theoretic perspective. Finally, it would be useful to identify weaker or alternative conditions, beyond the $(\alpha,\gamma)$-good generator assumption in Section~\ref{sec:boost}, under which verifiers can be used for ``boosting'' generators. For example, can one obtain similar guarantees when success probabilities are step-dependent, when different generators specialize on different parts of the reasoning trace, or when the generator-verifier interaction is allowed limited backtracking or search? Characterizing the minimal assumptions on the generator needed for learned verifiers to amplify reasoning reliability remains an important open problem.

\section*{Acknowledgements}
This work was supported in part by the National Science Foundation under grants IIS-1901403, ECCS-2216899 and ECCS-2216970, Simons Investigator Award MPS-SICS-00826333, Office of Naval Research MURI Grant N000142412742, an OpenAI Superalignment grant, a DARPA AIQ grant and a Cooperative AI PhD Fellowship from the Cooperative AI Foundation.

 \clearpage

 \bibliographystyle{plainnat}
\bibliography{bibs}

\newpage
\appendix

\begin{appendices}

\section{Omitted Details and Proofs}
\subsection{Additional Related Work}
\label{app:additional-related}

We provide below more detailed background and context of relevant literature, organized by its relationship to our work.

{\bf Expressiveness of chain-of-thought reasoning.} Chain-of-thought prompting~\citep{wei2022chain}, where models produce step-by-step reasoning traces, has become a dominant paradigm for complex reasoning tasks~\citep{nye2021show,kojima2022large}. A growing theoretical literature explains why intermediate reasoning steps expand model capabilities. \citet{li2024chain} show that constant-depth transformers with chain-of-thought can solve inherently serial problems (such as composing permutations) that are impossible without intermediate steps, establishing a provable separation between standard and CoT-augmented transformers. \citet{merrill2024expressive} prove that transformers with polynomial-length chain-of-thought can express any problem solvable in polynomial time, and \citet{feng2024revealing} provide a theoretical perspective on why CoT helps by analyzing the circuit complexity of CoT-augmented computation. From a learning-theoretic perspective, \citet{malach2024auto} show that auto-regressive next-token predictors with intermediate tokens are universal learners. The CoT structure---exposing intermediate reasoning steps---is precisely what makes step-level verification both possible and meaningful, as it provides the sequential structure that our framework exploits.

{\bf Learning chain-of-thought reasoning and verification.} On the learnability side, \citet{joshi2025theory} formalize autoregressive learning with chain-of-thought, analyzing sample and computational complexity when reasoning chains are observed or latent, and \citet{altabaa2025cot} show that learning guarantees can be independent of chain length under suitable assumptions. More broadly, a growing line of work studies the sample efficiency and training dynamics of learning with chain-of-thought: \citet{wen2025sparse} show that CoT reduces sample complexity from exponential to polynomial on parity-type tasks by inducing sparse attention patterns; \citet{li2025training} provide generalization analysis for training nonlinear transformers with CoT; \citet{wang2025compositional} prove statistical query lower bounds for learning compositional functions without curriculum, but show polynomial sample complexity is achievable with curriculum learning; and \citet{kim2025parity} prove that with CoT supervision, parity can be learned in one gradient step, whereas without CoT supervision exponential iterations are needed. \citet{abbe2024globality} introduce the ``globality barrier'' to characterize when transformers can learn from scratchpads, showing that inductive scratchpads can break this barrier and enable out-of-distribution generalization. For verification, \citet{balcan2025learning} give learning-theoretic foundations for chain-of-thought verification from a statistical perspective, and \citet{amit2024models} study self-proving models where the generator generates outputs along with proofs of correctness.

In more detail, \citet{balcan2025learning} show that if the prompt-reasoning pairs are drawn from a fixed but unknown distribution (say chain-of-thoughts from a given trained LLM on random problems of interest), then it is possible to learn a verifier that correctly predicts the location of the first error in the reasoning (if one exists) with high probability, with a small sample complexity (linear in the VC dimension, $\log |H|$ for finite verifier function classes $H$). However, given feedback on mistakes in its reasoning, the LLM may produce a different reasoning which is out-of-distribution and we are no longer able to reliably verify. On the other hand, if we require the verifier to never accept an incorrect reasoning for prompts coming from a fixed distribution, they show a strong lower bound on the sample complexity, linear in the size of the verifier class $H$. Our work aims to develop a model for verification that overcomes this pessimistic lower bound by taking advantage of some weak assumptions on the correctness of LLM reasoning steps, while still allowing the LLM outputs to be adaptive.

{\bf Prover-verifier systems.} Our work connects to a growing line of literature on the intersection of interactive proof systems and machine learning. \citet{interactivegoldwasser} lay the foundation for interactive proof systems for PAC verification, where a verifier with sample access to data interacts with an untrusted prover in order to certify that a proposed hypothesis is approximately correct. \citet{amit2025theory} introduce a  framework complementary to ours and aim to train a prover that produces both an answer and auxiliary information that helps a fixed sound verifier check for correctness. Similar to our results on using verifiers to improve the performance of weak generators, the prover-verifier interaction is used as the training procedure and the resulting guarantees concern the correctness of the model’s behavior on individual inputs rather than on average. The key difference is that \citet{amit2025theory} assume access to a verifier with formal soundness guarantees and train the prover to produce outputs that this fixed verifier can certify, whereas we treat the generator as a fixed model, for which we aim to learn a reliable verifier. \citet{anil2022learning} introduce Prover-Verifier Games (PVGs) as a framework for inducing checkable solutions through the strategic interaction between a powerful prover and a weaker verifier, where both agents are learned. Subsequently, \citet{hammond2025neural} generalized this framework by considering multiple provers and IP protocols that allow for zero-knowledge proofs. A complementary line of work on AI safety via debate \cite{irving2018aisafetydebate, browncohen:debate} studies how interaction among competing provers can make verification easier for a limited verifier. While many works in the intersection of ML and IPs typically study protocols for checking or eliciting complete answers, our work focuses on the autoregressive, step-level setting where a learned verifier evaluates prefixes of a chain of thought as they are generated. 

{\bf Imitation learning and behavior cloning.} Our framework for using a verifier to improve the performance of a weak generator is closely connected to a long line of work on \emph{learning from demonstrations}. Behavior cloning is a form of imitation learning in which an agent learns a policy (here a verifier) directly from expert demonstrations. Given a dataset of observation–action pairs collected from an expert, behavior cloning treats policy learning as a supervised learning problem: the learned policy is trained to predict the expert’s action for each observed state
\citep{pomerleau1991efficient,bain1995framework}. The learner never observes states off this expert-induced
distribution, which is a limitation that on-policy methods such
as \textsc{DAgger} are designed to avoid
\citep{ross2010efficient,ross2011reduction}. Specifically, \citet{ross2010efficient} and \citet{ross2011reduction} showed that imitation learning can be viewed as a reduction from sequential decision making to \emph{online learning}, where the learner must be trained on states induced by its own behavior in order to avoid compounding errors from distribution shift. This perspective is closely aligned with our motivation: in verifier-guided reasoning, the distribution of reasoning prefixes observed by the verifier is not fixed in advance, but is shaped by the interaction between the generator and the verifier’s feedback. However, in our setting, we aim for fine-grained guarantees that capture the asymmetric role of soundness and completeness errors made by the learned verifier to the performance of the generator-verifier system, thus requiring a tailored  analysis.
We make this separation formal in Section~\ref{sec:bc-lower-bound}, by showing that behavior cloning on perfect-verifier transcripts can fail for a feedback-dependent generator on which our interactive learner succeeds. The separation also exposes an asymmetric compounding-error phenomenon---a completeness error may be relatively benign immediately, yet alter the generator's future behavior and thereby cause a later soundness error.

{\bf Theoretical foundations of post-training.} A parallel line of recent theory studies when and how post-training can improve reasoning. The \emph{sharpening} framework~\citep{huang2024sharpening} and the \emph{coverage principle}~\citep{chen2025coverageprinciple} show that sample-and-evaluate approaches (such as best-of-$N$) require the base model to place non-negligible probability mass on entirely correct sequences, and that this \emph{sequence-level} coverage governs the success of post-training and test-time scaling. \citet{tsilivis2025rlafterntp} provide a complementary account of why reinforcement learning after next-token prediction can unlock longer chain-of-thought behaviors, though it still relies on correct trajectories not being exponentially rare. These results operate at the \emph{sequence level}: they analyze selection or reweighting of complete reasoning traces.
In contrast, our work and~\citet{rohatgi2026taming} take a \emph{step-level} perspective. \citet{rohatgi2026taming} show how to use a (given) process verifier to guide backtracking via a random walk on partial reasoning traces. Our work addresses the more fundamental question of how to \emph{learn} such a step-level verifier, with formal guarantees that hold for arbitrary sequences of reasoning traces---whether generated by a fixed generator, an evolving generator undergoing RL training, or an adversary. 

{\bf Online learning and mistake bounds.} The theoretical foundations of online learning are closely tied to combinatorial complexity measures, most notably the Littlestone dimension~\citep{littlestone1988learning}, which characterizes learnability and optimal mistake bounds in adversarial online classification. In his seminal COLT 1989 paper, \citet{littlestone:1989} established an online-to-batch conversion, showing that online learnability implies PAC learnability with sample complexity controlled by the online mistake bound. A key challenge in using Littlestone’s argument in our setting of boosting weak generators is that the reasoning steps depend on the problem and context and are not i.i.d. Recent work has significantly refined this picture: randomized variants of the Littlestone dimension characterize expected mistake bounds for randomized learners \citep{filmus23randomized}, while computability-aware, strategy-aware and effective versions clarify when finite-mistake learning is achievable by constrained learners \citep{hasrati23computable,ahmadi2024strategic,delle-rose25effective}. Another extension related to abstention-error trade-off inspires our formulations for different types of verifier mistakes~\citep{sayedi2010trading,zhang2016extended}. More broadly, unified approaches to online learning connect regret minimization with competitive analysis and other sequential decision-making frameworks \citep{buchbinder2012unified}. Parallel to these developments, the online learning framework has been extended to sequential prediction problems with temporal structure \citep{pmlr-v30-Rakhlin13}, where the adversary's behavior is partially forecastable, e.g., given by a noisy autoregressive function.

A complementary line of work studies how the information available to an online learner affects learnability and mistake bounds. \citet{bendavid1997online} considered variants of the mistake-bound model in which the learner receives different degrees of advance information about the instance sequence, revealing fundamental distinctions between online and offline learning. More recently, \citet{chase2025optimal} characterized the value of unlabeled information in transductive online learning in terms of Littlestone dimension, establishing tight mistake bounds and further demonstrating how the interaction protocol can alter online learnability. At the other end of the information spectrum, the apple-tasting model studies online classification with one-sided feedback \citep{helmbold2000apple}. \citet{chase2025deterministic} recently established deterministic mistake bounds for this setting in terms of Littlestone dimension. These works are closely related to our analysis of verifier learning, where the information revealed to the learner is likewise determined by the interaction protocol. Our generator-verifier perspective is also conceptually related to interactive verification, where a verifier interacts with a more powerful prover to certify a computation \citep{goldwasser2015delegating,cormode2011verifying}. In our setting, however, the verifier must itself be learned, and the sequence of reasoning steps is induced by its interaction with the generator, making it necessary to control verifier mistakes along adaptively generated reasoning traces.

{\bf Process Reward Models and empirical work.} A substantial body of empirical work has developed   \emph{process reward models} (PRMs) for chain-of-thought verification. This literature pursues two distinct approaches. The first trains PRMs on \emph{human annotations} of step-level logical correctness. This includes \citet{uesato2022solving}  and \citet{lightman2023let}, who demonstrate that PRMs significantly outperform outcome reward models (supervision with only final answers), motivating the need to develop and study chain-of-thought verifiers. {More recent empirical work further studies step-level supervision, step-search, and error localization in PRM-style systems and benchmarks~\citep{zhang2025lessons,zheng2024processbench,ma2023letsreward,khalifa2025thinkprm,zhao2025genprm,zhang2024genrm,jacovi-etal-2024-chain}.} These human-annotated PRMs effectively approximate ground-truth verifiers. The second approach avoids step-by-step human annotation by training PRMs via \emph{Monte Carlo estimation}~\citep{wang2024mathshepherd,luo2024omegaprm}, where a step’s label is the fraction of sampled completions that reach the correct answer, or via \emph{implicit reward derivation} from outcome models~\citep{cui2025prime}. These produce generator-specific value functions rather than ground-truth verifiers. For instance, a logically flawed step may score highly if the generator can recover from it. PRMs are widely used to improve chain-of-thought by best-of-$N$ reranking and tree search~\citep{lightman2023let,snell2024scaling}, and as dense reward signals for reinforcement learning~\citep{setlur2025rewarding,cui2025prime}. Our work complements the growing empirical evidence supporting the power of verifier-guided reasoning by providing formal learnability guarantees, analyzing the soundness-completeness tradeoff and showing provable amplification of correctness for weak reasoners using learned verifiers.

{\bf Learning with abstention and selective classification.}
A long line of work equips classifiers with the option to abstain rather than risk a mistake, trading error against coverage
~\citep{chow1970optimum,rivest1988learning,elyaniv2010foundations,cortes2016learning}. Our soundness--completeness formulation is most
directly related to the \emph{online} branch of this literature \citep{sayedi2010trading,zhang2016extended}, where given a
budget on classification mistakes the player seeks to minimize abstentions. At a high-level this seems similar to parts of our work, where given a
budget on soundness mistakes the player seeks to minimize completeness mistakes. However, our chain-of-thought verification problem is fundamentally different in structure than binary classification with abstentions and it is only through our two-way online to online reduction from chain-of-thought to prefix verification that we can establish a connection between the two settings. We also extend our results way beyond just characterizing the trade-off between the two types of mistakes by studying a distinct linear-cost objective not captured by the above budgeted abstention analogy. For arbitrary costs on soundness and completeness mistakes, we introduce the WSC Littlestone dimemsion and prove that it exactly characterizes the optimal cumulative loss. 
We also study the downstream effects of the soundness-completeness trade-off to the  reliable generation of reasoning traces.

Several works on robust and reliable prediction are also close in spirit.
\citet{goldwasser2020beyond} output a selective classifier that guarantees low
error on arbitrary, potentially adversarial test examples while controlling its
rejection rate on the training distribution. Similarly,
\citet{balcan2022robustly} introduce \emph{robustly-reliable learning} under
data-poisoning attacks, where the learner predicts only when it can certify that its label is correct under a specified corruption budget, and otherwise abstains.
They characterize the region on which such predictions can be certified,
including against instance-targeted attacks. \citet{balcan2023reliable} extend
the reliable-learning perspective to challenging test-time environments,
including adversarial test-time attacks and natural distribution shifts. These works share with ours the goal of separately controlling harmful
incorrect predictions and the frequency with which the learner declines to
make a useful prediction. \citet{balcan2023analysis} make non-Lipschitz networks robust by learning to abstain on inputs perturbed within random low-dimensional subspaces and give data-driven guarantees for tuning the resulting accuracy--abstention trade-off. This is related in spirit to our asymmetric treatment of soundness and completeness, but differs from our online minimax formulation. In Section~\ref{sec:lscmistakes}, for arbitrary costs on soundness and completeness mistakes, we introduce the WSC-Littlestone dimension and show that it exactly characterizes the optimal cumulative loss. 

Abstention also appears in imitation learning, although often through stopping, expert intervention, or deferral rather than a conventional ``don't-know'' label. Most directly, \citet{goel2026learning} study \emph{selective imitation learning} under arbitrary dynamics shift, where a policy may stop when it can no longer act reliably. Their goal is to act well before stopping while rarely stopping in the training environment. Related interactive-imitation methods such as SafeDAgger \citep{zhang2016query}, Expert Intervention Learning \citep{spencer2022expert}, and ThriftyDAgger \citep{hoque2022thriftydagger} selectively query an expert or transfer control when the learner is uncertain or likely to act unsafely.

\subsection{Background: Littlestone Dimension}\label{app:mistake-tree}
\noindent We give an overview of the concepts and techniques from online classification \citep{littlestone1988learning, daniely2011multiclass} that are relevant to our work. We begin with the notion of \emph{shattered mistake trees}.
\paragraph{Mistake tree.} A mistake tree is a rooted binary tree, where each internal node corresponds to examples in some domain set $X$ and each edge corresponds to a label in a a set $Y$ (for example in binary classification $Y = \{-1, 1\}$ and in multiclass classification $Y = [L]$), such that the outgoing edges from a single node have different labels. We say that a mistake tree is \emph{shattered} by a hypotheis class $H$ if for any root-to-leaf path that traverses internal nodes $x_1, x_2, \dots, x_d$ there exists a hypothesis $h \in H$ such that the label of edge $(x_i, x_{i+1})$ is $h(x_i)$ for all $i \leq d$.

Intuitively, a mistake tree describes the strategy of an adversary against a deterministic learning algorithm. Internal nodes correspond to the examples that the adversary will present to the learner starting from the root. The pair of outgoing edges corresponds to the adversary's ability to induce a mistake no matter what the learner predicts. For example, if $Y = \{-1, 1\}$ and the learner predicts $-1$, the adversary can reveal label $+1$ and follow the corresponding downward edge to the next example. Therefore, the interaction between learner and adversary corresponds to a root-to-leaf path and  shattering ensures that the sequence of examples presented to the learner is realizable by the hypothesis class $H$.

The Littlestone dimension captures the intuition that the more complex a hypothesis class is, the deeper the mistake trees it should be able to shatter.

\begin{definition}[Littlestone dimension]
    The Littlestone dimension of a hypothesis class $H$, denoted as $\ldim(H)$, is the maximum integer $d$ such that there exists a complete mistake tree of depth $d$ that is shattered by $H$.
\end{definition}

With the preceding understanding of a shattered mistake tree, it immediately follows that the Littlestone dimension lower bounds the number of mistakes of any deterministic learning algorithm. \citet{littlestone1988learning} showed that there exists an online learning algorithm, the Standard Optimal Algorithm (SOA), that achieves this lower bound. These two results completely characterize the optimal mistake bounds in online classification. Later, \citet{daniely2011multiclass} extended these results to online multiclass classification.

\subsection{Summary of Algorithms for Chain-of-Thought Verification}
\SetAlgoNoEnd
\begin{algorithm}[H]
\caption{Chain-of-Thought verification with soundness budget $k$}
\label{alg:cot-sc-soa}

\DontPrintSemicolon
\SetAlgoLined
\LinesNumbered
\SetNoFillComment
\SetKwComment{tcp}{\(\triangleright\) }{}
\SetKwInOut{KwIn}{Input}

\KwIn{class of verifiers $H$; soundness mistake budget $k$.}

$H^{(1)} \gets H$\tcp*[r]{initialize version space}

\For{$t \in [T]$}{

  Observe example $z^{(t)} = (x^{(t)},\tau^{(t)} =(\tau^{(t)}_1,\ldots,\tau^{(t)}_L)) \in X \times \Sigma^L$
  
  \uIf{$k=0$}{
    Predict $\hat{y}^{(t)} \leftarrow \min \{\ell \in [L] \; \text{such that} \; 
    \exists h \in H^{(t)} : h(x^{(t)},\tau_1,\ldots,\tau_\ell)=\no \}$\;
  }
  \Else{
    \For{$\ell \in [L]$}{
      Compute $m_c \gets \mathrm{SC\mbox{-}Ldim}\!\left(
      \{\, h \in H^{(t)} : h(x,\tau_1,\ldots,\tau_\ell)=\yes \,\},\, k
      \right)$
      
      Compute $m_s \gets \mathrm{SC\mbox{-}Ldim}\!\left(
      \{\, h \in H^{(t)} : h(x,\tau_1,\ldots,\tau_\ell)=\no \,\},\, k-1
      \right)$
      
      \If(\tcp*[f]{includes the case where all $h\in H^{(t)}$ predict $\no$}){$m_c \le m_s$}{ 
        Predict $\hat{y}^{(t)} \gets \ell$\;
        \textbf{break}\;
      }
      
      \If{$\ell = L$}{
      Predict $\hat{y}^{(t)} \gets \infty$\;
    }
  }
  }

  Receive correct label $\hat{y}^{(t)} \in [L] \cup \{ \infty\}$\;

  \uIf(\tcp*[f]{soundness mistake}){$y^{(t)} < \hat{y}^{(t)}$}{
    {\sc Update} $H^{(t+1)} \gets \{\, h \in H^{(t)} : h(x^{(t)},\tau^{(t)}_1,\ldots,\tau^{(t)}_{y^{(t)}})=\no \,\}$\;
    Set $k \gets k-1$ \tcp*[r]{update soundness mistake budget}
  }
  \ElseIf(\tcp*[f]{completeness mistake}){$y^{(t)} > \hat{y}^{(t)}$}{
     Update $H^{(t+1)} \gets \{\, h \in H^{(t)} : h(x^{(t)},\tau_1,\ldots,\tau_{\hat{y}^{(t)}})=\yes \,\}$\;
  }
}
\end{algorithm}

\SetAlgoNoEnd
\begin{algorithm}[t]
\caption{Chain-of-Thought verification with linear cost objective}
\label{alg:cot-wsc-soa}

\DontPrintSemicolon
\SetAlgoLined
\LinesNumbered
\SetNoFillComment
\SetKwComment{tcp}{\(\triangleright\) }{}
\SetKwInOut{KwIn}{Input}

\KwIn{class of verifiers $H$; mistake costs $\gamma_s, \gamma_c \in \R_{\geq 0}$.}

$H^{(1)} \gets H$\tcp*[r]{initialize version space}

\For{$t \in [T]$}{

  Observe example $z^{(t)} = (x^{(t)},\tau^{(t)} =(\tau^{(t)}_1,\ldots,\tau^{(t)}_L)) \in X \times \Sigma^L$
  
 \For{$\ell \in [L]$}{
  
    Compute $\cY^{(t)} = \{ h(x^{(t)}, \tau^{(t)}_{1:\ell}) \; | \; h \in H^{(t)} \}$\tcp*[r]{set of possible outputs}
    
    \eIf(\tcp*[f]{all consistent verifiers reject}){$\cY^{(t)} = \{ \no \}$}{
       Predict $\hat{y}^{(t)} \gets \ell$\;
        \textbf{break}\;
    }{
        Compute $m_c = \gamma_c + \mathrm{WSC\text{-}Ldim}\bigl(H^{(t)}[((x^{(t)}, \tau^{(t)}_{1:\ell}), \yes)]\bigr)$\;

        Compute $m_s = \gamma_s + \mathrm{WSC\text{-}Ldim}\bigl(H^{(t)}[((x^{(t)}, \tau^{(t)}_{1:\ell}), \no)]\bigr)$\;

        \If(\tcp*[f]{minimize WSC-Ldim of future version space}){$m_c \le m_s$}{
            Predict $\hat{y}^{(t)} \gets \ell $\;
            \textbf{break}
            }
        \If{$\ell = L$}{
        Predict $\hat{y}^{(t)} \gets \infty$\;
}
        }
    }
    Receive correct label $y^{(t)} \in [L] \cup \{ \infty\}$\;

  \uIf(\tcp*[f]{soundness mistake}){$y^{(t)} < \hat{y}^{(t)}$}{
    Update $H^{(t+1)} \gets \{\, h \in H^{(t)} : h(x^{(t)},\tau^{(t)}_1,\ldots,\tau^{(t)}_{y^{(t)}})=\no \,\}$
  }
  \ElseIf(\tcp*[f]{completeness mistake}){$y^{(t)} > \hat{y}^{(t)}$}{
    Update  $H^{(t+1)} \gets \{\, h \in H^{(t)} : h(x^{(t)},\tau_1,\ldots,\tau_{\hat{y}^{(t)}})=\yes \,\}$\;
  }
  }
\end{algorithm}

\subsection{Ommitted Proofs and Additional Examples for SC-Littlestone dimension (Section~\ref{sec:frontier-examples})}\label{app:examples}

\subsubsection{Linear Threshold Verifiers with a Margin}

Below we present an example of an infinite class of verifiers that can be learned with a finite mistake bound. This is an adaptation of a similar example presented in prior work \cite{joshi2025theory, balcan2025learning}.

    Let $X, \Sigma = \{x \in \R^d: \|x\|_2 \leq r \}$ and consider the class of $d$-dimensional linear threshold verifiers:
    \begin{align*}
        H = \{h_{w_0, w}(x, \tau) \mapsto \mathbbm{1}\{w_0x + w[-l:]\tau[-l:]
         \geq 0 \} \; | \; w_0
          \in \R, w \in \R^d, l = \min \{d-1, |\tau| \} \}
    \end{align*}
    Now for a margin $\gamma>0$, we define a sequence 
    $(x^{(1)}, \tau_{1:L}^{(1)}, y^{(t)}), \dots, (x^{(T)}, \tau_1:L^{(T)}, y^{(T)}) \in X \times \Sigma^L \times [L] \cup \{ \infty \}$ to be ``margin-$\gamma$ realizable by $H$'' if there exists a single $h_{w_0^*, w^*} \in H$ such that
    \begin{align*}
        y^{(t)} = \min \{l: w^*_0x^{(t)} + w^*[-l:]\tau^{(t)}[-l:] < 0\} \;\;  \text{and} \;\; |\langle 
        (w^*_0, w^*), (x^{(t)}, \tau^{(t)}) \rangle| \geq \gamma.
    \end{align*}
     By classical results  on learning linear thresholds with a margin \cite{rosenblatt1958perceptron}, it follows that if we restrict mistake trees to only contain nodes in $\{(x, \tau) : |(w^*_0, w^*), (x, \tau)| \geq \gamma\}$ then $\ldim(H) \leq O(R^2/\gamma^2)$.

\subsubsection{Layered Product Verifiers}

\twoscalefrontier*

\begin{proof}
For $R_{n,b}$, we have $\Ldim(R_{n,b})=b$. For any root-to-leaf path in a mistake tree with at most $b$ distinct points, there exists a function $h_S \in R_{n,b}$, where $S$ contains all negative points along the path and none of the positive ones. Thus a depth-$b$ mistake tree can be shattered by $R_{n,b}$. For the upper bound, a learner that predicts $\yes$ on each new point makes at most $b$ mistakes, so $\Ldim(R_{n,b})=b$. Similarly for the singleton-rejector class, $\Ldim( R_{n'})=1$. Since these two classes act on disjoint layers, Proposition~\ref{prop:additive}(i) gives $\Ldim(H_\beta)=b+1$.

\medskip \noindent \emph{Zero soundness budget.} A sound learner for $R_{n,b}$ must reject every point that the learner is uncertain about. An adversary can therefore reveal all $n-b$ correct points before revealing a member of $S$, so $\SCLdim(R_{n,b},0)=n-b$. Similarly, Proposition~\ref{prop:rn} gives $\SCLdim( R_{n'},0)=n'-1$. Proposition \ref{prop:additive}(iii) therefore yields $\SCLdim(H_\beta,0)=(n-b)+(n'-1)$. 

\medskip \noindent \emph{Budgets $1\leq k\leq b$.} For the upper bound, a learner can reject every new point in $\Sigma$ and accept every new point in $\Sigma'$. The first layer then causes at most $n-b$ completeness mistakes, while the second causes at most one soundness mistake, namely when $\sigma^\star$ is first encountered. Thus this learner uses only one unit of soundness budget and makes at most $(n-b)+1$ mistakes, so $\SCLdim(H_\beta,k)\leq(n-b)+1$ for every $k\geq1$.

For the matching lower bound, the adversary first plays only on the $\Sigma$ layer, querying previously unseen points. Whenever the learner predicts $\no$, the adversary may reveal $\yes$, forcing a completeness mistake and declaring that point to lie outside $S$. Whenever the learner predicts $\yes$, the adversary may reveal $\no$, forcing a soundness mistake and declaring that point to lie in $S$. This remains realizable until either the learner has made $n-b$ completeness mistakes or it has made $b$ soundness mistakes.

If $k<b$, the latter event cannot occur while also respecting the budget constraint, so the adversary forces $n-b$ completeness mistakes on the first layer. It then queries a new point of $\Sigma'$. No matter what the learner predicts, the adversary can force one more mistake for a total of $(n-b)+1$. It remains to consider $k=b$. If the adversary has already forced $n-b$ completeness mistakes on $\Sigma$, one query in $\Sigma'$ again gives the desired total. Otherwise, the learner has made $b$ soundness mistakes on $\Sigma$ and exhausted its budget. The adversary can then present the $n'-1$ correct points of $\Sigma'$ before $\sigma^\star$, forcing $n'-1$ additional completeness mistakes. The resulting total is $b+n'-1$, which is at least $(n-b)+1$ by the assumption $n'\geq n-2b+2$. Hence $\SCLdim(H_\beta,k)\geq(n-b)+1$ throughout $1\leq k\leq b$.

Notice that Proposition~\ref{prop:additive}(ii) alone gives only the general superadditive lower bound obtained by assigning the same budget to both layers. The sharper middle regime above comes from the fact that the single global unit of soundness budget can be spent entirely on the second layer, while the first layer is learned soundly.  

\medskip \noindent \emph{Budgets $k\geq b+1$.} Apply Lemma \ref{lem:general-shape}(iii) with $\Ldim(H_\beta)=b+1$.
\end{proof}
\subsubsection{Subspace Verifiers}\label{sec:subspace}

\subspace*

Call a pair of
subspaces $(S,T)$ of $\R^d$ with $S\subset T$ and $\dm T=\dm S+1$ a
\emph{flag} (a term from algebraic geometry for referring to nested subspaces), and let $\chi_{S,T}(v)=\mathbf{1}[v\in T\setminus S]$.

\begin{lemma}\label{lem:normal}
Let $z=(x,\tau_1)$, $S=\spn(x)$, $T=\spn(x,\tau_1)$. If $\tau_1\in S$ then
$h_v(z)=\yes$ for every $v$. If $\tau_1\notin S$ then $(S,T)$ is a flag and
$h_v(z)=\chi_{S,T}(v)$ for every $v$. Conversely, every flag $(S,T)$ of
$\R^{d}$ with $\dm S\le d'$ is realized by some example; since
$d'\ge d-1$, this covers all flags of $\R^{d}$.
\end{lemma}

\begin{proof}
If $\tau_1\in S$ then $\tau_1\in S\subseteq\spn(x,v)$ for every $v$.
Otherwise $T=S\oplus\langle \tau_1\rangle$. If $v\in S$ then
$\spn(x,v)=S\not\ni \tau_1$, so $h_v(z)=\no$. If $v\in T\setminus S$,
write $v=s+c\tau_1$ with $s\in S$ and $c\neq 0$; then
$\tau_1=(v-s)/c\in\spn(x,v)$, so $h_v(z)=\yes$. If $v\notin T$ and
$x_1=s+bv$ for some $s\in S$, $b\in\R$, then $b\neq 0$ since
$\tau_1\notin S$, hence $v=(\tau_1-s)/b\in T$, a contradiction; so
$h_v(z)=\no$. For the converse, take $x$ with columns spanning $S$ and
$\tau_1\in T\setminus S$.
\end{proof}

Constant examples cannot label internal nodes of shattered trees, so all
internal nodes below are flag examples. We also record how flags trace on
a subspace.

\begin{lemma}\label{lem:trace}
Let $V\subseteq\R^d$ be a subspace and $(S,T)$ a flag of $\R^d$. On
$V\setminus\{0\}$, either $\dm(T\cap V)=\dm(S\cap V)$ and $\chi_{S,T}$
restricts to the constant $\no$, or $\dm(T\cap V)=\dm(S\cap V)+1$ and it
restricts to $\chi_{S\cap V,\,T\cap V}$, the classifier of a flag of
$V$. Conversely, every flag of $V$ is the restriction of an ambient flag
example.
\end{lemma}

\begin{proof}
Set-theoretically $(T\setminus S)\cap V=(T\cap V)\setminus(S\cap V)$,
and $(T\cap V)/(S\cap V)$ embeds into the one-dimensional $T/S$, so the
dimension gap is $0$ or $1$; the two cases follow. For the converse, a
flag $(S',T')$ of $V$ satisfies $S'\cap V=S'$ and $T'\cap V=T'$, so the
ambient example realizing it (Lemma~\ref{lem:normal}) restricts to
$\chi_{S',T'}$.
\end{proof}

\begin{lemma}\label{lem:ldim-ub}
$\Ldim(\Hd)\le d-1$.
\end{lemma}

\begin{proof}
We prove by strong induction on $r$ that, for every $r$-dimensional subspace
$V$ and every nonempty $P\subseteq V\setminus\{0\}$, $\Ldim\bigl(\{h_v:v\in P\}\bigr)\le r-1$.

When $r=1$, all nonzero vectors in $V$ are scalar multiples of one
another. By scale invariance, they define the same hypothesis, so the
Littlestone dimension is zero. Now let $r\ge 2$, and suppose that $\{h_v:v\in P\}$ shatters a complete
tree of depth $m$. By Lemmas~\ref{lem:normal} and~\ref{lem:trace}, the
root restricts to the classifier of a flag $(S',T')$ in $V$. If $T'=V$, then the parameters of all hypotheses taking the $\no$-branch
lie in $S'$, which has dimension $r-1$. If $T'\subsetneq V$, then the
parameters of all hypotheses taking the $\yes$-branch lie in $T'$, whose
dimension is at most $r-1$. Thus, in either case, one of the two subtrees has depth $m-1$ and is
shattered by hypotheses whose parameters lie in a subspace of dimension
at most $r-1$. The induction hypothesis gives $m-1\le r-2$,
and hence $m\le r-1$. Taking $V=\R^d$ and $P=\R^d\setminus\{0\}$ proves the result.
\end{proof}

For the lower bound, we use the standard fact that a vector space over an infinite field cannot be covered by finitely many proper subspaces.

\begin{lemma}\label{lem:obstacle}
$\Ldim(\Hd)\ge d-1$.
\end{lemma}

\begin{proof}
Let $V\subseteq\R^d$ be a subspace with $\dm V=r\ge 1$ and let
$\mathcal{A}$ be a finite family of proper subspaces of $V$. We will show that there
is a complete mistake tree of depth $r-1$with the following properties: (1) every internal node corresponds to a flag $(S', T')$ with $T'\subseteq V$, (2) the tree is shattered by $\Hd$, and (3) every
root-to-leaf path is consistent with some $h_v$ satisfying
$v\in V\setminus(\{0\}\cup\bigcup_{A\in\mathcal{A}}A)$. Then taking $V=\R^d$, $\mathcal{A}=\emptyset$ will give the desired statement.

We proceed by induction on $r$. If $r=1$, every proper subspace of $V$
is $\{0\}$. Thus a tree consisting of a single leaf has the required
depth, and its unique path is realized by any $v\in V\setminus\{0\}$.

Now suppose that $r\ge 2$. Choose distinct hyperplanes $W,W'$ of $V$
such that neither belongs to $\mathcal{A}$. This is possible because
$V$ has infinitely many hyperplanes, whereas $\mathcal{A}$ is finite.
Place at the root an example realizing the flag $(W,V)$. By
Lemma~\ref{lem:normal}, its classifier on $V$ is $v\longmapsto \mathbf{1}[v\notin W]$. For each $A\in\mathcal{A}$, the intersection $A\cap W$ is a proper
subspace of $W$. Indeed, if $W\subseteq A$, then the fact that $W$ is a
hyperplane and $A$ is a proper subspace of $V$ would imply $A=W$,
contrary to our choice of $W$. We may therefore apply the induction
hypothesis in $W$ with the family $\{A\cap W:A\in\mathcal{A}\}$.
Attach the resulting tree below the $\no$-edge of the root. Every path
in this subtree is realized by some $v\in
    W\setminus
    \left(
        \{0\}\cup
        \bigcup_{A\in\mathcal{A}}(A\cap W)
    \right)$.
Since $v\in W$, it receives label $\no$ at the root; and since
$v\notin A\cap W$ for every $A\in\mathcal{A}$, it also avoids every
member of $\mathcal{A}$. Similarly, every $A\cap W'$ is a proper subspace of $W'$, and
$W\cap W'$ is proper in $W'$ because $W\neq W'$. Apply the induction
hypothesis in $W'$ with the family $\{A\cap W':A\in\mathcal{A}\}\cup\{W\cap W'\}$,
and attach the resulting tree below the $\yes$-edge of the root. Every
path in this subtree is realized by some $v\in
    W'\setminus
    \left(
        \{0\}\cup
        \bigcup_{A\in\mathcal{A}}(A\cap W')
        \cup (W\cap W')
    \right)$.
Such a vector avoids every member of $\mathcal{A}$. Moreover,
$v\in W'$ and $v\notin W\cap W'$ imply $v\notin W$, so it receives
label $\yes$ at the root. By induction, both subtrees are complete of depth $r-2$, and all their
flag examples satisfy $T'\subseteq V$. The resulting tree is therefore
complete of depth $r-1$ and has all the required properties. Taking $V=\R^d$ and $\mathcal{A}=\emptyset$ gives $\Ldim(\Hd)\ge d-1$.
\end{proof}

\begin{lemma}\label{lem:difficult}
Let $V\subseteq\R^d$ be a subspace with $\dm V=r$, let $k\ge 0$ satisfy
$r\ge k+2$, let $\mathcal{A}$ be a finite family of proper subspaces of
$V$, and let $m\in\N$. Then there is a $(k,m)$-difficult SC-mistake tree, such that (1) every internal node corresponds to a flag $(S', T')$ with  $T'\subseteq V$, (2) the tree is shattered by $\Hd$, and (3) every path is consistent with some $h_v$ satisfying $v \in V\setminus(\{0\}\cup\bigcup_{A\in\mathcal{A}}A)$.
\end{lemma}

\begin{proof}
The claim is immediate when $m=0$, so assume that $m\ge 1$. We proceed
by induction on $k$.

Choose pairwise distinct hyperplanes $W_1,\ldots,W_m$ of $V$, none of
which belongs to $\mathcal{A}$. For each $j$, let $z_j$ be an example
realizing the flag $(W_j,V)$. Arrange these examples along a path of
curvy edges $z_1\rightarrow z_2\rightarrow\cdots\rightarrow z_m$,
and let the curvy child of $z_m$ be a leaf. By
Lemma~\ref{lem:normal}, the classifier at $z_j$ is $v\longmapsto \mathbf{1}[v\notin W_j]$. Thus, the curvy edge corresponds to the label $\yes$, and the straight
edge corresponds to the label $\no$.

Suppose first that $k=0$. Make the straight child of every $z_j$ a
leaf. The path consisting entirely of curvy edges is consistent with an $h_v$ where $v\in
    V\setminus
    \left(
        \{0\}\cup
        \bigcup_{j=1}^m W_j\cup
        \bigcup_{A\in\mathcal{A}}A
    \right)$.
Such a vector exists because a finite union of proper subspaces cannot
cover $V$. The path that takes curvy edges at $z_1,\ldots,z_{j-1}$ and the straight
edge at $z_j$ is realized by any $h_v$ such that $v\in
    W_j\setminus
    \left(
        \{0\}\cup
        \bigcup_{i<j}(W_i\cap W_j)\cup
        \bigcup_{A\in\mathcal{A}}(A\cap W_j)
    \right)$.
Each subspace in this union is proper in $W_j$. Indeed,
$W_i\cap W_j$ is proper because $W_i\neq W_j$, and $A\cap W_j$ is
proper because $W_j\notin\mathcal{A}$. Hence such a vector exists. It
lies outside $W_i$ for every $i<j$, so it takes the preceding curvy
edges, and it lies in $W_j$, so it takes the straight edge at $z_j$.
It also avoids every $A\in\mathcal{A}$. Thus the tree is shattered.
Moreover, the only path with no straight edges has length $m$, so the
tree is $(0,m)$-difficult.

Now suppose that $k\ge 1$. For each $j$, attach below the straight edge
at $z_j$ the tree given by the induction hypothesis for $\left(
        W_j,\;
        k-1,\;
        m-j,\;
        \{W_i\cap W_j:i<j\}
        \cup
        \{A\cap W_j:A\in\mathcal{A}\}
    \right)$.
The induction hypothesis applies because $\dm W_j=r-1\ge (k-1)+2$,
and all the listed subspaces are proper in $W_j$.

The path consisting entirely of curvy edges is realized as in the base
case. Consider any other path, and suppose that its first straight edge
occurs at $z_j$. By the induction hypothesis, the remainder of the path
is realized by some $h_v$ with $v\in W_j$ such that $v\notin W_i\cap W_j$ for $ i<j$ and $v\notin A\cap W_j$ for $A\in\mathcal{A}$.
Since $v\in W_j$, it takes the straight edge at $z_j$. Moreover,
$v\notin W_i\cap W_j$ implies $v\notin W_i$ for every $i<j$, so it
takes all preceding curvy edges. Similarly,
$v\notin A\cap W_j$ implies $v\notin A$ for every
$A\in\mathcal{A}$. Therefore, every path is realized by a hypothesis
with the required parameter, and the tree is shattered.

Finally, consider a path containing at most $k$ straight edges. If it
follows the entire curvy path, then its length is $m$. Otherwise, let
$z_j$ be the first node at which it takes a straight edge. The path has
length $j$ upon entering the attached subtree and contains at most
$k-1$ further straight edges. Since that subtree is
$(k-1,m-j)$-difficult, the remaining part of the path has length at
least $m-j$. Hence the total length is at least $j+(m-j)=m$. Thus the tree is $(k,m)$-difficult.
\end{proof}

\begin{proof}[Proof of Theorem~\ref{thm:subspace}]
For $0\le k\le d-2$, apply Lemma~\ref{lem:difficult} with $V=\R^d$,
$\mathcal{A}=\emptyset$, and arbitrary $m$. For $k\ge d-1=\Ldim(\Hd)$
(Lemmas~\ref{lem:ldim-ub} and~\ref{lem:obstacle}), apply
Lemma~\ref{lem:general-shape}(c).
\end{proof}

\begin{remark}\label{rem:b5}
 We  note
that the two-sided learner given in Example~B.5 of \cite{balcan2025learning} in fact achieves $d$
total mistakes rather than $d+1$ by slightly tightening their argument  (once its maintained subspace
$S^*\ni h^*$ has dimension $1$, all of its branches predict perfectly),
and its remaining one-mistake gap to the optimum $d-1$ stems from its
opening rule of predicting $\no$ until the first mistake, which lets the
adversary collect one uninformative completeness mistake on an example
with $\spn(x,\tau_1)=\R^d$. The algorithm of Theorem~\ref{thm:scub} does not pay it.
\end{remark}
\subsection{Omitted Proofs for Weighted SC-Littlestone dimension (Section~\ref{sec:wsc-examples})}\label{sec:omitted-proofs}

\begin{proposition}\label{prop:zerolocation}
    For $\gamma_c = 1$, $\gamma_s = d$, and any finite class $H$, we have that $\text{WSC-Ldim}(H) \leq O\left( \frac{d}{\log d} \log(|H|) \right)$.
\end{proposition}

\begin{proof}
    For a fixed size of a verifier class, $H$, we will show an upper bound on the minimum weighted path among all WSC-mistake trees with at most $|H|$ leaves. For simplicity, we only concern ourselves with the structure of the binary tree, as we can populate it with internal node instances and show that for each path there exist a consistent verifier ad hoc.

    To maximize the weight of the minimum path, we construct the tree as follows. Starting from the root node, we greedily split the node, $v$ with the minimum root-to-$v$ cumulative weight. Each node is split by adding an s-edge with weight $d>1$ and a c-edge with weight $1$.
    Letting $L(w) \in \N$ be the minimum number of leaf nodes needed to construct a tree with minimum path weight at least $w \in \R_{\geq 0}$, we get the following linear recurrence relation $L(w) = L(w-1) + L(w-d)$
    where $L(w) = 1$ if $w \leq 0$.
    We will now show that $L(w) = \Theta(r^w)$, where
    $r$ is the solution to 
    \begin{align}
        r^{-d} + r^{-1} = 1. \label{eq:r}
    \end{align}
    We start by showing that there exists some constant $C_1$ such that for all $w \in \R_{\geq 0}$ $L(w) \leq C_1 r^w$. We proceed by induction. For the base case, we choose $C_1 \geq \max_{0\leq t \leq d} L(t) r^{-t}$, so the statement holds trivially.
    Assuming that the statement holds for all $t \leq T-1$, we get
    \begin{align*}
        L(T) &= L(T-1) +L(T-d) \\ 
        & \leq C_1 r^{T-1} +C_1 r^{T-d} \\
        &= C_1 r^T(r^{-1} +r^{-d}) \\
        & =C_1 r^T \tag{from \Cref{eq:r}}.
    \end{align*}
    A similar argument shows that there exists $C_2$ such that for all $w \in \R_{\geq 0}$ $L(w) \geq C_2 r^w$. 
    We let $W^* = \sup \{T: L(T) \leq |H|\}$. Using $L(T) = \Theta(r^T)$,
    \begin{align*}
        L(T) \leq |H| \implies T \leq   \frac{\log |H|}{\log r}.
    \end{align*}
    Therefore, $W^* = \frac{\log |H|}{\log r} +O(1)$.
    We will now show that $\frac{1}{\log r} = \Theta(\frac{d}{\log d})$. We have that 
    $r^{-d}+r^{-1} = 1$. For $r= 1/u$, we get $u^d +u=1$. So for large $d$, the solution is of the form $u = 1-\delta$ for small $\delta>0$. We have that $(1-\delta)^d \approx e^{-d\delta}$. So we get $e^{-d\delta} \approx \delta$ or $ye^{y} \approx d$ for $y = d\delta$. Equivalently, $W(d) = y \iff y \approx \ln d -\ln\ln d + o(1)$, where $W$ is the Lambert $W$ function. So we get that $\delta = \frac{\ln d}{d}$ and $u = \frac{1}{r} = 1-\delta \implies \ln r \approx \delta = \frac{\ln d}{d}$. So we get that $W^* = \Theta(\frac{d}{\log d} \log |H|)$, which concludes the proof.
    
    \end{proof}

\additive*
\begin{proof}
All three lower bounds use \emph{concatenation}. Given trees
$\mathcal{T}_1,\dots,\mathcal{T}_L$, where the nodes of $\mathcal{T}_i$
lie in $Z_i$, replace every leaf of $\mathcal{T}_i$ by a copy
of $\mathcal{T}_{i+1}$. If each $\mathcal{T}_i$ is shattered by $H_i$,
the concatenation is shattered by $H$. A root-to-leaf path decomposes
into segments, segment $i$ being a root-to-leaf path of $\mathcal{T}_i$,
and if $h_i\in H_i$ realizes segment $i$ then $(h_1,\dots,h_L)$ realizes
the whole path, since layer-$i$ nodes are labeled by $h_i$ alone.

(i) \emph{Lower bound:} concatenating complete shattered trees of depths
$\Ldim(H_i)$ gives a complete shattered tree of depth
$\sum_i\Ldim(H_i)$. \emph{Upper bound:} let $\mathcal{T}$ be a complete
tree of depth $m$ shattered by $H$. Version spaces along any path remain
products $V=V_1\times\cdots\times V_L$, since a layer-$i$ node restricts
only the $i$-th factor. Define the potential $\Phi(V)=\sum_i\Ldim(V_i)$.
At any internal node $z$ of layer $i$
at least one of the two version spaces $V_i[(z,y)]$ has Littlestone dimension at
most $\Ldim(V_i)-1$ (otherwise, joining two complete shattered trees of
depth $\Ldim(V_i)$ below a root labeled $z$ would contradict the
definition of $\Ldim(V_i)$). Following such an edge at every step, the
potential strictly decreases and within $\Phi(H)$ steps it
reaches $0$. At this point the version space is a single function and such a node must be a leaf. In a complete tree of depth $m$,
every root-to-leaf path has length $m$, so
$m\le\Phi(H)=\sum_i\Ldim(H_i)$.

(ii) For each $i$, let $\mathcal{T}_i$ be a
$(k,\mu_i)$-difficult shattered tree with $\mu_i=\SCLdim(H_i,k)$ (or with
$\mu_i$ arbitrarily large if the latter is infinite), and concatenate. A
root-to-leaf path with at most $k$ straight edges has at most $k$
straight edges \emph{within each segment}. By $(k,\mu_i)$-difficulty,
segment $i$ has length at least $\mu_i$, so the path has length at least
$\sum_i\mu_i$, and the concatenation is
$(k,\sum_i\mu_i)$-difficult.

(iii) The lower bound is (ii) at $k=0$. For the upper bound, let
$\mathcal{T}$ be $(0,m)$-difficult and shattered by $H$, and let
$z_1,\dots,z_m$ be the first $m$ internal nodes of its all-$\yes$ path.
As in the proof of Proposition~\ref{prop:rn}, for each $t$ there is a
function  $h^{(t)}=(h^{(t)}_1,\dots,h^{(t)}_L)\in H$ that realizes the straight
edge at $z_t$, so $h^{(t)}(z_t)=\no$ and $h^{(t)}(z_s)=\yes$ for
$s<t$, and there is a function $h^{(\infty)}$ that realizes the all-$\yes$ path.
Fix a layer $i$ and let $z_{t_1},\dots,z_{t_{m_i}}$ (in path order) be
the layer-$i$ nodes among $z_1,\dots,z_m$. The caterpillar with spine
$z_{t_1},\dots,z_{t_{m_i}}$ is $(0,m_i)$-difficult by construction, and
it is shattered by $H_i$: the path $\yes^{\,j-1}\no$ is realized by
$h^{(t_j)}_i$, which rejects $z_{t_j}$ and accepts all earlier layer-$i$
nodes (these being among the $z_s$ with $s<t_j$), and the all-$\yes$
path is realized by $h^{(\infty)}_i$. Hence $m_i\le\SCLdim(H_i,0)$.
Since every node lies in exactly one layer,
$m=\sum_im_i\le\sum_i\SCLdim(H_i,0)$.
\end{proof}

\wscfixedcard*
\begin{proof}
\emph{Upper bound.} The learner that always outputs $\yes$ makes a soundness mistake
only on the $b$ points of $S$ (cost $\le b\gs$). The learner that always outputs $\no$
 makes a completeness mistake only on the $n-b$ correct points (cost
$\le(n-b)\gc$).

\emph{Lower bound.} If the learner
predicts $\yes$, the adversary can reveal $\no$ (declaring the point to lie in
$S$), which forces a soundness mistake of weight $\gs$, provided fewer than $b$ points
have so far been declared in $S$. If the learner predicts $\no$, the adversary can
reveal $\yes$ (declaring the point outside $S$), a completeness mistake of
weight $\gc$, provided fewer than $n-b$ points have been declared outside.
At any stage with $s'<b$ known points in $S$ and $c'<n-b$ known points outside $S$ both
labels are realizable by some $b$-subset consistent with the history, so the
adversary may always force a mistake. A path terminates only once $s'=b$ or $c'=n-b$. Its weight is then at least $b\gs$ or at least
$(n-b)\gc$, respectively. The minimum-weight root-to-leaf path has weight
$\min(b\gs,(n-b)\gc)$.
\end{proof}

\wscadditive*

\begin{proof}
\emph{Lower bound.} Concatenate optimal WSC-mistake trees $T_1,\dots,T_L$ by
attaching a fresh copy of $T_{i+1}$ at each leaf of $T_i$. Since layer-$i$
nodes are labeled by $h_i$ alone, a verifier $(h_1,\dots,h_L)$ realizes a
root-to-leaf path iff each $h_i$ realizes the path's $i$-th segment. Hence the
concatenation is shattered by $H$. Every root-to-leaf path is a concatenation
of one path per layer, so its weight is the sum of the per-segment weights,
which is at least $\sum_i\WSCLdim(H_i)$.

\emph{Upper bound.} Run an independent copy of the optimal weighted learner
(Algorithm~\ref{alg:SCL-soa}) on each layer,  querying for each query the
learner for the corresponding layer. A realizable sequence for $H$ restricts on each
layer to a sequence realizable by $H_i$. By
Theorem~\ref{thm:wscl-ub} the cumulative cost incurred on layer $i$ is at
most $\WSCLdim(H_i)$. Summing over the disjoint layers gives cumulative cost
at most $\sum_i\WSCLdim(H_i)$.
\end{proof}

\wschbeta*

\begin{proof}
The closed form is Proposition~\ref{prop:wsc-additive} applied to
$\WSCLdim(R_{n,b})=\min(b\gs,(n-b)\gc)$
(Proposition~\ref{prop:wsc-fixed-card}) and
$\WSCLdim(R_{n'})=\min(\gs,(n'-1)\gc)$
(Proposition~\ref{prop:wsc-singleton}). The first summand equals $b\gs$ when
$\lambda\le(n-b)/b$ and $(n-b)\gc$ otherwise. The second equals $\gs$ when
$\lambda\le n'-1$ and $(n'-1)\gc$ otherwise. It remains to check that the two
thresholds are ordered, i.e.\ $(n-b)/b\le n'-1$, so that the three regimes
appear in the stated order and the middle regime is nonempty. From
$n'\ge n-2b+2$ we have $n'-1\ge n-2b+1$, so it suffices that
$(n-b)/b\le n-2b+1$, i.e.\ $n-b\le b(n-2b+1)$, i.e.\
$n(1-b)\le 2b(1-b)$. For $b=1$ both sides vanish. For $b\ge2$ dividing by
$1-b<0$ gives $n\ge2b$, which holds since $n\ge2b+1$. Substituting the
regime values yields the piecewise formula (in the first regime
$b\gs+\gs=(b+1)\gs$, in the third $(n-b)\gc+(n'-1)\gc$).
\end{proof}

\wscsubspace*

\begin{proof}
The boundary case is Lemma~\ref{lem:wsc-boundary}. Assume $\gs,\gc>0$.

\medskip
\noindent\emph{Upper bound: $(d-1)\max(\gs,\gc)$.}
For $r\ge1$ let
\[
F(r)=\sup\big\{\,\WSCLdim(\{h_v:v\in P\})\;:\;
V\le\mathbb R^d,\ \dim V=r,\ \emptyset\ne P\subseteq V\setminus\{0\}\,\big\},
\]
where trees are taken with nodes labeled by examples that trace, on $V$, to
flags of $V$. Scale invariance gives $F(1)=0$ (a one-dimensional version
space induces a single function). Let $T$ be any WSC-mistake tree shattered
by $\{h_v:v\in P\}$ with $\dim V=r\ge2$. By Lemmas~\ref{lem:normal}
and~\ref{lem:trace}, its root traces on $V$ to a nonconstant flag classifier
$\chi_{S,T'}(v)=\mathbf 1[v\in T'\setminus S]$ with $S\subset T'\subseteq V$,
$\dim T'=\dim S+1$. Two cases:
\begin{itemize}
\item If $T'=V$, the straight/$\no$\ edge confines $v$ to $S$, of dimension
$r-1$, so $T$ has a root-to-leaf path of weight at most $\gs+F(r-1)$.
\item If $T'\subsetneq V$, the curvy/$\yes$ edge confines $v$ to
$T'\setminus S\subseteq T'$, of dimension at most $r-1$, so $T$ has a
root-to-leaf path of weight at most $\gc+F(r-1)$.
\end{itemize}
In either case the minimum path weight of $T$ is at most
$\max(\gs,\gc)+F(r-1)$. Taking the supremum over $T$ gives
$F(r)\le F(r-1)+\max(\gs,\gc)$, hence
$\WSCLdim(H_{\mathrm{sub}})\le F(d)\le(d-1)\max(\gs,\gc)$. This is the
weighted analogue of the dimension-drop argument of
Lemma~\ref{lem:ldim-ub}.

\medskip
\noindent\emph{Lower bound when $\gs\ge\gc$.}
Apply Lemma~\ref{lem:obstacle}  with $V=\mathbb R^d$, empty obstacle family,
$k=d-2$ (so $r=d\ge k+2$), and any
$m\ge\lceil(d-1)\gs/\gc\rceil$. This yields a $(d-2,m)$-difficult
SC-mistake tree shattered by $H_{\mathrm{sub}}$, viewed as a WSC-mistake tree.
Consider any root-to-leaf path. If it has at least $d-1$ straight edges, its
weight is at least $(d-1)\gs$. Otherwise it has at most $d-2$ straight edges,
so by $(d-2,m)$-difficulty its length is at least $m$. Since every edge has
weight at least $\gc$ (using $\gs\ge\gc$), its weight is at least
$m\gc\ge(d-1)\gs$. Hence the tree has minimum path weight at least
$(d-1)\gs=(d-1)\max(\gs,\gc)$.

\medskip
\noindent\emph{Lower bound when $\gc\ge\gs$.}
We prove, by induction on $r$, the following dual statement: for every
subspace $V$ with $\dim V=r$ and every finite family $\mathcal A$ of proper
subspaces of $V$, there is a WSC-mistake tree with nodes tracing to flags of
$V$, shattered by $H_{\mathrm{sub}}$, whose minimum path weight is at least
$(r-1)\gc$ and each of whose root-to-leaf paths admits a realizer $h_v$ with
$v\in V\setminus\big(\{0\}\cup\bigcup_{A\in\mathcal A}A\big)$. Taking
$V=\mathbb R^d$ and $\mathcal A=\emptyset$ then gives
$\WSCLdim(H_{\mathrm{sub}})\ge(d-1)\gc=(d-1)\max(\gs,\gc)$.

The base case $r=1$ is a single leaf: the required realizer exists because a
finite union of proper subspaces cannot cover the vector space $V$ over an
infinite field. For $r\ge2$, set $m=\lceil(r-1)\gc/\gs\rceil$ and choose
distinct hyperplanes $W_1,\dots,W_m$ of $V$, none in $\mathcal A$ (possible
as there are infinitely many hyperplanes and finitely many exclusions). For
each $j$ fix a hyperplane $S_j\subset W_j$ and let node $z_j$ trace to the
flag $(S_j,W_j)$, so that $\chi_{S_j,W_j}(v)=\mathbf 1[v\in W_j\setminus S_j]$
and $\dim W_j=r-1$. Arrange $z_1,\dots,z_m$ as a spine of \emph{straight/$\no$}
edges, with a leaf below the straight edge of $z_m$. Below the
\emph{curvy/$\yes$} edge at each $z_j$ (which corresponds to
$v\in W_j\setminus S_j$) attach the inductively constructed tree inside
$W_j$, with augmented obstacle family
\[
\{S_j\}\;\cup\;\{W_i\cap W_j:i<j\}\;\cup\;\{A\cap W_j:A\in\mathcal A\},
\]
each member of which is proper in $W_j$ (in particular $W_i\cap W_j\subsetneq
W_j$ since $W_i\ne W_j$).

\emph{Shattering.} A subtree realizer $v$ for the branch below $z_j$ lies in
$W_j\setminus S_j$, hence takes the curvy edge at $z_j$. It avoids
$W_i\cap W_j$ for $i<j$, and being in $W_j$ this forces $v\notin W_i$, hence
$v\notin W_i\setminus S_i$, so it takes the straight edge at each earlier
$z_i$. Also, it avoids every $A\in\mathcal A$. The all-straight path has a
realizer outside $\bigcup_j W_j$ (and off the obstacles), which is nonempty
for the same union-of-subspaces reason. Ambient realizability of every flag
$(S_j,W_j)$ with $\dim W_j=r-1\le d-1\le d'$ is guaranteed by
Lemma~\ref{lem:normal}, and traces agree by Lemma~\ref{lem:trace}.

\emph{Minimum path weight.} The all-straight (spine) path has weight
$m\gs\ge(r-1)\gc$ by the choice of $m$. A path that first turns curvy at
$z_j$ has weight at least
\[
(j-1)\gs+\gc+(r-2)\gc\;\ge\;\gc+(r-2)\gc\;=\;(r-1)\gc,
\]
using the inductive bound $(r-2)\gc$ on the attached $W_j$-subtree (with
$\dim W_j=r-1$). Hence every root-to-leaf path has weight at least
$(r-1)\gc$, completing the induction.
\end{proof}

\section{Alternate Mistake Model}

In this section, we analyze learnability under an alternate model for the learner's types of mistakes. We are interested in the trade-off between soundness and completeness mistakes (at the sequence level this time) via a linear cost objective. Specifically, the learner incurs a cost of $\gamma_s \in \R_{\geq 0}$ for every sequence-level soundness mistake and cost $\gamma_c \in \R_{\geq 0}$ for every sequence-level completeness mistake. We additionally introduce the notion of a ``location'' mistake, for which the learner incurs cost $\gamma_l \in \R_{\geq 0}$. Formally, we distinguish between the following three types of mistakes:
\begin{enumerate}
    \item \textbf{(Sequence-level) soundness mistake}, which occurs when the learner accepts an incorrect reasoning trace.
    Formally, the learner makes a soundness mistake in round~$t$ if $ \hat{y}^{(t)} = \infty$ but $y^{(t)} \in [L]$.\looseness-1
    \item \textbf{(Sequence-level) completeness mistake},  which occurs when the learner indicates error in a correct reasoning trace. Formally, the learner makes a 
    completeness mistake in round~$t$ if $ \hat{y}^{(t)} \in [L]$ but $y^{(t)} = \infty$.\looseness-1
    \item \textbf{Location mistake}, which occurs when the learner rejects an incorrect reasoning trace, but mispredicts the location of the first faulty step. Formally, the learner makes a location mistake in round~$t$ if $ \hat{y}^{(t)}, y^{(t)} \in [L]$ but $\hat{y}^{(t)} \neq y^{(t)}$.
\end{enumerate}

For convenience, we drop the term ``sequence-level'' in the remainder of this section. The learner's objective is to minimize the linear cumulative cost $\gamma_s M_s + \gamma_c M_c + \gamma_l M_l$, where $M_s, M_c$, and $M_l$ is the number of soundness, completeness, and location mistakes that the learner makes over $T$ timesteps. Here we assume that $\gamma_s \geq \gamma_c \geq \gamma_l$.

As in \Cref{sec:tradingsoundnesscompleteness}, our goal here is to define the right notion of a mistake tree that will capture the optimal strategy of an adversary against any online verification algorithm. However, here the optimal strategy of the adversary depends on the learner's prediction in a way that might not come down to a binary decision. Specifically, if a learner predicts the first faulty step of a reasoning trace to be $\hat{y}^{(t)} \in [L]$, the adversary could induce either a completeness mistake or a location mistake and for the latter the optimal label $y^{(t)} \in [L]$ for the adversary to choose might differ depending on the learner's prediction. Moreover, as is typical in mistake trees, edges are labeled with an element in the label set and in this setting a label in $[L]$ could be used to induce either a location or a completeness mistake, so the tree structure must be expressive enough to distinguish between the two. Another challenge is the real-valued cost for each mistake, which must be accounted for in the tree structure.

We now introduce the notion of an SCL-mistake tree, which has weighted edges of different types designed to handle the complexity of our online verification setting. 

\paragraph{SCL-mistake tree.} An SCL-mistake tree is a rooted binary tree, where every internal node is labeled with an example in $X \times \Sigma^L$ and every edge is labeled in $[L] \cup \{ \infty \}$, such that the outgoing edges from a single node have different labels. There are 3 possible types of edges: an edge labeled with $i \in [L]$ may be either an \emph{s-edge} or an \emph{l-edge}. An edge labeled $\infty$ is always a \emph{c-edge}. For every node, exactly one of the following must hold: (a) the node has two outgoing l-edges, or (b) the node has one outgoing s-edge and one outgoing c-edge.
Each edge $e$ has a corresponding weight $w(e)$, which is equal to $\gamma_s, \gamma_c, \gamma_l$ if $e$ is an s-, c-, and l-edge, respectively. 

We say that an SCL mistake tree is \emph{shattered} by a verifier class $H$ if for any root-to-leaf path that traverses internal nodes $z_1, z_2, \dots, z_d$ there exists a verifier $h \in H$ such that  the label of edge $(z_i, z_{i+1})$ is $h^{\text{Ver-CoT}}(z_i)$ for all $i \leq d$. 

We highlight that the SCL mistake tree is a binary tree, despite the adversary potentially needing more than two labels to play optimally (which would therefore require more than two children per node). We find that this subtlety is resolved when $\gamma_s \geq \gamma_c \geq \gamma_l$, which is a mild assumption in our online verification setting.

\begin{definition}[SCL-Littlestone dimension]
    For a class of verifiers, $H$, the SCL-Littlestone dimension, $\text{SCL-Ldim}(H)$, is defined as the supremum $w \in \R$ such that there exists an SCL mistake tree, where all root-to-leaf paths have cumulative weight at least $w$.
\end{definition}
If $H = \emptyset$, then $\text{SCL-Ldim}(H) = 0$. If for all $w \in \R_{\geq 0}$ there exists an SCL-mistake tree shattered by $H$, where all root-to-leaf paths have weight at least $w$, then we say that $\text{SCL-Ldim}(H) = \infty$. Next we show that any deterministic verification algorithm using $H$ must suffer cumulative loss at least $\text{SCL-Ldim}(H)$. 

\begin{theorem}\label{thm:oldscllb}
For any verifier class, $\cH$, and any deterministic learning algorithm, there exists a realizable sequence such that the learner must suffer cumulative loss at least $W$ for every $W < \text{SCL-Ldim}(H)$. 
\end{theorem}

The proof of \Cref{thm:oldscllb} follows similarly to that of \Cref{thm:scllb}, so we omit it.

 We next analyze the performance of \Cref{alg:oldSCL-soa}. Her, we denote the immediate loss for a learner's prediction $\hat{y}$ and a true prediction $y$ as $\ell(\hat{y}, y) \in \{0, \gamma_s, \gamma_c, \gamma_l\}$.

\begin{algorithm}[t]
\caption{SCL-SOA}\label{alg:oldSCL-soa}
\DontPrintSemicolon
\LinesNumbered

\KwIn{class of verifiers $H$, mistake costs $\gamma_s, \gamma_l, \gamma_c \in \mathbb{R}_{\ge 0}$.}

$H^{(1)} \leftarrow H$\tcp*[r]{set initial version space}

\For{$t \in [T]$}{
    Observe example $z^{(t)} = \bigl(x^{(t)}, \tau^{(t)} = (\tau_1^{(t)}, \ldots, \tau_L^{(t)})\bigr)$\;

    Compute $\cY^{(t)} = \{i \in [L] \cup \{ \infty \}\; | \; \exists h\in H^{(t)} \text{ such that } h^{\text{Ver}-L}(z^{(t)}) = i \}$\tcp*[r]{All possible first faulty steps}
    
    \eIf(\tcp*[f]{all consistent verifiers agree}){$|\cY^{(t)}|=1$}{
        Predict $\hat{y}^{(t)} \in \cY^{(t)}$\;
    }{
        \For{\emph{each} $i \in \cY^{(t)}$}{
            Compute
            $m_i = \max_{j \in \cY^{(t)}} \Bigl[\ell(i,j) + \mathrm{SCL\text{-}Ldim}\bigl(H^{(t)}[(z^{(t)}, j)]\bigr)\Bigr]$\;
        }

        Predict $\hat{y}^{(t)} = \arg\min_{i \in \cY^{(t)}} m_i$\;
    }

    Receive correct label $y^{(t)} \in \cY$\;
    
    Set $H^{(t+1)} \leftarrow H^{(t)}[(z^{(t)}, y^{(t)})]$\tcp*[r]{Update version space}
}
\end{algorithm}

\begin{theorem}\label{thm:oldsclsoaub}
    For any class of verifiers, $H$, SCL-SOA (\Cref{alg:SCL-soa}) achieves cumulative loss at most $\text{SCL-Ldim}(H)$ on any realizable sequence of examples.
\end{theorem}
\begin{proof}
We will show that for every timestep $t$, $\ell(\hat{y}^{(t)}, y^{(t)}) \leq \text{SCL-Ldim}(H^{(t)}) - \text{SCL-Ldim}(H^{(t+1)})$.
We consider the prediction $\hat{y}^{(t)}$ of \Cref{alg:SCL-soa} in timestep $t$ and get 
    \begin{align*}
        \ell(\hat{y}^{(t)}, y^{(t)}) + \text{SCL-Ldim}(H^{(t+1)}) &\leq \max_{j \in \cY^{(t)}} \left[ \ell(\hat{y}^{(t)}, j) + \text{SCL-Ldim}(H^{(t)}[(z^{(t)}, j)]) \right] \\
        & = \min_{i \in \cY^{(t)}} \max_{j \in \cY^{(t)}} \left[ \ell(i, j) + \text{SCL-Ldim}(H^{(t)}[(z^{(t)}, j)]) \right]. 
    \end{align*}
    From the definition of the SCL-Littlestone dimension, we also  know that
    \begin{align*}
        \text{SCL-Ldim}(H^{(t)})  \geq \max_{\{k_1, k_2\} \subseteq \cY^{(t)}} \min_{i \in \{k_1, k_2\} }\left[ \gamma_i + \text{SCL-Ldim}(H^{(t)}[(z^{(t)}, i)]) \right]
    \end{align*}
    where $k_1, k_2$ are the edges of the root node, $x^{(t)}$, in an SCL mistake tree. So it suffices to show that
    \begin{align*}
        \min_{i \in \cY^{(t)}} \max_{j \in \cY^{(t)}} \left[ \ell(i, j) + \text{SCL-Ldim}(H^{(t)}[(z^{(t)}, j)]) \right] \leq \max_{\{k_1, k_2\} \subseteq \cY^{(t)}} \min_{i \in \{k_1, k_2\} }\left[ \gamma_i + \text{SCL-Ldim}(H^{(t)}[(z^{(t)}, i)]) \right]
    \end{align*}
    First, we observe that if $|\cY^{(t)}| \leq 1$ we are done.
    Otherwise we consider the following three cases.
    \paragraph{Case 1:} 
    \begin{align*}
        \gamma_c + \text{SCL-Ldim}(H^{(t)}[(z^{(t)}, \infty)]) \geq \max_{i \in \cY^{(t)}} \left[ \gamma_l + \text{SCL-Ldim}(H^{(t)}[(z^{(t)}, i)]) \right]
    \end{align*}
    where we will denote with $i^*$ the $\argmax$ of the RHS. Given that $\gamma_s \geq \gamma_l$, it follows that $x^{(t)}$ must have one outgoing s-edge labeled $i^*$ and one outgoing c-edge labeled $\infty$, where
    \begin{align}
        & \max_{\{k_1, k_2\} \subseteq \cY^{(t)}} \min_{i \in \{k_1, k_2\} }\left[ \gamma_i + \text{SCL-Ldim}(H^{(t)}[(z^{(t)}, i)]) \right] \nonumber \\
        = & \min \left\{  \gamma_s + \text{SCL-Ldim}(H^{(t)}[(z^{(t)}, i^*)]), \gamma_c + \text{SCL-Ldim}(H^{(t)}[(z^{(t)}, \infty)])\right\} \label{eq:case1}
    \end{align}
    
If $\hat{i} \in [L]$, then $ \max_{j \in \cY^{(t)}} \left[ \ell(\hat{i}, j) + \text{SCL-Ldim}(H^{(t)}[(z^{(t)}, j)])\right]  =  \gamma_c + \text{SCL-Ldim}(H^{(t)}[(z^{(t)}, \infty)]) $ which is at most equal to (\ref{eq:case1}), or $\hat{i}$ is suboptimal.
    
    If $\hat{i} = \infty$, then
    \begin{align*}
        &\max_{j \in \cY^{(t)}} \left[ \ell(\hat{i}, j) + \text{SCL-Ldim}(H^{(t)}[(z^{(t)}, j)])\right] \\
        =  & \min \left\{ \text{SCL-Ldim}(H^{(t)}[(z^{(t)}, \infty)]) , \gamma_s +\text{SCL-Ldim}(H^{(t)}[(z^{(t)}, i^*)])  \right\}
    \end{align*}
    which again is at most equal to (\ref{eq:case1}).

    \paragraph{Case 2:}
    \begin{align*}
        \gamma_c + \text{SCL-Ldim}(H^{(t)}[(z^{(t)}, \infty)]) \leq \max_{i \in \cY^{(t)}\setminus \{\max \cY^{(t)}\}} \left[ \gamma_l + \text{SCL-Ldim}(H^{(t)}[(z^{(t)}, i)]) \right]
    \end{align*}
    where we will denote the maximizing and the second maximizing element of the RHS $i_1^*, i_2^* \in [L]$, respectively. Then $x^{(t)}$ must have two distinct outgoing l-edges $i_1^*$ and $i_2^*$, such that
    \begin{align}
        \max_{\{k_1, k_2\} \subseteq \cY^{(t)}} \min_{i \in \{k_1, k_2\} }\left[ \gamma_i + \text{SCL-Ldim}(H^{(t)}[(z^{(t)}, i)]) \right]
        =  \gamma_l + \text{SCL-Ldim}(H^{(t)}[(z^{(t)}, i_2^*)]) \label{eq:case2}
    \end{align}
    
     If $\hat{i} =i_1^*$, then $ \max_{j \in \cY^{(t)}} \left[ \ell(\hat{i}, j) + \text{SCL-Ldim}(H^{(t)}[(z^{(t)}, j)])\right]  = \gamma_l + \text{SCL-Ldim}(H^{(t)}[(z^{(t)}, i_2^*)])$, which is exactly (\ref{eq:case2}).
    
    If $\hat{i} \in [L], $but $ \hat{i} \neq i_1^*$, then \[ \max_{j \in \cY^{(t)}} \left[ \ell(\hat{i}, j) + \text{SCL-Ldim}(H^{(t)}[(z^{(t)}, j)])\right]  = \gamma_l \allowbreak+ \text{SCL-Ldim}(H^{(t)}[(z^{(t)}, i_1^*)]),\] which means that $\hat{i}$ is suboptimal.
    
    If $\hat{i} = \infty$, then $ \max_{j \in \cY^{(t)}} \left[ \ell(\hat{i}, j) + \text{SCL-Ldim}(H^{(t)}[(z^{(t)}, j)])\right]  = \gamma_s + \text{SCL-Ldim}(H^{(t)}[(z^{(t)}, i_1^*)])$, which means that $\hat{i}$ is suboptimal.

    \paragraph{Case 3:}
    \begin{align*}
         \max_{i \in \cY^{(t)}\setminus \{\max \cY^{(t)}\}} \left[ \gamma_l + \text{SCL-Ldim}(H^{(t)}[(z^{(t)}, i)]) \right] &\leq \gamma_c + \text{SCL-Ldim}(H^{(t)}[(z^{(t)}, \infty)]) \\
         & \leq \max_{i \in \cY^{(t)}} \left[ \gamma_l + \text{SCL-Ldim}(H^{(t)}[(z^{(t)}, i)]) \right]
    \end{align*}
    where we will denote the maximizing and the second maximizing element of the RHS $i_1^*, i_2^* \in [L]$, respectively. Then $x^{(t)}$ must have one outgoing s-edge $i_1^*$ and one outgoing c-edge $\infty$, where
    \begin{align}
         \max_{\{k_1, k_2\} \subseteq \cY^{(t)}} \min_{i \in \{k_1, k_2\} }\left[ \gamma_i + \text{SCL-Ldim}(H^{(t)}[(z^{(t)}, i)]) \right] 
        =  \gamma_c + \text{SCL-Ldim}(H^{(t)}[(z^{(t)}, \infty)]) \label{eq:case3}
    \end{align}

    If $\hat{i} =i_1^*$, then $ \max_{j \in \cY^{(t)}} \left[ \ell(\hat{i}, j) + \text{SCL-Ldim}(H^{(t)}[(z^{(t)}, j)])\right]  = \gamma_c + \text{SCL-Ldim}(H^{(t)}[(z^{(t)}, \infty)])$, which is equal to (\ref{eq:case3}).
    
    If $\hat{i} \in [L]$ but $ \hat{i} \neq i_1^*$, then \[ \max_{j \in \cY^{(t)}} \left[ \ell(\hat{i}, j) + \text{SCL-Ldim}(H^{(t)}[(z^{(t)}, j)])\right] \\  = \gamma_l + \text{SCL-Ldim}(H^{(t)}[(z^{(t)}, i_1^*)]),\] which means that $\hat{i}$ is suboptimal.
    
    If $\hat{i} = \infty$, then $ \max_{j \in \cY^{(t)}} \left[ \ell(\hat{i}, j) + \text{SCL-Ldim}(H^{(t)}[(z^{(t)}, j)])\right]  = \gamma_s + \text{SCL-Ldim}(H^{(t)}[(z^{(t)}, i_1^*)])$, which means that $\hat{i}$ is suboptimal since $\gamma_s \geq \gamma_l$.
\end{proof}

We note that when $\gamma_s=\gamma_c=\gamma_l = 1$, the SCL-Littlestone dimension recovers the definition of the Littlestone dimension. We now present an example that shows an $O(L)$ separation between the setting, where $\gamma_l=0$ and $\gamma_l=1$. This highlights the added difficulty of learning verifiers that correctly predict the location of the first faulty step for a given reasoning trace.

\begin{example}
    Consider an example where correct reasoning traces correspond to satisfying assignments of an unknown conjunction of size $L$. Let $x_1, x_2, 
    \dots, x_L$ be Boolean variables, and let $H = \Big\{ h_S : \{0,1\}^L \to \{0,1\} \;\Big|\;  h_S(x) = \bigwedge_{i \in S} x_i \;\wedge\; \bigwedge_{j \in T} \neg x_j,\; S,T \subseteq [L],\ |S|+|T| = L \Big\}$
    be the class of verifiers. Let $h_{S^*}$ be the unknown target conjunction. A reasoning trace $\tau \in \{0, 1\}^L$ is a Boolean assignment and is correct if and only if $h_{S^*}(\tau) = 1$. Let $\gamma_s = \gamma_c =1$. If $\gamma_l = 0$, then the learner can learn $h_{S^*}$ with at most 1 completeness mistake by rejecting every reasoning trace. A single correct proof will indicate exactly which literals are negated. On the other hand, if $ \gamma_l >0$, the learner can be forced to incur loss $ \gamma_l \cdot L/2$. To see this, notice that an adversary can ask the learner to verify an assignment at timestep $t$ and reveal the true incorrect step to be either step $2t-1$ or $2t$, depending on which step the learner did not flag as incorrect. Thus, the adversary can guarantee a location mistake for at least $L/2$ timesteps.
\end{example}

\section{Sequential Harnessing: Connections with Imitation Learning and Online Chain-of-Thought Generation}\label{app:online-template}

Many of the results in this paper involve reductions between different kinds of learning problems.
In this section, we provide a generic sequential-learning reduction framework for reductions of this type. While this framework is not able to capture the asymmetry and trade-off
between soundness and completeness mistakes, we show it can capture several
important interactive learning settings, including imitation learning,
chain-of-thought verification,
verifier-guided reasoning, and
online chain-of-thought generation. 
First, we place our framework in a broader sequential-learning context by identifying 
 a common
``harnessing”\footnote{The term ``harnessing" or ``scaffolding"~\citep{suzgun2024meta} has been recently coined to describe complementary models or tools that augment and orchestrate the capabilities of LLMs.} argument that appears implicitly in several online reductions, especially in
imitation learning.
Then we show how this argument directly yields baseline online learnability guarantees for reasoning and verification systems.
It also yields a clean online chain-of-thought generation result, while eliminating the $O(\log L)$ chain length factor from prior work~\citep{joshi2025theory}.

\paragraph{Notation.} We first recall the usual online learning terminology.
An online prediction problem 
$\mathsf P$ is defined over an instance (input) space $I$, a label (output) space $Y$, and predictors 
$H\subseteq Y^I$.
A labeled sequence $G_T=((i^{(1)},y^{(1)}),\ldots,(i^{(T)},y^{(T)}))\in \Gamma := (I\times Y)^*$ is \emph{realizable} for $\mathsf P$ if there is a function $h^\star\in H$ such that $y^{(t)}=h^\star(i^{(t)})$ for every $t$. A (deterministic) online learner for $\mathsf P$ is a map from labeled histories to predictors $A:\Gamma\rightarrow Y^I$.
The predictor $A(G_T)$ need not be in $H$. On round $t$, after seeing the history
$G_{t-1}:=\bigl((i^{(1)},y^{(1)}),\ldots,(i^{(t-1)},y^{(t-1)})\bigr)$,
the learner sets $\hat h_t:=A(G_{t-1})$, predicts
$\hat y^{(t)}=\hat h_t(i^{(t)})$, 
and then observes $y^{(t)}$. Its mistake count on the sequence is $\sum_t \mathbf{1}\{\hat y^{(t)}\neq y^{(t)}\}$. We write $\operatorname{MB}(A,\mathsf P)$ for the smallest $M$ such that $A$ makes at most $M$ mistakes on every realizable sequence for $\mathsf P$, and
$\operatorname{MB}(\mathsf P):=\inf_A \operatorname{MB}(A,\mathsf P)$
for the optimal realizable mistake bound of the problem.

We now define some new definitions for the sequential harnessing framework. For the remainder of this section, fix two  online learning problems:
$\mathsf P_{\mathrm{base}}$ over $(I_{\mathrm{base}},Y_{\mathrm{base}},H_{\mathrm{base}}),$ and $ 
\mathsf P_{\mathrm{sys}}$ over $(I_{\mathrm{sys}},Y_{\mathrm{sys}},H_{\mathrm{sys}})$. The base problem would correspond to the task of training an  LLM in the context of LLM harnessing, although we will see that the abstraction is useful in a variety of settings.

\begin{definition}[Harness]
Fix a {\it base} instance-label space pair $(I_{\mathrm{base}},Y_{\mathrm{base}})$ and a {\it system} instance-label space pair $(I_{\mathrm{sys}},Y_{\mathrm{sys}})$.  Let
$\Gamma_{\mathrm{base}}:=(I_{\mathrm{base}}\times Y_{\mathrm{base}})^\ast$
be the set of finite labeled base \emph{transcripts}. A \emph{harness from $(I_{\mathrm{base}},Y_{\mathrm{base}})$ to $(I_{\mathrm{sys}},Y_{\mathrm{sys}})$} is a deterministic map
\[
S:I_{\mathrm{sys}}\times \Gamma_{\mathrm{base}}
\to
I_{\mathrm{base}}\sqcup Y_{\mathrm{sys}},
\]
where the disjoint union denotes two types of harness actions, a base query or a  final system output.
As shown below, this captures a procedure that turns any base predictor $h:I_{\mathrm{base}}\to Y_{\mathrm{base}}$ into a system predictor by adaptively querying base examples and then returning one system output.
\end{definition}

In a typical usage, the input to the harness corresponding to the base transcript is generated autoregressively. The harness induces a map from base to system predictors 
$S[\cdot]:Y_{\mathrm{base}}^{I_{\mathrm{base}}}
\to
Y_{\mathrm{sys}}^{I_{\mathrm{sys}}}$,
as follows.
For $h:I_{\mathrm{base}}\to Y_{\mathrm{base}}$ and $i\in I_{\mathrm{sys}}$, initialize the base  transcript $\gamma_0=\emptyset$. Let $m$ denote the first value of $t$ for which $S$ outputs a system label for the pair $h,i$ in the following recursive procedure. If $S(i,\gamma_{t-1})=q_t\in I_{\mathrm{base}}$, set $\gamma_t=(\gamma_{t-1},(q_t,h(q_t)))$; if $S(i,\gamma_m)=y\in Y_{\mathrm{sys}}$, halt and set the induced harness mapping $S[h](i)=y$ and  $\operatorname{Transcript}_{S}^{h}(i)=\gamma_m$,  the base transcript generated by $S$. We assume this recursion  halts after finitely many base queries for every $h$ and $i$.  
 We will also need the notion of a decoder. 

 \begin{definition}[Decoder and harness-decoder consistency]
     A \emph{decoder} is a function $D:I_{\mathrm{sys}}\times Y_{\mathrm{sys}}\to\Gamma_{\mathrm{base}}$ that outputs a base transcript given a system input and output. A harness-decoder pair $(S, D)$ is \emph{consistent}, if for every system predictor there exists a base predictor for which the harness $S$ matches its behavior and the decoder recovers its trace. That is, for every $g\in H_{\mathrm{sys}}$ there exists $h_g\in H_{\mathrm{base}}$, such that for every $i\in I_{\mathrm{sys}}$,
$g(i)=S[h_g](i)
\text{ and }
D(i,g(i))=\operatorname{Transcript}_{S}^{h_g}(i)$.
 \end{definition}

 We show that given a consistent harness-decoder pair, it is possible to construct a system learner that preserves the mistake bounds of the base predictor (Algorithm~\ref{alg:trace-decodable-reduction}). This result is based on a simple inductive argument over the base transcript. The idea appears implicitly in prior work on online-to-online reductions related to imitation learning~\citep{ross2011reduction}. The reduction implies that in order to obtain good system performance, it is sufficient to construct a consistent harness-decoder pair for the system and  base problems.

\begin{theorem}[Online-to-online reduction]\label{thm:online-to-online-reduction}
Let $\mathsf P_{\mathrm{base}}$ and $\mathsf P_{\mathrm{sys}}$ be  two   online learning problems as defined above.
Suppose there is a harness $S$ consistent with a decoder $D$.
Let $R_{S,D}$ be the online-to-online reduction defined in \Cref{alg:trace-decodable-reduction}.
Then, for every base learner $A_{\mathrm{base}}:\Gamma_{\mathrm{base}}\to Y_{\mathrm{base}}^{I_{\mathrm{base}}}$,
\[
\operatorname{MB}(R_{S,D}(A_{\mathrm{base}}),\mathsf P_{\mathrm{sys}})
\le
\operatorname{MB}(A_{\mathrm{base}},\mathsf P_{\mathrm{base}}).
\]
In particular, $\operatorname{MB}(\mathsf P_{\mathrm{sys}})\le \operatorname{MB}(\mathsf P_{\mathrm{base}})$.
\end{theorem}

\begin{algorithm}[h]
\caption{Online-to-online reduction $R_{S,D}$}
\label{alg:trace-decodable-reduction}
\DontPrintSemicolon
\LinesNumbered
\SetNoFillComment
\SetKwComment{tcp}{\(\triangleright\) }{}
\SetKwProg{Fn}{Function}{:}{}
\KwIn{base learner $A_{\mathrm{base}}:\Gamma_{\mathrm{base}}\to Y_{\mathrm{base}}^{I_{\mathrm{base}}}$; harness $S:I_{\mathrm{sys}}\times\Gamma_{\mathrm{base}}\to I_{\mathrm{base}}\sqcup Y_{\mathrm{sys}}$; decoder $D:I_{\mathrm{sys}}\times Y_{\mathrm{sys}}\to\Gamma_{\mathrm{base}}$}
\KwOut{system learner $A_{\mathrm{sys}}:\Gamma_{\mathrm{sys}}\to Y_{\mathrm{sys}}^{I_{\mathrm{sys}}}$}
\Fn{\(A_{\mathrm{sys}}(G=((i_1,y_1),\ldots,(i_m,y_m)))\)}{
    Initialize base transcript $B_0:=\emptyset$\;
    \For{$t=1$ \KwTo $m$}{
        Define the current base predictor $\hat h_t:=A_{\mathrm{base}}(B_{t-1})$\;
        Let $\hat T_t:=\operatorname{Transcript}_{S}^{\hat h_t}(i_t)$ and $\hat y_t:=S[\hat h_t](i_t)$ (transcript and output from the harness)\;
        \If{$\hat y_t\neq y_t$}{
            Let $T_t:=D(i_t,y_t)$ be the decoded base transcript\;
            Let $(z_t,a_t)$ be the first base query-label pair where $T_t$ disagrees with $\hat T_t$, with $a_t$ taken from $T_t$\footnotemark\;
            Define $B_t:=(B_{t-1},(z_t,a_t))$\;
        }
        \Else{
            Define $B_t:=B_{t-1}$\;
        }
    }
    Define the final base predictor $\hat h_{m+1}:=A_{\mathrm{base}}(B_m)$\;
    \Return \(S[\hat h_{m+1}]\)\;
}
\BlankLine
\Return \(A_{\mathrm{sys}}\)\;
\end{algorithm}
\footnotetext{For realizable system histories covered by \Cref{thm:online-to-online-reduction}, the proof of \Cref{thm:online-to-online-reduction} shows that this disagreement exists whenever $\hat y_t\neq y_t$.}

\begin{proof}[Proof of \Cref{thm:online-to-online-reduction}]
Let $A_{\mathrm{sys}}:=R_{S,D}(A_{\mathrm{base}})$. Fix a realizable system input-label sequence 
\[
G_T^{\mathrm{sys}}=((i^{(1)},y^{(1)}),\ldots,(i^{(T)},y^{(T)})),
\]
and, by system realizability, choose  a witness $g^\star\in H_{\mathrm{sys}}$  such that $g^\star(i^{(t)}) = y^{(t)}$ for all $t$. By the harness-decoder consistency assumption, there exists $h^\star\in H_{\mathrm{base}}$ such that, for every $i\in I_{\mathrm{sys}}$,
\[
g^\star(i)=S[h^\star](i),
\qquad
D(i,g^\star(i))=\operatorname{Transcript}_{S}^{h^\star}(i).
\]
The first identity is used to identify the true system label with the harness output
\[
y^{(t)}=g^\star(i^{(t)})=S[h^\star](i^{(t)}).
\]
Let $B_t$ be the base transcript history  produced by \Cref{alg:trace-decodable-reduction} after processing the first $t$ labeled system examples. We prove by induction that $B_t$ is realizable by $h^\star$, that $|B_t|$ is the number of system mistakes in the first $t$ rounds, and that $A_{\mathrm{base}}$ makes a mistake on every example of $B_t$ in order.

The claim is trivial for $B_0=\emptyset$. Suppose it holds through round $t-1$, and let
\[
\hat h_t:=A_{\mathrm{base}}(B_{t-1}),\qquad
\hat y^{(t)}:=S[\hat h_t](i^{(t)}).
\]
If $\hat y^{(t)}=y^{(t)}$, Algorithm~\ref{alg:trace-decodable-reduction} sets $B_t=B_{t-1}$, so the induction claim is unchanged.

Now suppose $\hat y^{(t)}\neq y^{(t)}$. The decoded transcript is
\[
T^\circ
:=
D(i^{(t)},y^{(t)})
=
D(i^{(t)},g^\star(i^{(t)}))
=
\operatorname{Transcript}_{S}^{h^\star}(i^{(t)}).
\]
Here we used the second decoder-consistency identity.  $S[\hat h_t](i^{(t)})$ and $S[h^\star](i^{(t)})$ have different final outputs. $S$ is a deterministic map of the current system input and labeled base transcript, so two executions with the same transcript history take the same next action and cannot halt with different final outputs. Hence there is a first base-label disagreement between $\operatorname{Transcript}_{S}^{\hat h_t}(i^{(t)})$ and $T^\circ$. Write this disagreement as $(z,h^\star(z))$  on the base transcript. Then
\[
\hat h_t(z)\neq h^\star(z),
\]
so the appended example $(z,h^\star(z))$ is both realizable by the base-class witness $h^\star\in H_{\mathrm{base}}$ and a genuine mistake of $A_{\mathrm{base}}$ on the current base history $B_{t-1}$. \Cref{alg:trace-decodable-reduction} appends exactly this one example, hence $|B_t|=|B_{t-1}|+1$ exactly when the system learner makes a mistake.

Thus the final base sequence $B_T$ is realizable and $A_{\mathrm{base}}$ makes exactly $|B_T|$ mistakes on it, where $|B_T|$ is the number of system mistakes. Taking the worst case over realizable system sequences gives
\[
\operatorname{MB}(R_{S,D}(A_{\mathrm{base}}),\mathsf P_{\mathrm{sys}})
\le
\operatorname{MB}(A_{\mathrm{base}},\mathsf P_{\mathrm{base}}).
\]
Taking the infimum over $A_{\mathrm{base}}$ gives $\operatorname{MB}(\mathsf P_{\mathrm{sys}})\le \operatorname{MB}(\mathsf P_{\mathrm{base}})$.
\end{proof}

Note that this template  tracks only a single total mistake objective in contrast with the rest of the paper, where the main technical challenge is to separately track soundness and completeness mistakes. 
We now illustrate the versatility of this framework by applying it to previously studied problems as well as other settings in this work.

\subsection{Imitation Learning as Sequential Harnessing}
\label{subsec:imitation-harnessing}

{Our sequential harnessing framework is closely related to
classical imitation learning reductions such as \cite{ross2010efficient} and refinements like DAgger
\cite{ross2011reduction}. These works show how sequential decision-making problems can be reduced to online learning by training on states induced by the learner's own policy. Subsequent works have studied stochastic experts and policies, including minimax rates when the expert follows an arbitrary stochastic policy~\citep{rajaraman2020toward}, practical approaches that learn stochastic policies~\citep{ho2016generative}, and variance-dependent sample complexity analysis for stochastic policies~\citep{foster2024behavior}}. Intuitively, the base predictor corresponds to
a local policy that predicts the next action from the current state,
while the system predictor corresponds to the induced trajectory-level
policy.

\paragraph{Imitation learning setup.}
Let  $X$ denote a set of contexts/tasks,
 $\mathcal S$ denote a state space,
 $\mathcal A$ denote an action space,
 $L$ denote the horizon length, and
 $U : X \times \mathcal A^{<L} \to \mathcal S$
    denote the deterministic state-transition map.
Given a context $x \in X$ and an action sequence
$a_{1:\ell-1} \in \mathcal A^{\ell-1}$,
the induced state at time $\ell$ is
$s_\ell = U(x,a_{1:\ell-1})$. We use a deterministic abstraction for the reduction, and the randomness in tasks or state transitions can be encoded within the context $x$. Also, unlike the standard Markov formulation, we allow the next state to depend on the entire action history instead of just the current state and the last action.

Let $\Pi \subseteq \mathcal A^{\mathcal S}$ denote a policy class.
Each policy $\pi \in \Pi$ induces a trajectory predictor
$g_\pi : X \to \mathcal A^L$
defined  by
\[
a_\ell = \pi(s_\ell),
\qquad
s_\ell = U(x,a_{1:\ell-1}),
\qquad g_\pi(x) = (a_1,\ldots,a_L).
\]

\paragraph{Base and system prediction problems.}
Define the base prediction problem by
$I_{\mathrm{base}} = \mathcal S,
Y_{\mathrm{base}} = \mathcal A,
H_{\mathrm{base}} = \Pi$.
Thus, the base learner predicts a local action from a state.
Define the system prediction problem by
$I_{\mathrm{sys}} = X,
Y_{\mathrm{sys}} = \mathcal A^L$,
and
$H_{\mathrm{sys}}
=
\{g_\pi : \pi \in \Pi\}$.
Thus, the system learner predicts an entire trajectory.

\paragraph{Harness construction.}

Let
$\Gamma_{\mathrm{base}}
=
(\mathcal S \times \mathcal A)^*$.
For
$x \in X,
\gamma=((s_1,a_1),\ldots,(s_m,a_m))
\in \Gamma_{\mathrm{base}}$,
define the harness
$S_{\mathrm{IL}}
:
X \times \Gamma_{\mathrm{base}}
\to
I_{\mathrm{base}} \sqcup Y_{\mathrm{sys}}$
by
\[
S_{\mathrm{IL}}(x,\gamma)
=
\begin{cases}
U(x,a_{1:m}),
&
m<L,
\\[4mm]
(a_1,\ldots,a_L),
&
m=L.
\end{cases}
\]
That is, the harness repeatedly queries the current state,
receives an action prediction from the base learner,
and after $L$ steps outputs the induced trajectory.
For any policy $\pi \in \Pi$,
the induced system predictor satisfies
\[
S_{\mathrm{IL}}[\pi](x)
=
g_\pi(x).
\]

\paragraph{Decoder construction.}

Given a trajectory label
$y=(a_1,\ldots,a_L)\in\mathcal A^L$,
define the decoder
$D_{\mathrm{IL}}
:
X \times \mathcal A^L
\to
(\mathcal S \times \mathcal A)^L$
by
$D_{\mathrm{IL}}(x,y)
=
\bigl(
(s_1,a_1),
\ldots,
(s_L,a_L)
\bigr)$,
where
$s_\ell
=
U(x,a_{1:\ell-1})$.
Thus, the decoder converts an expert trajectory into the corresponding
sequence of state-action supervision examples.

\paragraph{Consistency of the harness-decoder pair.} It is straightforward to verify the consistency condition of the sequential harnessing
framework.
Fix any
$g \in H_{\mathrm{sys}}$.
By definition of $H_{\mathrm{sys}}$, there exists a policy
$\pi \in \Pi$ such that
$g=g_\pi$.
Let
$h_g := \pi$.
Then for every $x \in X$,
\[
S_{\mathrm{IL}}[h_g](x)
=
S_{\mathrm{IL}}[\pi](x)
=
g_\pi(x)
=
g(x).
\]
Moreover, if
$g(x)=(a_1,\ldots,a_L)$,
then the rollout trace generated by the harness equals
\[
\operatorname{Transcript}^{\pi}_{S_{\mathrm{IL}}}(x)
=
\bigl(
(U(x,\emptyset),a_1),
(U(x,a_1),a_2),
\ldots,
(U(x,a_{1:L-1}),a_L)
\bigr),
\]
which is exactly
$D_{\mathrm{IL}}(x,g(x))$.
Hence,
$(S_{\mathrm{IL}},D_{\mathrm{IL}})$
is a consistent harness-decoder pair.

Applying the sequential harnessing theorem now yields an online-to-online
reduction in the imitation learning setting, and the argument appears implicitly in prior work~\citep{ross2010efficient}. 
This imitation-learning perspective also explains why the sequential-harnessing abstraction
does not separately track soundness and completeness mistakes. When training only from expert
trajectories, a deviation of the learner from the expert trajectory  changes the future state distribution; analogously, if a
verifier makes a completeness mistake and the interaction halts at that prefix,
then later prefixes are never observed. Hence an early false rejection can mask later false
acceptances that would have occurred on incorrect continuations of the trace. The harnessing
reduction therefore controls only the induced trajectory-level mistake, not the separate numbers
of soundness and completeness mistakes. The latter distinction requires the more refined
soundness-completeness analysis developed in the main verifier-learning sections.

\subsection{Prefix and Full CoT Verifier Equivalence as Sequential Harnessing}

We now show that the two online-to-online reductions in Section~\ref{sec:reduction} can
be viewed as special cases of the sequential harnessing framework. In
this subsection we ignore the distinction between soundness and
completeness mistakes and instead use the standard $0$--$1$ mistake
model, where a learner makes a mistake whenever its predicted label
differs from the true label.
{
There is one technical subtlety: in prefix verification the verifier is only required to label a prefix when all
previous steps are correct. We formalize this using partial concept
classes, following the framework of \citet{alon2021partialconcepts}.
A partial online prediction problem is a triple
$\mathsf P=(I,Y,H)$ with $H\subseteq (Y\cup\{\star\})^I$, where
$\star$ denotes an undefined label. A labeled sequence
$((i^{(1)},y^{(1)}),\ldots,(i^{(T)},y^{(T)}))\in(I\times Y)^*$
is realizable if there exists $h^\star\in H$ such that
$h^\star(i^{(t)})=y^{(t)}\neq\star$ for every $t$. A learner
outputs total predictors in $Y^I$, and
$\operatorname{MB}(A,\mathsf P)$ is the worst-case number of mistakes
over realizable labeled sequences in this partial sense.

\paragraph{The two online problems.}

Let the prefix-verification problem be
$P_{\mathrm{pfx}} =
(I_{\mathrm{pfx}},Y_{\mathrm{pfx}},H_{\mathrm{pfx}})$,
where $I_{\mathrm{pfx}} = X \times \Sigma^{\le L}$,
$Y_{\mathrm{pfx}} = \{\yes,\no\}$, and
$H_{\mathrm{pfx}} = H\subseteq (Y_{\mathrm{pfx}}\cup\{\star\})^{I_{\mathrm{pfx}}}$.
For a prefix verifier $h\in H$, the label
$h(x,\tau_{1:\ell})$ is defined exactly on the promised prefixes:
\[
h(x,\tau_{1:\ell})\neq\star
\quad\Longleftrightarrow\quad
h(x,\tau_{1:j})=\yes\ \text{for every }j<\ell.
\]
On such prefixes, $h$ predicts whether the last step is valid. Thus,
realizable prefix-verification sequences only contain prefixes satisfying
the target-dependent promise, but the learner is not given this domain.
Let the full chain-of-thought verification problem be
$P_{\mathrm{cot}} =
(I_{\mathrm{cot}},Y_{\mathrm{cot}},H_{\mathrm{cot}})$,
where $I_{\mathrm{cot}} = X \times \Sigma^L$ and
$Y_{\mathrm{cot}} = [L]\cup\{\infty\}$. Each
$h \in H$ induces a full-trace verifier
$g_h(x,\tau_{1:L})
=
\min(\{\ell \in [L] : h(x,\tau_{1:\ell})=\no\}\cup\{\infty\})$.
Thus,
$H_{\mathrm{cot}}=\{g_h : h \in H\}$, which is a total class.

Section~\ref{sec:reduction} proves reductions in both directions between these two
online problems.
We now show that both reductions can be viewed as harness-decoder
constructions, in the 0-1 mistake setting.

The harness itself is unchanged in the partial setting:
Algorithm~\ref{alg:trace-decodable-reduction} applies $S$ only to the
total predictors output by the base learner. Partiality only changes which
labeled sequences are realizable, and which decoded base transcripts 
are valid supervision for the base learner.

\begin{theorem}[Promised online-to-online reduction]\label{thm:partial-online-to-online-reduction}
Let $\mathsf P_{\mathrm{base}}$ be a partial and $\mathsf P_{\mathrm{sys}}$ be a total
online prediction problem. Suppose there is a harness $S$ consistent with
a decoder $D$ on the promise in the following sense: for every
$g\in H_{\mathrm{sys}}$ there exists $h_g\in H_{\mathrm{base}}$ such that
for every system input $i$ with $g(i)\neq\star$, writing
\[
D(i,g(i))=((q_1,a_1),\ldots,(q_m,a_m)),
\]
we have $h_g(q_r)=a_r\neq\star$ for every $r\le m$, and the decoded
transcript is exactly a valid execution transcript for the harness:
\[
S(i,((q_1,a_1),\ldots,(q_{r-1},a_{r-1})))=q_r
\quad\text{for every }r\le m,
\qquad
S(i,D(i,g(i)))=g(i).
\]
Let $R_{S,D}$ be the online-to-online reduction defined in
Algorithm~\ref{alg:trace-decodable-reduction}. Then, for every
base learner $A_{\mathrm{base}}:\Gamma_{\mathrm{base}}\to
Y_{\mathrm{base}}^{I_{\mathrm{base}}}$,
\[
\operatorname{MB}(R_{S,D}(A_{\mathrm{base}}),\mathsf P_{\mathrm{sys}})
\le
\operatorname{MB}(A_{\mathrm{base}},\mathsf P_{\mathrm{base}}).
\]
In particular, $\operatorname{MB}(\mathsf P_{\mathrm{sys}})
\le \operatorname{MB}(\mathsf P_{\mathrm{base}})$.
\end{theorem}

\begin{proof}
The proof is the same induction as in
\Cref{thm:online-to-online-reduction}, with the quantifiers restricted to
defined labels. Fix a realizable system sequence
\[
G_T^{\mathrm{sys}}=((i^{(1)},y^{(1)}),\ldots,(i^{(T)},y^{(T)}))
\]
and choose a partial witness $g^\star\in H_{\mathrm{sys}}$ such that
$g^\star(i^{(t)})=y^{(t)}\neq\star$ for every $t$. By
the promise-consistency assumption, there is $h^\star\in H_{\mathrm{base}}$
such that for each observed system input $i^{(t)}$, the decoded transcript
\[
T^{(t)}:=D(i^{(t)},y^{(t)})
\]
is a valid harness transcript ending in $y^{(t)}$, and every labeled
example in $T^{(t)}$ is defined and labeled by $h^\star$.
Let $B_t$ be the base history produced by
Algorithm~\ref{alg:trace-decodable-reduction} after the first $t$ system
examples. We prove by induction that every labeled example in $B_t$ is
defined and labeled by $h^\star$, and that $A_{\mathrm{base}}$ made a
mistake on each such example.

The claim is trivial at $t=0$. Suppose it holds through $t-1$ and let
$\hat h_t=A_{\mathrm{base}}(B_{t-1})$. If
$S[\hat h_t](i^{(t)})=y^{(t)}$, the algorithm appends nothing. Otherwise,
the decoded transcript
\[
T^\circ:=D(i^{(t)},y^{(t)})
\]
is a valid harness transcript ending in $y^{(t)}$ and is defined under
$h^\star$. Since the learned harness execution ends in the different
output $S[\hat h_t](i^{(t)})$, determinism of the harness implies a first
base-label disagreement between
$\operatorname{Transcript}_{S}^{\hat h_t}(i^{(t)})$ and $T^\circ$. If this
disagreement is $(z,a)$ on the decoded base transcript, then
$h^\star(z)=a\neq\star$ and $\hat h_t(z)\neq a$. Thus the algorithm
appends a defined base example on which the current base predictor errs.

Therefore the number of system mistakes is at most the number of mistakes
made by $A_{\mathrm{base}}$ on the realizable partial base sequence
$B_T$. Taking the worst case over realizable partial system sequences
gives the first inequality, and taking the infimum over base learners gives
the second.
\end{proof}
}

\subsubsection{From prefix verification to full CoT verification}

Here the base problem is prefix verification and the system problem is
full CoT verification,
$P_{\mathrm{base}} = P_{\mathrm{pfx}}$ and
$P_{\mathrm{sys}} = P_{\mathrm{cot}}$.

\paragraph{Harness.}

A base transcript is an element of
$\Gamma_{\mathrm{pfx}}
=
(I_{\mathrm{pfx}}\times Y_{\mathrm{pfx}})^*$.
Given a full trace
$z=(x,\tau_{1:L}) \in I_{\mathrm{cot}}$
and transcript
$\gamma=((q_1,b_1),\ldots,(q_m,b_m))$, 
define the harness
$S_{\mathrm{pfx}\to\mathrm{cot}}
:
I_{\mathrm{cot}} \times \Gamma_{\mathrm{pfx}}
\to
I_{\mathrm{pfx}} \sqcup Y_{\mathrm{cot}}$
as follows.
If $m=0$, the harness queries the first prefix,
$S_{\mathrm{pfx}\to\mathrm{cot}}(z,\emptyset)
=
(x,\tau_{1:1})$.
Suppose now that
$q_j=(x,\tau_{1:j})$ for all $j\le m$.
If some previous answer equals $\no$, let
$j^\star=\min\{j\le m : b_j=\no\}$.
The harness halts and outputs $j^\star$.
If all previous answers are $\yes$ and $m<L$, the harness queries the
next prefix,
$S_{\mathrm{pfx}\to\mathrm{cot}}(z,\gamma)
=
(x,\tau_{1:m+1})$.
Finally, if all previous answers are $\yes$ and $m=L$, the harness
halts and outputs $\infty$.

{
Equivalently, for any total prefix predictor
$\bar h:I_{\mathrm{pfx}}\to Y_{\mathrm{pfx}}$,
\[
S_{\mathrm{pfx}\to\mathrm{cot}}[\bar h](x,\tau_{1:L})
=
\min(
\{\ell\in[L] : \bar h(x,\tau_{1:\ell})=\no\}
\cup
\{\infty\}
).
\]
This is the only way the reduction uses the harness at inference time:
it is always run with the total predictor produced by the base learner.
}

\paragraph{Decoder.}

For a full-trace label
$y\in [L]\cup\{\infty\}$,
define the decoder
$D_{\mathrm{pfx}\to\mathrm{cot}}
:
I_{\mathrm{cot}}\times Y_{\mathrm{cot}}
\to
\Gamma_{\mathrm{pfx}}$
as follows.
If $y\in[L]$, define
\[
D_{\mathrm{pfx}\to\mathrm{cot}}
((x,\tau_{1:L}),y)
=
\bigl(
((x,\tau_{1:1}),\yes),
\ldots,
((x,\tau_{1:y-1}),\yes),
((x,\tau_{1:y}),\no)
\bigr).
\]
If $y=\infty$, define
\[
D_{\mathrm{pfx}\to\mathrm{cot}}
((x,\tau_{1:L}),\infty)
=
\bigl(
((x,\tau_{1:1}),\yes),
\ldots,
((x,\tau_{1:L}),\yes)
\bigr).
\]
Thus, the decoder reconstructs the sequence of prefix-verification
interactions corresponding to the true location of the first error.

{
\paragraph{Consistency.}

Fix any system target $g\in H_{\mathrm{cot}}$. By definition of
$H_{\mathrm{cot}}$, there exists a prefix verifier
$h\in H_{\mathrm{pfx}}$ such that $g=g_h$; this $h$ is the base witness.
For every full trace $z=(x,\tau_{1:L})$, let
$W(z)=D_{\mathrm{pfx}\to\mathrm{cot}}(z,g(z))$. If $g(z)=y\in[L]$,
then $W(z)$ consists of the labels $\yes$ on prefixes $1,\ldots,y-1$ and
$\no$ on prefix $y$, all of which are defined by the partial prefix
verifier $h$. If $g(z)=\infty$, then $W(z)$ consists of the labels $\yes$
on all prefixes, again all defined by $h$. In both cases, running
$W(z)$ through $S_{\mathrm{pfx}\to\mathrm{cot}}$ makes the harness query
exactly those prefixes and halt with output $g(z)$. Thus
$(S_{\mathrm{pfx}\to\mathrm{cot}},
D_{\mathrm{pfx}\to\mathrm{cot}})$ satisfies the promise-consistency
condition of \Cref{thm:partial-online-to-online-reduction}.
Applying \Cref{thm:partial-online-to-online-reduction} therefore gives
$\operatorname{MB}(P_{\mathrm{cot}})
\le
\operatorname{MB}(P_{\mathrm{pfx}})$,
recovering the single-error version our result (Theorem \ref{thm:red1}).
}

\subsubsection{From full CoT verification to prefix verification}

We now recover the reverse reduction. Here the base problem is full CoT
verification and the system problem is prefix verification:
$P_{\mathrm{base}} = P_{\mathrm{cot}}$ and
$P_{\mathrm{sys}} = P_{\mathrm{pfx}}$.
As in Theorem~\ref{thm:red2}, assume there exists a special token
$F\in\Sigma$ such that every verifier rejects $F$ after every correct
prefix:
$h(x,(\tau_{1:\ell-1},F))
=
\no$
for every $h\in H$, every $x\in X$, and every correct prefix
$\tau_{1:\ell-1}$.

\paragraph{Harness.}
A base transcript is an element of
$\Gamma_{\mathrm{cot}}
=
(I_{\mathrm{cot}}\times Y_{\mathrm{cot}})^*$.
Given a prefix instance
$z_{\mathrm{pfx}}=(x,\tau_{1:\ell}) \in I_{\mathrm{pfx}}$,
define its padded full trace by
$\operatorname{pad}_F(x,\tau_{1:\ell})
=
(x,\tau_1,\ldots,\tau_\ell,F,\ldots,F)
\in X\times \Sigma^L$.
Define the harness
$S_{\mathrm{cot}\to\mathrm{pfx}}
:
I_{\mathrm{pfx}}\times \Gamma_{\mathrm{cot}}
\to
I_{\mathrm{cot}}\sqcup Y_{\mathrm{pfx}}$
as follows.
Initially, the harness queries the padded trace:
$S_{\mathrm{cot}\to\mathrm{pfx}}(z_{\mathrm{pfx}},\emptyset)
=
\operatorname{pad}_F(z_{\mathrm{pfx}})$.
After receiving a full-trace label
$\hat y\in[L]\cup\{\infty\}$,
the harness halts and outputs
\[
S_{\mathrm{cot}\to\mathrm{pfx}}
(
z_{\mathrm{pfx}},
((\operatorname{pad}_F(z_{\mathrm{pfx}}),\hat y))
)
=
\begin{cases}
\no,
&
\hat y \le \ell,
\\
\yes,
&
\hat y > \ell.
\end{cases}
\]
Thus, the harness accepts the original prefix iff the padded full trace
survives strictly beyond position $\ell$.
{
For any $h\in H$, let $g_h$ denote its induced full-trace verifier.
Then, on any promised prefix instance, equivalently on any
$(x,\tau_{1:\ell})$ with $h(x,\tau_{1:\ell})\neq\star$,
$S_{\mathrm{cot}\to\mathrm{pfx}}[g_h](x,\tau_{1:\ell})
=
h(x,\tau_{1:\ell})$.
}

\paragraph{Decoder.}
Define
$D_{\mathrm{cot}\to\mathrm{pfx}}
:
I_{\mathrm{pfx}}\times Y_{\mathrm{pfx}}
\to
\Gamma_{\mathrm{cot}}$
as follows.
If the correct prefix label is $\no$, define
$D_{\mathrm{cot}\to\mathrm{pfx}}
((x,\tau_{1:\ell}),\no)
=
(
(\operatorname{pad}_F(x,\tau_{1:\ell}),\ell)
)$.
If the correct prefix label is $\yes$ and $\ell<L$, define
$D_{\mathrm{cot}\to\mathrm{pfx}}
((x,\tau_{1:\ell}),\yes)
=
(
(\operatorname{pad}_F(x,\tau_{1:\ell}),\ell+1)
)$.
Finally, if $\ell=L$, define
$D_{\mathrm{cot}\to\mathrm{pfx}}
((x,\tau_{1:L}),\yes)
=
(
(\operatorname{pad}_F(x,\tau_{1:L}),\infty)
)$.

{
\paragraph{Consistency.}
Fix any system target $h\in H_{\mathrm{pfx}}$ and let $g_h$ denote its
induced full verifier; this $g_h$ is the base witness. Consider any
promised prefix instance $(x,\tau_{1:\ell})$, equivalently any prefix
with $h(x,\tau_{1:\ell})\neq\star$. The decoded transcript contains a
single full-trace example, and its label is defined because $g_h$ is a
total full CoT verifier. If $h(x,\tau_{1:\ell})=\no$, the promise implies
that the padded trace first fails at position $\ell$, so the decoded label
is $\ell$. If $h(x,\tau_{1:\ell})=\yes$ and $\ell<L$,
token $F$ induces a first faulty step at position $\ell+1$. If
$h(x,\tau_{1:L})=\yes$, the padded trace is the original full trace and
the decoded label is $\infty$. In each case, running the decoded
transcript through $S_{\mathrm{cot}\to\mathrm{pfx}}$ makes the harness halt
with the prefix label $h(x,\tau_{1:\ell})$. Thus
$(S_{\mathrm{cot}\to\mathrm{pfx}},
D_{\mathrm{cot}\to\mathrm{pfx}})$ satisfies the promise-consistency
condition of \Cref{thm:partial-online-to-online-reduction}, assuming the
existence of token $F \in \Sigma$.
Applying \Cref{thm:partial-online-to-online-reduction} yields
$\operatorname{MB}(P_{\mathrm{pfx}})
\le
\operatorname{MB}(P_{\mathrm{cot}})$,
recovering the single-error version of our result  (Theorem \ref{thm:red2}).
}

\subsection{Chain-of-Thought Generation via Sequential Harnessing}\label{subsec:online-generations}

We now show how the sequential harnessing framework can be used to
derive generalization guarantees for autoregressive Chain-of-Thought
(CoT) generation. The key idea is that a mistake in a generated
trajectory must correspond to a mistake in next-token prediction. Sequential harnessing therefore converts
online learnability of next-token prediction into online learnability of
entire reasoning trajectories.

\paragraph{Setup.}

Let $\Sigma$ be a finite token alphabet and let
$F \subseteq \Sigma^{\Sigma^\ast}$ be a class of next-token predictors.
For a predictor $f \in F$ and prompt $x \in \Sigma^\ast$, define the
autoregressive trajectory generated by $f$ recursively as follows.
Starting from prompt $x$, define $y_1 = f(x)$ and, for $s \ge 2$,
define $y_s = f(x,y_{1:s-1})$.
The induced length-$T$ chain-of-thought generator is the map
$g_f : \Sigma^\ast \to \Sigma^T$ defined by
$g_f(x) = (y_1,\ldots,y_T)$.
We denote the induced trajectory class by
$F^{\mathrm{CoT}-T} = \{g_f : f \in F\}$.
Our goal is to learn the trajectory class
$F^{\mathrm{CoT}-T}$.

\paragraph{Base and system prediction problems.}
Define the base prediction problem
$P_{\mathrm{tok}}
=
(I_{\mathrm{tok}},Y_{\mathrm{tok}},H_{\mathrm{tok}})$
by
$I_{\mathrm{tok}} = \Sigma^\ast$,
$Y_{\mathrm{tok}} = \Sigma$,
and
$H_{\mathrm{tok}} = F$.
Thus, a base instance is a token prefix and the label is the next token.
Define the system prediction problem
$P_{\mathrm{CoT}}
=
(I_{\mathrm{CoT}},Y_{\mathrm{CoT}},H_{\mathrm{CoT}})$
by
$I_{\mathrm{CoT}} = \Sigma^\ast$,
$Y_{\mathrm{CoT}} = \Sigma^T$,
and
$H_{\mathrm{CoT}} = F^{\mathrm{CoT}-T}$.
Thus, a system predictor outputs an entire length-$T$ reasoning
trajectory.

\paragraph{Harness construction.}
Let
$\Gamma_{\mathrm{tok}}
=
(I_{\mathrm{tok}}\times Y_{\mathrm{tok}})^\ast$.
Given a prompt $x$ and transcript
$\gamma
=
((q_1,a_1),\ldots,(q_m,a_m))
\in
\Gamma_{\mathrm{tok}}$,
define the harness
$S_{\mathrm{CoT}}
:
I_{\mathrm{CoT}}
\times
\Gamma_{\mathrm{tok}}
\to
I_{\mathrm{tok}}
\sqcup
Y_{\mathrm{CoT}}$
as follows.
If $m<T$, the harness outputs the next-token prediction query
$S_{\mathrm{CoT}}(x,\gamma)
=
(x,a_{1:m})$.
If $m=T$, the harness halts and outputs the trajectory
$S_{\mathrm{CoT}}(x,\gamma)
=
(a_1,\ldots,a_T)$.
Thus, the harness repeatedly queries the current prefix, appends the
predicted next token, and after $T$ rounds outputs the full generated
trajectory.
For every $f \in F$, the induced system predictor satisfies
$S_{\mathrm{CoT}}[f](x)
=
g_f(x)$.

\paragraph{Decoder construction.}
Given a target trajectory
$y=(y_1,\ldots,y_T)\in\Sigma^T$,
define the decoder
$D_{\mathrm{CoT}}
:
I_{\mathrm{CoT}}
\times
Y_{\mathrm{CoT}}
\to
\Gamma_{\mathrm{tok}}$
by
$D_{\mathrm{CoT}}(x,y)
=
\bigl(
((x),y_1),
((x,y_1),y_2),
\ldots,
((x,y_{1:T-1}),y_T)
\bigr)$.
Thus, the decoder converts a supervised chain-of-thought trajectory into
the corresponding sequence of next-token supervision examples.

\paragraph{Consistency.}
Fix any system predictor
$g_f \in H_{\mathrm{CoT}}$ induced by some
$f \in F$. 
Then for every prompt $x$,
$S_{\mathrm{CoT}}[f](x)
=
g_f(x)$.
Moreover, the transcript generated by the harness while interacting with
$f$ on prompt $x$ is exactly
$D_{\mathrm{CoT}}(x,g_f(x))$.
Hence
$(S_{\mathrm{CoT}},D_{\mathrm{CoT}})$
is a consistent harness-decoder pair.
Applying the sequential harnessing theorem therefore yields the
following result.

\begin{theorem}[Online CoT generation via harnessing]\label{thm:cot-generation}
Let $F \subseteq \Sigma^{\Sigma^\ast}$ be a next-token prediction class
and let
$F^{\mathrm{CoT}-T}$
denote the induced length-$T$ autoregressive trajectory class.
Then
$MB(F^{\mathrm{CoT}-T})
\le
MB(F)$.
In particular, if $\ldim(F)=d<\infty$, then there exists an online
learner for
$F^{\mathrm{CoT}-T}$
that makes at most $d$ mistakes, independent of the trajectory length
$T$.
\end{theorem}

\begin{proof}
The mistake bound follows immediately from the sequential harnessing
theorem and the consistency of
$(S_{\mathrm{CoT}},D_{\mathrm{CoT}})$.
Concretely, suppose the learner predicts trajectory
$\hat y=(\hat y_1,\ldots,\hat y_T)$
while the target trajectory is
$y=(y_1,\ldots,y_T)$.
If $\hat y \neq y$, let
$s^\star
=
\min\{s : \hat y_s \neq y_s\}$.
By definition of $s^\star$, the prefixes agree before time $s^\star$:
$\hat y_{1:s^\star-1}
=
y_{1:s^\star-1}$.

Therefore, the trajectory mistake exposes a next-token
prediction mistake on the prefix
$(x,y_{1:s^\star-1})$:
$\hat f(x,y_{1:s^\star-1})
=
\hat y_{s^\star}
\neq
y_{s^\star}
=
f^\star(x,y_{1:s^\star-1})$.
Thus, every trajectory-level mistake induces a base-level next-token
mistake, or,
$MB(F^{\mathrm{CoT}-T})
\le
MB(F)$.

Finally, if $\ldim(F)=d$, the Standard Optimal Algorithm for online
realizable multiclass prediction makes at most $d$ mistakes on $F$,
implying the same mistake bound for
$F^{\mathrm{CoT}-T}$.
\end{proof}

We now obtain a PAC-learning consequence via the standard
online-to-batch conversion.

\begin{corollary}[PAC learnability of CoT generation]\label{cor:generation}
Suppose $\ldim(F)=d<\infty$. Then the autoregressive trajectory class
$F^{\mathrm{CoT}-T}$ is PAC learnable with sample complexity
$O\!\left(
\frac{
d\log(1/\epsilon)
+
\log(1/\delta)
}{
\epsilon
}
\right)$,
independent of the trajectory length $T$.
\end{corollary}

\begin{proof}
By the previous theorem,
$MB(F^{\mathrm{CoT}-T}) \le d$.
Applying the standard realizable online-to-batch conversion for finite
mistake-bound classes yields the desired sample complexity bound.
\end{proof}

\noindent\emph{Comparison with prior and concurrent work.}
The work of \citet{joshi2025theory} introduced the autoregressive CoT learning model and gave general PAC guarantees based on ERM (consistency in the realizable setting). 
After the standard online-to-batch conversion, \Cref{cor:generation} gives the analogous CoT-supervised guarantee with \(\ldim(F)\) in place of \(\ldim(F)\log T\), which appears in their corresponding sample complexity bound. Notably, our algorithm is not ERM-based and instead runs a realizable online learner for \(F\) which updates only on the first wrong next-token prediction. Length-independent CoT bounds were also obtained concurrently by \citet{doronarad2026theoryonlinelearningautoregressive} in their online learning framework, as well as by \citet{hanneke2026samplecomplexityautoregressive} via a stable sample compression approach.

\subsection{Improving Weak Generators as Sequential Harnessing}
We now reinterpret the verifier-assisted generator framework of
Section~\ref{sec:boost} in the language of sequential harnessing.
The
base learner is now a verifier rather than a predictor of the next
reasoning step. The harness uses the verifier interactively in order to
search for a correct reasoning trajectory generated by one or more weak
generators.

Let $X$ denote a space of problem instances and let $\Sigma$ denote a
space of reasoning steps. Fix a horizon $L$. Let
$H \subseteq \{\yes,\no\}^{X\times\Sigma^\ast}$ denote a class of
prefix verifiers. Recall that a verifier
$h(x,\tau_{1:\ell})$ predicts whether the last reasoning step
$\tau_\ell$ is valid assuming the preceding prefix is correct.
Suppose we are also given a collection of generators
$\mathcal G=\{G_1,\ldots,G_N\}$. Given a problem instance $x$ and a
correct prefix $\tau_{1:\ell}$, generator $G_i$ generates a distribution
over candidate next reasoning steps.
The key assumption is that the generator collection is
$\alpha$-good, i.e., for every correct prefix, at least one generator produces a
correct next step with probability at least $\alpha$.
The goal is now to construct a strong generator, given access to a ``perfect'' prefix verifier.

\paragraph{Base and system problems.}

The base problem is prefix verification with instance
space  \(I_{\mathrm{ver}}=X\times\Sigma^{\le L}\), label space
\(Y_{\mathrm{ver}}=\{\yes,\no\}\), and verifier class
\(H_{\mathrm{ver}}=H\subseteq
(Y_{\mathrm{ver}}\cup\{\star\})^{I_{\mathrm{ver}}}\). Let
\(\mathsf P_{\mathrm{ver}}=(I_{\mathrm{ver}},Y_{\mathrm{ver}},
H_{\mathrm{ver}})\).
Let
\(\Gamma_{\mathrm{ver}}=(I_{\mathrm{ver}}\times
Y_{\mathrm{ver}})^\ast\) be the verifier-transcript space. The system
problem is the trace-generation interaction viewed through this
transcript: \(I_{\mathrm{gen}}=X\times\Omega\) and
\(Y_{\mathrm{gen}}=\Gamma_{\mathrm{ver}}\). Here
\(\omega\in\Omega\) is a random seed fixing the generator samples for every
possible prefix, generator index, and iteration counter that
Algorithm~\ref{alg:weak-to-strong} could request. Thus, after
conditioning on \(\omega\), the generator wrapper is deterministic for
every verifier transcript, not only for the transcript generated by the
perfect verifier.

\paragraph{Harness construction.}

Fix the set of generators \(\mathcal G\) and the order in which
Algorithm~\ref{alg:weak-to-strong} queries the generators. The harness
\(S_{\mathrm{gen}}:I_{\mathrm{gen}}\times\Gamma_{\mathrm{ver}}\to
I_{\mathrm{ver}}\sqcup Y_{\mathrm{gen}}\) simulates the deterministic execution of 
Algorithm~\ref{alg:weak-to-strong} with random choices fixed by \(\omega\).
Given a system input \((x,\omega)\) and a partial labeled verifier
transcript \(\gamma\), the harness the harness runs this simulation from the beginning, supplying the verifier responses recorded in
\(\gamma\). If the next action is a verifier query
\(q\in I_{\mathrm{ver}}\), the harness outputs \(q\). If the simulation
halts after consuming \(\gamma\), the harness outputs the completed
transcript \(\gamma\). On inconsistent transcripts, define the harness
arbitrarily. This definition is unchanged by partiality: at inference
time the harness is run with the total predictors output by the base
learner.

For every target verifier \(h\in H_{\mathrm{ver}}\), the target execution
only queries promised prefixes: rejected candidates are tested after the
current accepted prefix, and accepted candidates become the next correct
prefix. Hence the induced target transcript predictor
\(g_h:I_{\mathrm{gen}}\to Y_{\mathrm{gen}}\) is well-defined by
\(g_h(x,\omega)=
\operatorname{Transcript}_{S_{\mathrm{gen}}}^{h}(x,\omega)\). Define
\(H_{\mathrm{gen}}:=\{g_h:h\in H_{\mathrm{ver}}\}\) and
\(\mathsf P_{\mathrm{gen}}:=(I_{\mathrm{gen}},Y_{\mathrm{gen}},
H_{\mathrm{gen}})\).

\paragraph{Decoder construction.}

Since the system label is already the labeled verifier transcript, the
decoder \(D_{\mathrm{gen}}:I_{\mathrm{gen}}\times Y_{\mathrm{gen}}\to
\Gamma_{\mathrm{ver}}\) is simply
\(D_{\mathrm{gen}}((x,\omega),\gamma)=\gamma\). Here the transcript contains both the candidate steps accepted by the
verifier and the candidate steps rejected by the verifier, so it gives
the base learner the full transcript needed after a system-level
mistake.

\paragraph{Consistency.}

Fix any system target \(g\in H_{\mathrm{gen}}\). By definition of
\(H_{\mathrm{gen}}\), there exists a verifier
\(h_g\in H_{\mathrm{ver}}\) such that \(g=g_{h_g}\). For every system
input \((x,\omega)\), the identity decoder gives
\(D_{\mathrm{gen}}((x,\omega),g(x,\omega))
=g(x,\omega)
=\operatorname{Transcript}_{S_{\mathrm{gen}}}^{h_g}(x,\omega)\).
Every query in this decoded transcript is defined and labeled by
\(h_g\), so \((S_{\mathrm{gen}},D_{\mathrm{gen}})\) satisfies the
promise-consistency condition. Applying
\Cref{thm:partial-online-to-online-reduction}
gives \(\operatorname{MB}(\mathsf P_{\mathrm{gen}})\le
\operatorname{MB}(\mathsf P_{\mathrm{ver}})\).
Equivalently, whenever the induced transcript predictor makes a mistake,
the two deterministic executions have a first verifier query on which the
current verifier predictor disagrees with the target verifier.

\paragraph{Batch transcript learning algorithm.}

Let \(h^\star\) denote the perfect verifier and write
\(g^\star:=g_{h^\star}\). Let \(\mu_\Omega\) be the seed distribution
for the seeded execution underlying \(S_{\mathrm{gen}}\): a seed
\(\omega\sim\mu_\Omega\) fixes all generator samples used by the wrapper.
For each pair \((x,\omega)\), the transcript label is
\(g^\star(x,\omega)\). Operationally, this label is obtained before
learning by running the seeded scaffold on \((x,\omega)\), using
\(h^\star\) to answer verifier queries, and recording the resulting
labeled verifier transcript. This is well-defined even when the wrapper
later times out; the label is the transcript, not a successful reasoning trace.
Let \(\operatorname{Post}:Y_{\mathrm{gen}}\to
\Sigma^L\cup\{\textsf{idk}\}\) denote the deterministic map that reads a
completed verifier transcript and returns the final trace or abstention
produced by the corresponding wrapper execution.

Thus the batch data are ordinary supervised transcript examples. Write
\[
\mathcal S_{\mathrm{tr}}
=
\bigl(((x_s,\omega_s),y_s)\bigr)_{s=1}^{n_1},
\qquad
\mathcal V_{\mathrm{tr}}
=
\bigl(((x'_r,\omega'_r),y'_r)\bigr)_{r=1}^{n_2},
\]
where \((x_s,\omega_s)\) and \((x'_r,\omega'_r)\) are independent draws
from \(D\times\mu_\Omega\), and
\(y_s=g^\star(x_s,\omega_s)\),
\(y'_r=g^\star(x'_r,\omega'_r)\). The algorithm below receives only
these labeled datasets, not \(h^\star\).

\begin{algorithm}[H]
\caption{Batch imitation-learning transcript selection}
\label{alg:batch-transcript-proof-generation}
\DontPrintSemicolon
\LinesNumbered
\SetNoFillComment
\KwIn{conservative online transcript learner \(A_{\mathrm{gen}}\);
training transcript sample \(\mathcal S_{\mathrm{tr}}\); validation
transcript sample \(\mathcal V_{\mathrm{tr}}\)}
\KwOut{seeded generator \(\hat G\)}
Set \(n_1:=|\mathcal S_{\mathrm{tr}}|\)\;
For each \(s\in[n_1+1]\), let
\(\hat g_s:=A_{\mathrm{gen}}(\mathcal S_{\mathrm{tr},<s})\), where
\(\mathcal S_{\mathrm{tr},<s}\) is the prefix of
\(\mathcal S_{\mathrm{tr}}\) of length \(s-1\)\;
Let \(\hat s\in[n_1+1]\) minimize the empirical trace error of
\(\hat g_s\) on \(\mathcal V_{\mathrm{tr}}\)\;
\Return the seeded generator \(\hat G\) defined by
\(\hat G(x,\omega):=
\operatorname{Post}(\hat g_{\hat s}(x,\omega))\)\;
\end{algorithm}

\paragraph{Imitation-learning interpretation.}

The setting of \Cref{alg:batch-transcript-proof-generation} is an
imitation-learning version of the transcript-level system problem, in the
same spirit as \Cref{subsec:imitation-harnessing}. The training data are
completed ground-truth transcripts generated before learning, rather than
on-policy supervision collected from the verifier learner's own
interaction. This makes the data assumption simpler than in the main
generator-boosting result: the learner only needs supervised transcript
examples \(((x,\omega),g^\star(x,\omega))\). The price is that the
guarantee is weaker. It controls the probability that the learned
transcript predictor differs from the ground-truth transcript predictor
\(g^\star\), and hence gives a single aggregate end-to-end failure bound.
It does not separate soundness and completeness errors, as the main result
in \Cref{sec:boost} does.

\begin{theorem}[End-to-end batch guarantee for transcript harnessing]
\label{thm:batch-transcript-proof-generation}
Suppose \(h^\star\in H_{\mathrm{ver}}\) is a perfect verifier and the
generator collection \(\mathcal G\) is \((\alpha,\gamma)\)-good for
distribution \(D\). Let \(A_{\mathrm{ver}}\) be a conservative online
verifier learner with
\(\operatorname{MB}(A_{\mathrm{ver}},\mathsf P_{\mathrm{ver}})\le
\bar M\), and let
\(A_{\mathrm{gen}}:=R_{S_{\mathrm{gen}},D_{\mathrm{gen}}}
(A_{\mathrm{ver}})\) be the induced online transcript learner. Let
\(\mathcal S_{\mathrm{tr}}\) and \(\mathcal V_{\mathrm{tr}}\) be
independent transcript datasets as defined above, with
\[
|\mathcal S_{\mathrm{tr}}|
=
\Theta\!\left(
\frac{\bar M+\log(1/\delta_{\mathrm{tr}})}
{\epsilon_{\mathrm{tr}}}
\right),
\qquad
|\mathcal V_{\mathrm{tr}}|
=
\Theta\!\left(
\frac{1}{\epsilon_{\mathrm{tr}}}
\log\frac{\bar M+1}{\delta_{\mathrm{tr}}}
\right).
\]
Run Algorithm~\ref{alg:batch-transcript-proof-generation} with
\(A_{\mathrm{gen}},\mathcal S_{\mathrm{tr}},\mathcal V_{\mathrm{tr}}\),
with the constants hidden above chosen sufficiently large, and with
\[
N=\left\lceil
\frac{1}{\alpha}\log\frac{NL}{\delta}
\right\rceil
\]
rounds at each step, where each round samples all \(N\) generators.
Then, with probability at least \(1-\delta_{\mathrm{tr}}\) over the
draw of \(\mathcal S_{\mathrm{tr}}\) and \(\mathcal V_{\mathrm{tr}}\),
the transcript predictor \(\hat g_{\hat s}\) selected by the algorithm
has trace error at most \(\epsilon_{\mathrm{tr}}\):
\[
\Pr_{x\sim D,\omega\sim\mu_\Omega}
\left[
\hat g_{\hat s}(x,\omega)\neq g^\star(x,\omega)
\right]
\le
\epsilon_{\mathrm{tr}}.
\]
Consequently, with the same probability over the draw of
\(\mathcal S_{\mathrm{tr}}\) and \(\mathcal V_{\mathrm{tr}}\),
\[
\Pr_{x\sim D,\omega\sim\mu_\Omega}
\left[
\hat G(x,\omega)\text{ is not a correct reasoning trace}
\right]
\le
\epsilon_{\mathrm{tr}}+(1-\gamma)+\delta.
\]
\end{theorem}

\begin{proof}
The labeled transcript examples are realizable by the system target
\(g^\star\). By the promised online-to-online reduction above, the
online transcript learner \(A_{\mathrm{gen}}\) has mistake bound at most
\(\bar M\) on every realizable transcript sequence.
The standard mistake-bound-to-PAC conversion, followed by validation over
the candidate sequence \(\hat g_1,\ldots,\hat g_{n_1+1}\), which has at
most \(\bar M+1\) distinct elements by conservativeness, gives with
probability at least \(1-\delta_{\mathrm{tr}}\) the displayed trace-error
bound for \(\hat g_{\hat s}\).

It remains to translate transcript accuracy into reasoning trace accuracy. If
\(\hat g_{\hat s}(x,\omega)=g^\star(x,\omega)\), then the
learned transcript predictor and the perfect-verifier transcript have the
same postprocessed final output. Thus a learned wrapper failure is
contained in the union of the trace-error event and
the event that the perfect-verifier wrapper itself does not produce a
correct reasoning trace.

On an \(\alpha\)-good instance, a fixed sampling round fails to include a
correct continuation with probability at most \(1-\alpha\). With the
choice of \(N\) above, a fixed step times out with probability at
most \((1-\alpha)^N\le \delta/(NL)\). A union bound over the \(L\)
steps gives failure probability at most \(\delta/N\le
\delta\) on the good set, while the non-good set has probability at
most \(1-\gamma\). Therefore the perfect-verifier wrapper fails with
probability at most \((1-\gamma)+\delta\). Combining this with the
trace-error bound gives the end-to-end guarantee.
\end{proof}

\paragraph{Connection to Section~\ref{sec:boost}.}

This example is closely related to the verifier-assisted generator
construction in Section~\ref{sec:boost}. 
However, there are several important conceptual differences.

First, the present harnessing formulation does not directly learn a
verifier from generator interaction. Instead, it assumes access to a correct
prefix verifier (oracle, or human expert) and uses it to improve generation through interaction
with weak generators.
In contrast, Section~\ref{sec:boost} explicitly studies how a verifier itself can be
learned online from oracle feedback while interacting with adaptive
generators. 

Second, the sequential harnessing abstraction only tracks ordinary
trajectory-level mistakes. It does not distinguish between soundness and
completeness errors.
The central contribution of Section~\ref{sec:boost} is precisely that these two error
types matter differently for generator amplification. In particular,
soundness errors correspond to accepting incorrect reasoning steps and
therefore directly control the probability that the wrapped generator
outputs an incorrect reasoning trace. Completeness errors instead lead primarily
to abstention or rejection of correct candidate steps. 

Thus, sequential harnessing captures the structural online reduction
underlying verifier-guided generation, while Section~\ref{sec:boost} develops the much
more refined analysis needed to reason about asymmetric verifier errors,
distribution shift induced by generator-verifier interaction, and PAC-style
generalization guarantees for the resulting wrapped generator.

\end{appendices}

\end{document}